\documentclass[12pt]{article}

\usepackage{times}
\usepackage{amsmath, amssymb, amsfonts,  amsthm}
\usepackage{bm, braket}
\usepackage{tikz,subfigure}
\usepackage{pgfplots}
\usepackage{adjustbox}
\usepackage{natbib}
\setcitestyle{numbers}
\usepackage[a4paper, total={6in, 8in}]{geometry}
\usepackage{hyperref}
\usepackage[affil-it]{authblk}

\title{Precise Asymptotic Analysis of Deep Random Feature Models}

\author[1]{David Bosch\footnote{davidbos@chalmers.se}}
\author[1]{Ashkan Panahi\footnote{ashkan.panahi@chalmers.se}}
\author[2]{Babak Hassibi\footnote{bhassibi@caltech.edu}}

\affil[1]{Department of Data Science and AI, Computer Science and Engineering, Chalmers University of Technology}
\affil[2]{Department of Electrical Engineering, California Institute of Technology}

\pgfplotsset{compat=1.17}
\begin{document}

\newcommand{\bDelta}{\bm{\Delta}}
\newcommand{\bLambda}{\bm{\Lambda}}
\newcommand{\bGamma}{\bm{\Gamma}}
\newcommand{\bSigma}{\bm{\Sigma}}
\newcommand{\bOmega}{\bm{\Omega}}
\newcommand{\bPsi}{\bm{\Psi}}

\newcommand{\balpha}{\bm{\alpha}}
\newcommand{\bdelta}{\bm{\delta}}
\newcommand{\bomega}{\bm{\omega}}
\newcommand{\bgamma}{\bm{\gamma}}
\newcommand{\bepsilon}{\bm{\epsilon}}
\newcommand{\blambda}{\bm{\lambda}}
\newcommand{\btheta}{\bm{\theta}}
\newcommand{\bphi}{\bm{\phi}}
\newcommand{\bpsi}{\bm{\psi}}
\newcommand{\bmeta}{\bm{\eta}}
\newcommand{\bzeta}{\bm{\zeta}}
\newcommand{\bmu}{\bm{\mu}}
\newcommand{\bnu}{\bm{\nu}}
\newcommand{\bpi}{\bm{\pi}}
\newcommand{\bsigma}{\bm{\sigma}}

\newcommand{\bargamma}{\bar{\gamma}}
\newcommand{\bartheta}{\bar{\theta}}

\newcommand{\tilDelta}{\tilde{\Delta}}
\newcommand{\tlDelta}{\tilde{\Delta}}

\newcommand{\tlepsilon}{\tilde{\epsilon}}
\newcommand{\tltheta}{\tilde{\theta}}
\newcommand{\tlgamma}{\tilde{\gamma}}

\newcommand{\bA}{\mathbf{A}}
\newcommand{\bD}{\mathbf{D}}
\newcommand{\bE}{\mathbf{E}}
\newcommand{\bG}{\mathbf{G}}
\newcommand{\bH}{\mathbf{H}}
\newcommand{\bI}{\mathbf{I}}
\newcommand{\bJ}{\mathbf{J}}
\newcommand{\bL}{\mathbf{L}}
\newcommand{\bM}{\mathbf{M}}
\newcommand{\bN}{\mathbf{N}}
\newcommand{\bP}{\mathbf{P}}
\newcommand{\bQ}{\mathbf{Q}}
\newcommand{\bR}{\mathbf{R}}
\newcommand{\bS}{\mathbf{S}}
\newcommand{\bT}{\mathbf{T}}
\newcommand{\bU}{\mathbf{U}}
\newcommand{\bW}{\mathbf{W}}
\newcommand{\bX}{\mathbf{X}}
\newcommand{\bY}{\mathbf{Y}}
\newcommand{\bZ}{\mathbf{Z}}

\newcommand{\ba}{\mathbf{a}}
\newcommand{\bb}{\mathbf{b}}
\newcommand{\bc}{\mathbf{c}}
\newcommand{\bd}{\mathbf{d}}
\newcommand{\be}{\mathbf{e}}
\newcommand{\mbf}{\mathbf{f}}
\newcommand{\bg}{\mathbf{g}}
\newcommand{\bh}{\mathbf{h}}
\newcommand{\bl}{\mathbf{l}}
\newcommand{\bn}{\mathbf{n}}
\newcommand{\bp}{\mathbf{p}}
\newcommand{\bq}{\mathbf{q}}
\newcommand{\br}{\mathbf{r}}
\newcommand{\bs}{\mathbf{s}}
\newcommand{\bu}{\mathbf{u}}
\newcommand{\bv}{\mathbf{v}}
\newcommand{\bw}{\mathbf{w}}
\newcommand{\bx}{\mathbf{x}}
\newcommand{\by}{\mathbf{y}}
\newcommand{\bz}{\mathbf{z}}

\newcommand{\hbeta}{\hat{\beta}}
\newcommand{\htheta}{\hat{\theta}}
\newcommand{\hsigma}{\hat{\sigma}}
\newcommand{\hmu}{\hat{\mu}}

\newcommand{\hp}{\hat{p}}
\newcommand{\hr}{\hat{r}}
\newcommand{\hs}{\hat{s}}
\newcommand{\hw}{\hat{w}}
\newcommand{\hx}{\hat{x}}

\newcommand{\hN}{\hat{N}}

\newcommand{\hbSigma}{\hat{\bm{\Sigma}}}

\newcommand{\hba}{\hat{\mathbf{a}}}
\newcommand{\hbs}{\hat{\mathbf{s}}}
\newcommand{\hbx}{\hat{\mathbf{x}}}
\newcommand{\hbv}{\hat{\mathbf{v}}}
\newcommand{\hbw}{\hat{\mathbf{w}}}

\newcommand{\hbW}{\hat{\mathbf{W}}}

\newcommand{\dif}{\text{d}}

\newcommand{\bbC}{\mathbb{C}}
\newcommand{\bbE}{\mathbb{E}}
\newcommand{\bbR}{\mathbb{R}}
\newcommand{\bbN}{\mathbb{N}}
\newcommand{\bbZ}{\mathbb{Z}}

\newcommand{\calA}{\mathcal{A}}
\newcommand{\calB}{\mathcal{B}}
\newcommand{\calC}{\mathcal{C}}
\newcommand{\calD}{\mathcal{D}}
\newcommand{\calE}{\mathcal{E}}
\newcommand{\calF}{\mathcal{F}}
\newcommand{\calG}{\mathcal{G}}
\newcommand{\calH}{\mathcal{H}}
\newcommand{\calL}{\mathcal{L}}
\newcommand{\calN}{\mathcal{N}}
\newcommand{\calM}{\mathcal{M}}
\newcommand{\calP}{\mathcal{P}}
\newcommand{\calS}{\mathcal{S}}
\newcommand{\calT}{\mathcal{T}}
\newcommand{\calV}{\mathcal{V}}
\newcommand{\calW}{\mathcal{W}}
\newcommand{\calX}{\mathcal{X}}
\newcommand{\calY}{\mathcal{Y}}

\newcommand{\calhL}{\mathcal{\hat{L}}}

\newcommand{\tlA}{\tilde{A}}
\newcommand{\tlC}{\tilde{C}}
\newcommand{\tlD}{\tilde{D}}

\newcommand{\tlf}{\tilde{f}}
\newcommand{\tlv}{\tilde{v}}
\newcommand{\tls}{\tilde{s}}
\newcommand{\tlw}{\tilde{w}}
\newcommand{\tlx}{\tilde{x}}
\newcommand{\tly}{\tilde{y}}
\newcommand{\tlz}{\tilde{z}}

\newcommand{\barb}{\bar{b}}
\newcommand{\barm}{\bar{m}}
\newcommand{\barn}{\bar{n}}
\newcommand{\barr}{\bar{r}}
\newcommand{\barv}{\bar{v}}
\newcommand{\barx}{\bar{x}}
\newcommand{\bary}{\bar{y}}
\newcommand{\barz}{\bar{z}}

\newcommand{\barA}{\bar{A}}
\newcommand{\barC}{\bar{C}}
\newcommand{\barD}{\bar{D}}
\newcommand{\barH}{\bar{H}}
\newcommand{\barK}{\bar{K}}
\newcommand{\barL}{\bar{L}}
\newcommand{\barV}{\bar{V}}
\newcommand{\barW}{\bar{W}}
\newcommand{\barX}{\bar{X}}
\newcommand{\barZ}{\bar{Z}}

\newcommand{\barba}{\bar{\ba}}
\newcommand{\barbe}{\bar{\be}}
\newcommand{\barbg}{\bar{\bg}}
\newcommand{\barbh}{\bar{\bh}}
\newcommand{\barbx}{\bar{\bx}}
\newcommand{\barby}{\bar{\by}}
\newcommand{\barbz}{\bar{\bz}}

\newcommand{\barbA}{\bar{\bA}}

\newcommand{\tlbA}{\tilde{\bA}}
\newcommand{\tlbE}{\tilde{\bE}}
\newcommand{\tlbG}{\tilde{\bG}}
\newcommand{\tlbW}{\tilde{\bW}}
\newcommand{\tlbX}{\tilde{\bX}}
\newcommand{\tlbY}{\tilde{\bY}}

\newcommand{\tlbf}{\tilde{\mbf}}
\newcommand{\tlbg}{\tilde{\bg}}
\newcommand{\tlbv}{\tilde{\bv}}
\newcommand{\tlbw}{\tilde{\bw}}
\newcommand{\tlbx}{\tilde{\bx}}
\newcommand{\tlby}{\tilde{\by}}
\newcommand{\tlbz}{\tilde{\bz}}

\newcommand{\tc}{\text{c}}
\newcommand{\td}{{\text{d}}}
\newcommand{\ter}{{\text{r}}}
\newcommand{\ts}{{\text{s}}}
\newcommand{\tw}{{\text{w}}}

\newcommand{\bzero}{\mathbf{0}}
\newcommand{\bone}{\mathbf{1}}

\newcommand{\suml}{\sum\limits}
\newcommand{\minl}{\min\limits}
\newcommand{\maxl}{\max\limits}
\newcommand{\infl}{\inf\limits}
\newcommand{\supl}{\sup\limits}
\newcommand{\liml}{\lim\limits}
\newcommand{\intl}{\int\limits}
\newcommand{\ointl}{\oint\limits}
\newcommand{\bigcupl}{\bigcup\limits}
\newcommand{\bigcapl}{\bigcap\limits}

\newcommand{\opconv}{\text{conv}}

\newcommand{\eref}[1]{(\ref{#1})}

\newcommand{\sinc}{\text{sinc}}
\newcommand{\tr}{\text{Tr}}
\newcommand{\var}{\text{Var}}
\newcommand{\cov}{\text{Cov}}
\newcommand{\tth}{\text{th}}
\newcommand{\proj}{\text{proj}}

\newcommand{\nwl}{\nonumber\\}

\newenvironment{vect}{\left[\begin{array}{c}}{\end{array}\right]}
\newtheorem{theorem}{Theorem}
\newtheorem{lemma}{Lemma}
\theoremstyle{definition}
\newtheorem{definition}{Definition}
\newtheorem{assum}{Assumption}
\newtheorem{cor}{Corollary}

\theoremstyle{definition}
\newtheorem{assump}{Assumption}

\theoremstyle{definition}
\newtheorem{remark}{Remark}

\newcommand{\polylog}{\mathrm{polylog}\ }
\newcommand{\bF}{\bm{F}}
\newcommand{\bPhi}{\bm{\Phi}}
\newcommand{\bV}{\bm{V}}
\newcommand{\bB}{\bm{B}}
\newcommand{\bC}{\bm{C}}
\newcommand{\bbeta}{\bm{\beta}}
\newcommand{\diag}[1]{\mathrm{diag}\lbrace #1 \rbrace}
\newcommand{\bxi}{\bm{\xi}}
\newcommand{\bTheta}{\bm{\Theta}}
\newcommand{\uk}{^{(k)}}
\newcommand{\ukT}{^{(k)T}}
\newcommand{\norm}[1]{\left\|#1\right\|}
\newcommand{\isp}{\frac{1}{\sqrt{p}}}
\newcommand{\bXi}{\bm{\Xi}}
\newcommand{\bPi}{\bm{\Pi}}
\newcommand{\bUpsilon}{\bm{\Upsilon}}
\newcommand{\bvarphi}{\bm{\varphi}}
\newcommand{\barphi}{\bar{\phi}}
\newcommand{\barbR}{\bar{\bR}}
\newcommand{\bt}{\bm{t}}
\newcommand{\bK}{\bm{K}}
\newcommand{\op}{\mathrm{op}}
\newcommand{\hbd}{\hat{\bd}}
\newcommand{\whp}{\quad\mathrm{w.h.p}}
\newcommand{\hbtheta}{\hat{\btheta}}
\newcommand{\poly}{\mathrm{poly}}

\maketitle

\begin{abstract}%
We provide exact asymptotic expressions for the performance of regression by an $L-$layer deep random feature (RF) model, where the input is mapped through multiple random embedding and non-linear activation functions. For this purpose, we establish two key steps: First, we prove a novel universality result for RF models and deterministic data, by which we demonstrate that a deep random feature model is equivalent to a deep linear Gaussian model that matches it in the first and second moments, at each layer. Second, we make use of the convex Gaussian Min-Max theorem multiple times to obtain the exact behavior of deep RF models.  We further characterize the variation of the eigendistribution in different layers of the equivalent Gaussian model, demonstrating that depth has a tangible effect on model performance despite the fact that only the last layer of the model is being trained. 
\end{abstract}

\section{Introduction}\label{sec:introduction}
Recent experimental and theoretical results \citep{zhang2021understanding, belkin2019reconciling} have demonstrated that the classical understanding of overparameterized machine learning (ML) models requires further examination. One model that has been studied extensively is the random features (RF) mode \citep{rahimi2007random}, which is closely related to overparameterized neural networks \citep{daniely2016toward, daniely2017sgd, jacot2018neural, liu2021random, bach2017equivalence}. In this paper, we examine an extension of the RF model, which we call the deep RF (DRF) model, being equivalent to a deep NN, but only trained in the output layer. We consider the asymptotic regime, where the number of data points, model parameters, and input dimension grow infinite at constant ratio \citep{belkin2020two, hastie2019surprises, bartlettLongLugosiTsigler_2020, Bartlett30063} and give exact expressions that characterize the deep RF model in terms of training and generalization error.

Our analysis consists of two key steps. First, we prove universality, i.e. we demonstrate that the DRF model is asymptotically equivalent to a deep Gaussian surrogate model, matching the original model in the first and second moments, at each layer \citep{panahi2017universal, oymak2018universality}. Universality for the 1-layer RF model has previously been proven, e.g. in \citep{HuUniversalityLaws}. We make use of a different proof technique to extend these results to arbitrary many layers and introduce a new Gaussian surrogate model for DRF. This universality result alleviates the general difficulty of analyzing RF or DRF models, as the non Gaussian features are in general not amenable to stardard analysis techniques such as comparison theorem \citep{gordon1985some, thrampoulidis2015gaussian}, Gaussian widths \citep{chandrasekaran2012convex} or replica methods \citep{mezard1987spin}.

Having established universality, we then make use of the Convex Gaussian Min Max Theorem (CGMT) \citep{thrampoulidis2015gaussian} to study DRFs. This theorem allows us to consider an alternative optimization problem with the same asymptotic statistics, and is a popular tool in the analysis of the asymptotic regime \citep{bosch2021double, chang2020provable, dhifallah2020precise, thrampoulidis2015regularized, Loureiro2021Learning, Bosch2022DoubleDescent}. We make use of a recursive application of the CGMT \citep{Bosch2022DoubleDescent} to obtain asymptotic expressions for square loss functions with arbitrary convex regularization for $L$-layer DRF models.

\section{Related Works}\label{sec:relatedworks}

The random features (RF) \citep{rahimi2007random} model has been extensively examined in the asymptotic regime, under a multitude of conditions. For an incomplete list see \citep{hastie2019surprises, mei2019generalization, montanari2019generalization, goldt2020gaussian, goldt2019modeling, gerace2020generalisation, dhifallah2020precise, ghorbani2021linearized, Bosch2022DoubleDescent}. In the case of ridge regression \citep{louart2018random, mei2019generalization} exact expression for the training and generalization error can be established. In other cases, exact analysis is difficult. It was observed by many authors \citep{mei2019generalization, hastie2019surprises, goldt2019modeling, gerace2020generalisation, goldt2020gaussian} that a Gaussian surrogate model that matched the first and second moments had asymptotically equivalent statistics. A concrete proof of RF universality is given in \cite{hu2022universality}. 
We utilize Lindeberg's approach \citep{lindeberg1922neue} to demonstrate universality of DRF. This approach has been used to prove universality results in many other optimization problems \citep{korada2011applications,panahi2017universal, montanari2017universality,oymak2018universality, abbasi2019universality}. \cite{hu2022universality} prove a central limit theorem between random features and their Gaussian equivalent features as a key step in demonstrating universality. We make use of a different proof technique, by instead considering the problem in a dual space, where we may directly bound the difference between the leave one out iterates.

Beside RF,  universality has been demonstrated for many other models \citep{goldt2020gaussian, seddik2020random, dhifallah2021inherent, loureiro2021capturing, gerace2022gaussian}. Recently, \cite{montanari2022universality} gave a proof for the universality of empirical risk minimization for not necessarily convex loss and regularization functions. Their result also assume that a central limit theorem similar to \citep{hu2022universality} holds.

Subject to Gaussian features, the CGMT \citep{gordon1985some, thrampoulidis2015gaussian} is a powerful tool in the determination of the asymptotic performance \citep{Loureiro2021Learning, dhifallah2020precise, thrampoulidis2015regularized, chang2020provable, bosch2021double, Bosch2022DoubleDescent}. The CGMT determines an alternative, asymptotically equivalent optimization problem in statistical properties. In the case of correlated features, such as in the RF or DRF model, the alternative optimization still remains intractable. This issue is resolved in \citep{Bosch2022DoubleDescent} by applying the CGMT twice. Relying on the particular structure of the DRF covariance matrices, we extend the method of \citep{Bosch2022DoubleDescent} where the CGMT is applied    recursively to determine a nested scalar optimization that is asymptotically equivalent to the DRF model.

The covariance matrices for the Gaussian surrogate model that we obtain are similar in structure to the kernel matrices given in \cite{lee2017deep}. The authors demonstrate an exact equivalence between an infinitely wide deep NN and a Gaussian Process with covariance kernels that are recursively defined in a similar manner to the ones discussed in this paper. However, \citep{lee2017deep} consider networks of fixed size but infinite width, while we consider the asymptotic regime, where the number of data points and the input dimensions grow as well, hence maintaining a relatively narrower network.

\subsection{Paper Outline}\label{sec:introduction:subsec:PaperOutline}
In section \ref{sec:setup}, we introduce the DRF problem and its Gaussian surrogate, and express the necessary assumptions for our results to hold.
In section \ref{sec:Universality}, we prove the main universality theorem of this paper. Our proof takes two steps, first proving universality of a single layer, and subsequently using an inductive argument to extend this result to a full DRF problem.
In section \ref{sec:CGMTanalysis}, we give an alterative scalar optimization problem derived by means of the CGMT, that is asymptotically equivalent to the DRF problem subject to square loss and arbitrary, strongly convex regularization. We demonstrate experimentally the veracity of the determined expressions.

\section{Setup and Assumptions}\label{sec:setup}
\subsection{Random Feature Model and Preliminaries}
We consider a supervised learning setup with a dataset $\calD = \lbrace(\bx_k, y_k) \in \bbR^{d}\times \bbR\rbrace_{k=1}^n$. To find a relationship between the data points $\bx_k$ and the labels $y_k$, we consider function of the following form 
\begin{eqnarray}\label{eq:GenericFeatureMap}
    Y_{\btheta}(\bx) = \frac{1}{\sqrt{p}}\btheta^T \calF(\bx), \qquad \btheta\in\bbR^p,
\end{eqnarray}
where $\calF:\bbR^d \rightarrow \bbR^p$ is a given mapping of the data, called a \textit{feature map}. We note that $Y_{\btheta}$ is dependent upon the choice of the vector $\btheta$ by a linear relation. This shows the main advantage of \eqref{eq:GenericFeatureMap}: while $Y_{\btheta}$ can represent nonlinear functions, selecting $\btheta$ amounts to a linear regression task. To find the optimal value of $\btheta$, denote $\mbf_i=\calF(\bx_i)$ and take $\bF$ as a matrix with $\mbf_i$ as columns. We consider the empirical risk minimization framework and the following optimization problem:
\begin{equation} \label{eq:OptimalPoint}
    \hat\btheta=\hat{\btheta}(\bF) = \arg\min_{\btheta \in \bbR^p} \frac{1}{n}\sum_{i = 1}^n \ell\left(\frac{1}{\sqrt{p}}\btheta^T\mbf_i, y_i\right) + R(\btheta).
\end{equation}
 Here, $\ell(\cdot, \cdot)$ is a loss function, and $R(\btheta)$ is a regularization function. To measure the performance of $\hat{\btheta}$ we make use of the two common metric for supervised learning, that being the training error $\calE_{train}(\bF)$, i.e the optimal value in \eqref{eq:OptimalPoint}, and
the generalization error
\begin{equation}\label{eq:GeneralizationLoss}
    \calE_{gen}(\bF) = \bbE\left[\ell\left(  \frac{1}{\sqrt{p}}\hat{\btheta}^T\calF(\bx_{new}),y_{new}\right)\right],
\end{equation}
where the expectation is taken over $(\bx_{new}, y_{new})$, a new datapoint drawn from the same distribution as the dataset $\calD$. 

The main purpose of this paper is to obtain exact asymptotic expressions for the supervised learning metrics $\calE_{train},\calE_{gen}$ and other properties of $\hbtheta$, when the feature map $\calF$ is a deep random feature, generalizing the \textit{random features maps} \citep{rahimi2007random}. To define the deep features, we remind the (shallow) random features are given by $\phi(\bx, \bw_j):= \sigma(\bw_j^T\bx)$, for $j=1,2,\ldots,p$, where $\sigma$ is an activation function, and $\bw_j \sim \calN(0, \frac{1}{d}\bI_d)$ are a set of random weights. In vector form, we express these relations as $\bphi(\bx, \bW) = (\sigma(\bw_j^T\bx))_{j=1}^p$,
where the matrix $\bW$ has rows $\bw_j$. Then, the deep random features are
given through the following recursion: For $p_0 = d, p_1, p_2, \ldots,p_L \in \bbN$, we define the matrices $\bW^{(l)} \in \bbR^{p_l\times p_{l-1}}$ for $l= 1, \ldots, L$, each having independent rows $\bw^{(l)}_j \sim \calN(0, \frac{1}{p_{l-1}}\bI)$. Letting $\bx^{(0)}:= \bx$ we define
\begin{equation}\label{eq:RecursiveRandomFeatures}
    \bx^{(l)} = \bphi(\bx^{(l-1)}, \bW^{(l)}) = (\sigma(\bw_j^{(l)T}\bx^{(l-1)}))_{j=1}^{p_l}, \qquad l = 1, \ldots, L,
\end{equation}
\subsection{Necessary Assumptions}
Our results rely on the following assumptions:

\begin{enumerate}
    \item[A1] For some universal positive constants $\mu, M$, the regularization function $R$ is $\mu$-strongly convex and $M$-smooth with $M$-bounded third derivative in tensor (operator) norm. Moreover $R(\btheta)$ is minimized at $\btheta = \bm{0}$
    \item[A2] $\ell$ is a $\frac{1}{Cn}-$ strongly convex function in the first argument and its third derivative with respect to the first argument is bounded by $Cn$ for some constant $C$. Moreover, there exists a vector $\balpha = (\alpha_k)$ called \textit{isolated predictions} satisfying: $ \alpha_k \in \arg\min_{\alpha} \ell(\alpha, y_k)$, and $\norm{\balpha}_2 \leq C\sqrt{n}$ for a fixed constant $C$.
    \item[A3] The activation function $\sigma$ is an odd function applied element wise, with bounded derivatives. Furthermore, let $g_1, g_2$ be Gaussian variables distributed as
    \begin{equation}
        \begin{bmatrix}
        g_1\\
        g_2
        \end{bmatrix} \sim \calN\left(\bm{0}, \begin{bmatrix}
            \alpha_1 & \rho \\
            \rho & \alpha_2
        \end{bmatrix}\right).
    \end{equation}
    Let the functions $\eta_1(\alpha_1, \alpha_2, \rho) = \bbE[\sigma(g_1)\sigma(g_2)]$ and $\eta_2(\alpha_1) = \bbE[\sigma^2(g_1)]$. Then $\eta_1, \eta_2$ should be thrice differentiable at $\alpha_1 = \alpha_2 = 1$ and $\rho = 0$
    \item[A4] The dimensions of the number of data points $n$, the size of the input $p_0 = d$ and the size of subsequent layers $p_l$, where $l=1,\ldots, L$ all grow to infinity at fixed ratios. We denote this by $n\sim p_0\sim\cdots\sim p_L$, where $a \sim b$ is defined to mean that $\frac{a}{b} \xrightarrow[a, b\rightarrow\infty]{} C$ for some constant $C$.
    \item[A5] For each layer, $l$, the weight matrix $\bW^{(l)}\in\bbR^{p_l \times p_{l-1}} = [\bw_1^{(l)}\ \bw_2^{(l)}\ \cdots\ \bw_{p_l}^{(l)}]^T$ are independent Gaussian variables $\bw_i^{(l)} \overset{i.i.d}{\sim}\calN(0, \frac{1}{p_{l-1}}\bI_{p_{l-1}})$ for $1\leq i \leq p_l$. Furthermore, $\bW^{(l)}$ are independent of the input variables $\bx$. 
\end{enumerate}

\begin{remark} For assumption 2, we note that the strong convexity assumption on the loss function becomes less restrictive as $n$ grows. In the asymptotic limit, the strong convexity is no longer a significant requirement. Furthermore, the isolated prediction vectors exist and the condition is satisfied immediately if the loss function is minimized at the labels, i.e. it is minimized at the point $\ell(y_k, y_k) < \infty$.
\end{remark}

\begin{remark}
For assumption 3, we note that the condition holds for the majority of loss function used in practice including $\tanh$ and the error function. Furthermore, if oddness is dropped, the assumption on the functions $\eta_1$ and $\eta_2$ are additionally satisfied for functions like ReLU, sigmoids, and Gaussian activations. However, we require oddness.   
\end{remark}

Finally we impose a condition upon the input vectors $\bx_i$:

\begin{definition}\label{def:regularity}
Let $d\sim n$, we call a set $\lbrace \bx_k\in\bbR^{d} \rbrace_{k=1}^n$ \textit{regular} if
\begin{enumerate}
    \item Letting $\bX = [\bx_1\ \bx_2\ \cdots\ \bx_n]$, there is a constant $c<\infty$ such that $\frac{1}{\sqrt{n}}\norm{\bX}_{\op} < c$
    \item It holds that
    \begin{equation}
        \max_{i, j}\left|\frac{1}{d}\bx_i^T\bx_j - \delta_{ij} \right| \leq \frac{\polylog n}{\sqrt{n}},
    \end{equation}
    where $\delta_{ij}$ is the Kronecker delta.
\end{enumerate}
\end{definition}
Note that the first condition for regularity is trivially satisfied for finite $n$, however the condition must also hold for a fixed $c$ in the asymptotic limit. Further, note that regularity is exhibited by $\bx$ being Gaussian with high probability.

\section{Universality}\label{sec:Universality}
In the case of a single layer, it has been proven \citep{HuUniversalityLaws} that the following Guassian feature map has asymptotically equivalent statistics to the random features given in section~\ref{sec:setup}
\begin{equation}
    \Tilde{\bphi}(\bx, \bW) = \rho_1\bW\bx + \rho_2\bg,
\end{equation}
where $\bg$ is a standard normal vector, and $\rho_1, \rho_2$ are constants depending only on the activation function, given by
$
     \rho_1  = \bbE[\sigma'(z)], \ \rho_2 = \sqrt{\bbE[\sigma^2(z)] - \rho_0^2 + \rho_1^2},\  z\sim\calN(0, 1).
$
Similarly we define a \textit{deep Gaussian equivalent feature map}, recursively. We define $\bgamma^{(0)} = \bx$, and then define
\begin{equation}\label{eq:GaussianRandomFeatures}
    \bgamma^{(l)} = \Tilde{\bphi}_l(\bgamma^{(l-1)}, \bW^{(l)}) := \rho_{1,l}\bW^{(l)}\bgamma^{(l-1)} + \rho_{2,l}\bg^{(l)}, \qquad l = 1, \ldots, L,
\end{equation}
where $\bg^{(l)}$ is an independent standard normal vector of dimension $p_l$, and the constants $\rho_{1,l},\rho_{2,l}$ are recursively defined as
$    \rho_{1,l} = \bbE[\sigma'(\alpha_{l-1}z)], \ \rho_{2,l} = \sqrt{\bbE[\sigma^2(\alpha_{l-1}z)] - \alpha_{l-1}^2\rho_{1,l}^2}$, $ \ z\sim\calN(0, 1)$.
Here, $\alpha_l$ are constants given by the following recursive definition:
 $       \alpha_0 = 1, \  \alpha_{l} = \sqrt{\rho_{1,l}^2\alpha_{l-1}^2 + \rho_{2,l}^2}$.
Now, we consider the following two feature mappings for an input vector $\bx^{(0)}$:
\begin{eqnarray}\label{eq:featureMapsDef}
    \calF(\bx^{(0)}) = \bx^{(l)},\qquad \calG(\bx^{(0)}) = \bgamma^{(l)},
\end{eqnarray}
where $\bx^{(l)}$ and $\bgamma^{(l)}$ are given in \eqref{eq:RecursiveRandomFeatures} and  \eqref{eq:GaussianRandomFeatures}, respectively.

\subsection{Revisiting Universality of a Single Layer}\label{subsec:UniversalitySingleLayer}
The proof of universality of deep random features is a specific application of a universality theorem for a single layer, which we derive in this section. This result is more general than the previous studies such as \cite{hu2022universality}. In the subsequent section, we shall demonstrate how the universality of deep random features follows from these results.

Let $\phi_{j}:\bbR^{d}\times \Omega_j\rightarrow \bbR$ for $j=1,2,\ldots, p$ be a random feature map, where $\Omega_j$ is a sample space equipped with an arbitrary probability measure, such that $\phi_j(\cdot, \omega)$ for any $\omega\in \Omega_j$ is a particular realization of the feature map. Let $\Omega = \Omega_1\times \Omega_2 \times \cdots \Omega_p$ be a product space equipped with the product measure, and let $\bphi:\bbR^{d}\times \Omega\rightarrow \bbR^p$ represent the vector of random features, such that $\bphi(\bx, \bomega) = (\phi_{j}(\bx, \omega_j))_j$, where $\bomega = (\omega_j)\in \Omega$ is a realization. 

Next we consider a $m\times p$ matrix $\bD$ with columns $\bd_j\in\bbR^m$, which we call a \textit{synthesis dictionary}. We define the re-represented random feature vectors $\mbf:\bbR^d\times \Omega \rightarrow \bbR^{m}$ given by
\begin{equation}\label{eq:RerepresentedRF}
    \mbf(\bx, \bomega):= \sum_{j=1}^p\bd_{j}\phi_j(\bx, \omega_j) = \bD\bphi(\bx, \omega).
\end{equation}
We note that if $m = p$ we can choose $\bD = \bI_p$ and retain the original set of random features. However, re-representing the features is necessary for the proof of the deep random features case. We will drop the argument $\bomega$ when there is no risk of confusion, and denote $\bphi(\bx)$, $\mbf(\bx)$ as the random features and their re-representation. Similarly let $\bphi_k = \bphi(\bx_k)$ and $\mbf_k = \mbf(\bx_k)$ for $k = 1,\ldots, n$ which are random vectors. Finally, let $\bPhi$ and $\bF$ be the matrices with $(\bphi_k), (\mbf_k)$ as columns. We assume that the random features are centered:
\begin{equation}\label{eq:RFcentered}
    \bbE_{\bomega}[\bphi(\bx_k, \bomega)] = \bm{0}\qquad k =1, 2, \ldots, n.
\end{equation}
We further define the data kernel matrices $\bK_{j} = (K_{j, kl})_{kl}$ where $\bK_j$ is the covariance matrix of the $j$th row of $\bPhi$, given by
\begin{equation}\label{eq:RFcovarianceMatrices}
    K_{j, kl} = \bbE_{\omega_j}[\phi(\bx_k, \omega_j)\phi(\bx_l, \omega_j)].
\end{equation}

Next, we introduce a $p\times n$ Gaussian matrix $\bGamma$ with independent rows, and where the $j$th row is distributed by $\calN(0, \bK'_j)$. We note that if $\bK_j = \bK_j'$ that $\bPhi$ and $\bGamma$ have the same first and second moments amongst their elements. We then define $\bG = \bD\bGamma$ and let $\bg_k$ be the $k$th column of $\bG$.

Before stating the main theorem for this section we state the conditions on the dataset and matrices $\bK_j$ and $\bD$ that must hold. We shall show in the next section that these conditions are satisfied in the case of deep random features. We remind the reader of the definition of a sub-Gaussian vector:

\begin{definition}\label{def:subgaussianity}
    We say that a random vector $\bu = (u_k)\in \bbR^n$ is $\tau-$sub-Gaussian if for any unit vector $\ba = (a_k)\in\bbR^n$ the variable $A = \ba^T\bu$ is sub-Gaussian with parameter $\tau$, i.e.
     $   \bbE\left[e^{\lambda A} \right] \leq e^{\frac{\tau^2\lambda^2}{2}}$
    for all $\lambda \in \bbR$.
\end{definition}
We state the following requisite conditions:
\begin{enumerate}
    \item[B1] There exists a positive constant $C$ such that for all $j$, it holds that $\norm{\bK_j}_{\op}\leq C$ and the $j$th random feature vector $\bphi^j = \lbrace \phi(\bx_k, \omega_j) \rbrace_{k}$ is $C$-sub-Gaussian
    \item[B2] There exists a positive constant $C$ such that $\norm{\bD}_{\op} \leq C$.
\end{enumerate}
These assumptions must hold for all values of $n, d, p, m$ and must continue to hold when they grow asymptotically.
Subject to these conditions we state the following theorem that demonstrates universality.
\begin{theorem}\label{thm:GenericUniversality}
    Suppose that assumptions A1, A2, B1 and B2 hold, and that $n\sim p \sim m$. Then,
    \begin{enumerate}
    \item For any real function $\psi$ with bounded first, second, and third derivatives, there exists a constant $c < \infty$ such that
    \begin{equation}
        \left|\bbE\psi(\calE_{train}(\bF)) - \bbE\psi(\calE_{train}(\bG)) \right| \leq  \frac{c}{n}\sum_{j =1}^p\norm{\bK_j - \bK_j'}_{\op} + \frac{c}{\sqrt{n}}.
    \end{equation}
    \item Let $\hat{\btheta}_F$ and $\hat{\btheta}_G$ be the optimal points for the optimization \eqref{eq:OptimalPoint} for $\bF$ and $\bG$ respectively. For any bounded function $h:\bbR^{m}\rightarrow\bbR$ with bounded second and third derivatives (in tensor norm), where the bounds are constant in $n, p, m$. There exists a constant $c<\infty$ such that
    \begin{equation}
        \left|\bbE h\left(\hat{\btheta}_F\right) - \bbE h\left(\hat{\btheta}_G\right) \right| \leq \frac{c}{n}\sum_{j =1}^p\norm{\bK_j - \bK_j'}_{\op} + \frac{c}{\sqrt{n}}.
    \end{equation}
    \end{enumerate}
\end{theorem}

\subsubsection{Proof Sketch}\label{sec:UniversalityTheorem:subsec:proof}
The proof is based on an application of Lindebergs argument with respect to the random features $\mbf$ in a dual space. We consider the optimization problem given in \eqref{eq:OptimalPoint} for some generic map $\bZ$ and note that by means of a splitting argument it may be expressed as
\begin{eqnarray}
    \calE_{train}(\bZ) = \min_{\btheta\in\bbR^{p}}\frac{1}{n}\sum_{k=1}^{n}\ell(\bz_k^T\btheta, y_k) + R(\btheta)\nwl
    = \min_{\btheta\in\bbR^{p}, \balpha\in\bbR^{n}}\max_{\bd\in\bbR^{n}}\frac{1}{n}\left(\sum_{k=1}^n\ell(\alpha_k, y_k) + d_k(\alpha_k -\bz_k^T\btheta) \right) + R(\btheta)\nwl
    = - \min_{\bd\in\bbR^n}\frac{1}{n}\sum_{k=1}^n \ell^*(-d_k, y_k) + R^*\left(\frac{1}{n}\bZ\bd\right)
\end{eqnarray}
where $\ell^*$ and $R^*$ are the legendre transforms of $\ell$ and $R$ respectively. We note that by assumption A1 that $R^*$ is $\frac{1}{M}-$ strongly convex and $\frac{1}{\mu}$-smooth. We then proceed in defining a series of $\bZ_r$ such that $\bZ_0 = \bPhi$ and $\bZ_p = \bGamma$. We show that the difference of the optimal value in the dual space between $\psi(\calE_{train}(\bZ_r))$ and $\psi(\calE_{train}(\bZ_{r+1}))$ is bounded by the sum of a $O(\frac{1}{n^{3/2}})$ term and the difference in operator norm between $\frac cn\norm{\bK_r - \bK'_r}_{op}$, which allows us to bound the total difference as in the given result. 

For part two, we note that $R_{\epsilon}(\btheta) = R(\btheta) \pm \epsilon h(\btheta)$ remains strongly convex for sufficiently small values of $\epsilon > 0$. As such, part 1 of the theorem holds for these cases. By bounding the difference in the values of $\calE_{train}$ at $\epsilon>0$ and at $\epsilon = 0$ the bound on $h(\btheta)$ may be obtained. The proof is given in full in appendix \ref{app:sec:UniversalityTheoremProof}.

\subsection{Multiple Layers}\label{subsec:multipleLayerUniversality}
In this section, we apply the results of the previous section to prove universality for DRF.
We shall consider the deep random features as given in eq \eqref{eq:featureMapsDef}. The proof of the equivalence relies on fixing all layers, except a single one, and demonstrating that the individual layer may be replaced by their Gaussian equivalent. This relies on an intermediate result, given in the following theorem, stating that the regularity of a dataset, as defined in definition \ref{def:regularity} is preserved under random feature mappings.
\begin{theorem}\label{thm:PreservationOfRegularity}
    Suppose that the set $\lbrace\bx_i\in\bbR^d \rbrace_{i=1}^n$ is regular and assumption 3 holds. Then define $\bz_i = \sigma(\bW\bx_i)$ where $\bW$ is a $p\times d$ matrix and has independent rows distributed by $\calN(0, \frac{1}{d}\bI)$. Then with probability higher than $1 - n^{-10}$ the set $\lbrace \bz_i\rbrace_{i =1}^n$ is regular\footnote{The exponent of $n$ is arbitrary and can be replaced by any other number}.
\end{theorem}
The main consequence of this theorem is that for $L$ layers where $n\sim p_0\sim p_1\sim \cdots \sim p_L$, with a probability converging to 1, all dataset $\bX^{(l)} = \lbrace \bx_i^{(l)} \rbrace_{i=1}^n$ for $l=1,\ldots, L$ are regular\footnote{Here we assume that the numbers of layers $L$ is fixed, but it is simple to show that the argument also holds for $L=\poly(n)$}. Now, we can state the main result of this section, which demonstrates a slightly more generic version of universality for an $l$-layered deep random feature model.

\begin{theorem}\label{thm:DeepRFUniversality}
Suppose that $n \sim p_0 \sim \cdots \sim p_l$ and take $q = \mathcal{O}(n)$ and let assumption A1-A5 hold. For a fixed final layer $l$, define noise appended features $\Tilde{\bx}^{(l)}_i, \tilde{\bgamma}_i^{(l)}$ as
\begin{equation}
    \tilde{\bx}_i^{(l)} = \begin{bmatrix}
        \bx_i^{(l)}\\
        \bv_i^{(l)}
    \end{bmatrix}\qquad 
    \tilde{\bgamma}_i^{(l)} = \begin{bmatrix}
        \bgamma_i^{(l)}\\
        \bv_i^{(l)}
    \end{bmatrix}
\end{equation}
where $\bv_i^{(l)}\in\bbR^{q}$ are independent standard Gaussian vectors. Take a $m \times (q+p_l)$ dictionary $\bD$, where $\norm{\bD}_{\op} < c$ for some constant $c<\infty$ and define the re-represented features
\begin{equation}
    \mbf_i = \bD\Tilde{\bx}_i, \quad \bg_i = \bD\tilde{\bgamma}_i
\end{equation}
and let $\bF = [\mbf_1\ \cdots\ \mbf_n]$ and let $\bG = [\bg_1\ \cdots\ \bg_n]$ be their matrix representations. Then under the assumption that $\bX$ is regular,
\begin{enumerate}
    \item For any real function $\psi$ with bounded first, second, and third derivatives, there exists a constant $c < \infty$ such that
    \begin{equation}
        \left|\bbE\psi(\calE_{train}(\bF)) - \bbE\psi(\calE_{train}(\bG)) \right| \leq \frac{\polylog n}{\sqrt{n}}
    \end{equation}
    \item Let $\hat{\btheta}_F$ and $\hat{\btheta}_G$ be the optimal solution of problem \eqref{eq:OptimalPoint} for $\bF$ and $\bG$. Then for any bounded function $h:\bbR^{p_{L}}\rightarrow\bbR$ with bounded second and third derivatives (in tensor norm), where the bounds are constant in $n, m, p_i$ for $0\leq i \leq l$. There exists a constant $c<\infty$ such that
    \begin{equation}
        \left|\bbE h\left(\hat{\btheta}_F\right) - \bbE h\left(\hat{\btheta}_G\right) \right| \leq \frac{\polylog n}{\sqrt{n}}
    \end{equation}
\end{enumerate}
\end{theorem}

Universality of the DRF problem follows directly from this theorem by choosing the final $L$th layer,  $q = 0$ and, hence adding no additional noise and $\bD = \bI_{p_L}$ such that no re-representation appears.

\subsubsection{Proof Sketch}\label{sec:RFuniversality:subsec:Proof}
The proof proceeds by means of induction. For the case that $l = 0$, ie a zero layer network the proof is immediate as $\bx^{0} = \bgamma^{(0)}$. Assuming that the induction hypothesis holds for a layer $l-1$ we may consider layer $l$.

We make use of an intermediate results which may be found in the appendix. In theorem~\ref{thm:ReplacemetnOfCovarianceMatrix} we show that if the data set $\bx^{(l)}$ is regular then covariance matrices of $\bx^{(l)}$ and $\bgamma^{(l)}$ are bounded by $\frac{c\polylog n}{\sqrt{n}}$ for some constant $c$. Then, the proof proceeds in two steps: First, we consider an intermediate vector 
\begin{eqnarray}
    \bar{\bgamma}^{(l)}_i = \begin{bmatrix}
        \rho_{1, l}\bW^{(l)}\bx_i^{(l-1)} + \rho_{2,l}\bh_i^{(l)}\\
        \bv_i^{(l)}
    \end{bmatrix}.
\end{eqnarray}
We bound the performance difference ($\calE_{train}$) between $\bx^{(l)}$ and $\bar{\bgamma}^{(l)}$  by  theorem \ref{thm:DeepRFUniversality}. Second, we observe that the difference in performance  between $\bar{\bgamma}^{(l)}$ and $\bgamma^{(l)}$ depends only on the difference between $\bx^{(l-1)}$ and $\bgamma^{(l-1)}$. As such, we may make use of the induction hypothesis to bound this difference. The full proof is given in appendix~\ref{app:sec:DeepRFUniversalityProof}.

\section{CGMT Analysis}\label{sec:CGMTanalysis}

Thanks to the universality results, we only require to analyze the deep Gaussian features $\bgamma_L$. Here, we present this analysis in one particular case where $\ell$ is the square loss, and the regularization function is generic. Additionally, we need to impose a model for the relationship between the labels $\by$ and the input variables $\bx^{(0)}$, which we specifically assume to be independent standard normal vectors. For this we make the following definition
\begin{equation}\label{eq:ymodel}
    y_i = \bx^{(L)T}_i\btheta^* + \nu_i,
\end{equation}
where $\btheta^*\in\bbR^{p_L}$ is the "true" relationship between the data and the parameters, $\nu_i\sim\calN(0, \sigma^2_{\bnu}\bI)$ is noise, and $\bx^{(L)}$ is defined in \eqref{eq:RecursiveRandomFeatures}. We let $\bnu = (\nu_i)_i$ and let $\bX^{(L)} = [\bx^{(L)}_1\ \bx_2^{(L)}\ \cdots \bx_n^{(L)}]$. Then, we consider the following optimization problem

\begin{equation}\label{eq:CGTMP1}
    P_1 = \min_{\btheta}\frac{1}{2n}\norm{\by - \bX^{(L)}\btheta}_2^2 + R(\btheta)=\min_{\be}\frac{1}{2n}\norm{\bnu - \bX^{(L)}\be}_2^2 + R(\btheta^*+\be),
\end{equation}
where $\be=\btheta-\btheta^*$ and the optimal solutions are denoted by $\hat{\btheta}_1,\hat\be_1$. We similarly consider the Gaussian equivalent model defined in eq~\eqref{eq:GaussianRandomFeatures}. In this case, the data is  generated by
\begin{equation}\label{eq:ytildemodel}
    \Tilde{\by}_i = \bgamma^{(L)T}_i\btheta^* + \nu_i.
\end{equation}
Again, we let $\Tilde{\bX}^{(L)} = [\bgamma_1^{(L)}\ \bgamma_2^{(L)}\ \cdots \bgamma_n^{(L)}]$ and define the Guassian equivalent optimization problem as
\begin{equation}\label{eq:CGTMP2}
    P_2 = \min_{\btheta}\frac{1}{2n}\norm{\tilde{\by} - \Tilde{\bX}^{(L)}\btheta}_2^2 + R(\btheta)= \min_{\be}\frac{1}{2n}\norm{\bnu - \Tilde{\bX}^{(L)}\be}_2^2 + R(\btheta^*+\be)
\end{equation}
with corresponding optimal solutions $\hat{\btheta}_2,\hat\be_2$. By applying theorem \ref{thm:DeepRFUniversality} to $\btheta\to\be$ and $\by\to\bnu$, we establish that  the statistics of $P_1$ and $P_2$ become weakly similar in the sense of their distributions.
Furthermore, for this particular choice of the relationship between the data and the labels the generalization error for the problem $P_1$ may be expressed as
\begin{eqnarray}
    \calE_{gen}(\be) = \sigma_{\bnu^2} + \be^T\bbE[\tilde{\bx}_{new}^{(L)T}\tilde{\bx}_{new}^{(L)}]\be.
\end{eqnarray}
This function satisfies the conditions on the function $h(\be)$. As such in this case, the generalization error is also universal.

For problem $P_2$, as the matrix $\Tilde{\bX}$ is Gaussian, it may be analyzed by the CGMT, (see appendix theorem~\ref{app:thm:CGMT}), which gives an asymptotic equivalence to a second alternative problem is follows:

\begin{theorem}\label{thm:CGMTsquareloss}
Let $n\sim p_0\sim \cdots \sim p_L$ and let assumptions 1-5 hold true. Consider the following optimization problem
\begin{eqnarray}\label{eq:CGTMP3}
    P_3 =     \max_{\beta >0 }\min_{q} T_L + \max_{\xi_L > 0, \chi_L > 0} \min_{t_L > 0, k_L > 0} T_{L-1} + \min_{\be} \frac{a}{2pl}\norm{\be}^2 + \frac{b}{p_L}\be^T\bg + R(\btheta + \btheta^*) + \nwl   \max_{\xi_{L-1} > 0, \chi_{L-1} > 0} \min_{t_{L-1} > 0, k_{L-1} > 0} \cdots \max_{\xi_0 \geq 0, \chi_0 \geq 0} \min_{t_0 > 0, k_0 > 0} \sum_{i=1}^{L-2} T_l(\be) 
\end{eqnarray}
Where $T_L$ is a function of $\beta, q$; $T_{L-1}, a, b$ are functions of $\beta, q, \xi_L, \chi_L, t_L, k_L$;  $\bg\in\bbR^{p_L}$ is a standard normal and $T_l$ are functions of $\be , \beta, q, \xi_i, \chi_i, t_i, k_i$ for $L \geq i \geq l$. The exact expressions for the functions $a, b, T_i$ are complicated and are given in the appendix equation \eqref{CGMTFullEquations}. 

Then, 
\begin{enumerate}
    \item Then the values of $P_2$ and $P_3$ become close, in sense that if $P_3$ converges to come value $c$ then $P_2$ will converge to the same value. 

    \item Let $\hat{\btheta}_3$ be the optimal point of $P_3$. Then for any bounded function $h:\bbR^{p_L}\rightarrow \bbR$ with bounded second and third derivatives (in tensor norm), where the bounds are constant in $n, p_i$. for $0\leq i\leq L$, then

    \begin{eqnarray}
        \Pr\left(|h(\hat{\be}_2) - h(\hat{\be}_3)| > \epsilon \right) \rightarrow 0 \quad \mathrm{as}\quad  n,p_0,\ldots,p_L \rightarrow\infty
    \end{eqnarray}

\end{enumerate}
\end{theorem}

A proof of this theorem may be found in the Appendix Section \ref{app:sec:CGMTproof}. Furthermore, if all layers, except the input have the same dimension $p$ the CGMT result can be simplified substantially, these results may be seen in theorem \ref{app:CGMT:allLayersSameSize}. It can be clearly seen that by the triangle inequality and the results of theorem \ref{thm:DeepRFUniversality} that $P_3$ and $P_1$ will similarly asymptotically become weakly similar; as will $h(\hat{\btheta}_3)$ and $h(\hat{\btheta}_1)$.

\subsection{Experimental Results}\label{sec:DepthBetter:subsec:ExperimentalResults}
We now demonstrate the validity of our results experimentally. We consider two regularization functions that satisfy assumption A1: the $\ell_2^2$ regularization and elastic net regularization, where $R(\btheta) = \lambda_1\norm{\btheta}_1 + \frac{1}{2}\lambda_2\norm{\btheta}_2^2$.

We consider standard Gaussian input of dimension $d$ and examine a 2-Layer RF model where both layers are of dimension $p$ and a 1 layer RF model with hidden layer of dimension $p$. The ratio $\frac{n}{d}$ was fixed to 1.5 for all experiments. The activation function was chosen to be $\tanh$.

In figure~\ref{fig:L2reglarization} we show the training and generalization error for $\ell_2^2$ regularization for 3 different regularization values as a function of the ratio $\frac{p}{n}$. We note that in the 1-Layer case $\frac{p}{n}$ is a measure of the under or overparameterization of the network. This relationship does not hold in the two layer case, however as may be seen from the figure this ratio is still useful in comparing the two models. In figure~\ref{fig:L2reglarization} the solid line represents the 2-layer case and the dashed line represents the 1-layer case. The triangles are our theoretical predictions for 2-layers, and squares similarly for 1-layer. For the Elastic net case we fix $\lambda_2$ to be $10^{-5}$ and vary only $\lambda_1$ these results are similarly shown in figure~\ref{fig:elacticNet}.

We note that in both types of regularization functions, for all values of $\frac{p}{n}$, the 2-layer deep RF model has consistently lower generalization error. With respect to training error the two layer case only outperforms 1-layer at large values of regularization. This suggests that even when training of the layer is not performed there can be a benefit to a deeper embedding of the input data. 
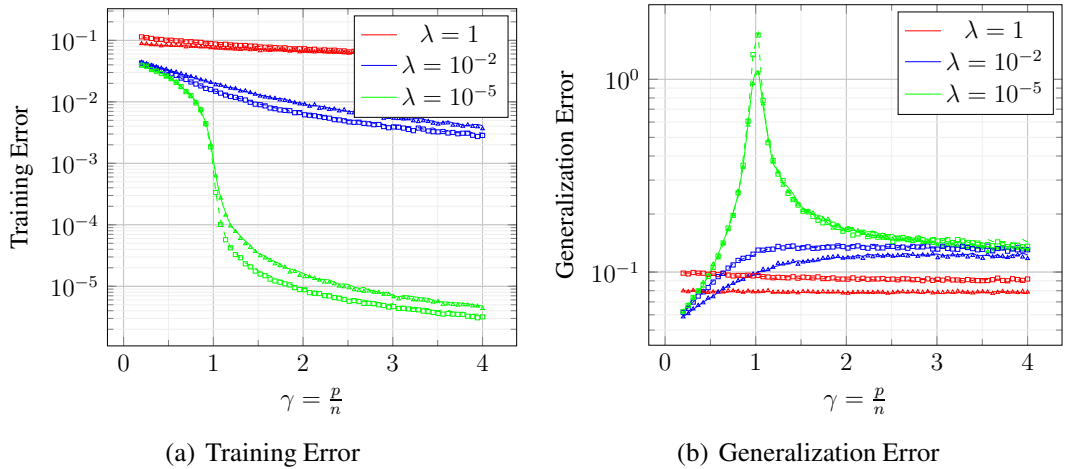
\begin{figure}[!hb]
    \centering
    
    \subfigure [Training Error] {
    \label{subfig:trainingErrorL2Synthetic}
    \resizebox{0.45 \textwidth}{!}{%
        \begin{tikzpicture}
        \begin{axis}[
          xlabel={$\gamma = \frac{p}{n}$},
          ylabel=Training Error,
          grid = both,
            minor tick num = 1,
            major grid style = {lightgray},
            minor grid style = {lightgray!25},
            ymode=log
          ]
        \addplot[ smooth, thin, red] table[ y=TrainActivation, x=gamma]{Data/SyntheticL2lam=1Gamma02-4_70_Delta1.5_100averages_sum300.dat};
        \addlegendentry{$\lambda = 1$}
        \addplot[ smooth, thin, blue] table[ y=TrainActivation, x=gamma]{Data/SyntheticL2lam=1e-2Gamma02-4_70_Delta1.5_100averages_sum300.dat};
        \addlegendentry{$\lambda = 10^{-2}$}
        \addplot[ smooth, thin, green] table[ y=TrainActivation, x=gamma]{Data/SyntheticL2lam=1e-5Gamma02-4_70_Delta1.5_100averages_sum300.dat};
        \addlegendentry{$\lambda = 10^{-5}$}
        \addplot[ dashed, thin, red] table[ y=TrainActivation, x=gamma]{Data/SyntheticL21Layerlam=1Gamma02-4_70_Delta1.5_100averages_sum300.dat};
        \addplot[ dashed, thin, blue] table[ y=TrainActivation, x=gamma]{Data/SyntheticL21Layerlam=1e-2Gamma02-4_70_Delta1.5_100averages_sum300.dat};
        \addplot[ dashed, thin, green] table[ y=TrainActivation, x=gamma]{Data/SyntheticL21Layerlam=1e-5Gamma02-4_70_Delta1.5_100averages_sum300.dat};
        \addplot[red, mark = triangle, mark size = 1pt, only marks] table[y=TrainGaussian, x=gamma]{Data/SyntheticL2lam=1Gamma02-4_70_Delta1.5_100averages_sum300.dat};
       \addplot[blue, mark = triangle, mark size = 1pt, only marks] table[y=TrainGaussian, x=gamma]
        {Data/SyntheticL2lam=1e-2Gamma02-4_70_Delta1.5_100averages_sum300.dat};
        \addplot[green, mark = triangle, mark size = 1pt, only marks] table[y=TrainGaussian, x=gamma]
        {Data/SyntheticL2lam=1e-5Gamma02-4_70_Delta1.5_100averages_sum300.dat};
        \addplot[red, mark = square, mark size = 1pt, only marks] table[y=TrainGaussian, x=gamma]{Data/SyntheticL21Layerlam=1Gamma02-4_70_Delta1.5_100averages_sum300.dat};
       \addplot[blue, mark = square, mark size = 1pt, only marks] table[y=TrainGaussian, x=gamma]
        {Data/SyntheticL21Layerlam=1e-2Gamma02-4_70_Delta1.5_100averages_sum300.dat};
        \addplot[green, mark = square, mark size = 1pt, only marks] table[y=TrainGaussian, x=gamma]
        {Data/SyntheticL21Layerlam=1e-5Gamma02-4_70_Delta1.5_100averages_sum300.dat};
        \end{axis}
    \end{tikzpicture}
        } 
    }
 \subfigure [Generalization Error] {
    \label{subfig:GeneralizationErrorL2Synthetic}
    \resizebox{0.45\textwidth}{!}{%
        \begin{tikzpicture}
        \begin{axis}[
          xlabel={$\gamma = \frac{p}{n}$},
          ylabel=Generalization Error,
          grid = both,
            minor tick num = 1,
            major grid style = {lightgray},
            minor grid style = {lightgray!25},
            ymode=log
          ]
        \addplot[ smooth, thin, red] table[ y=TestActivation, x=gamma]{Data/SyntheticL2lam=1Gamma02-4_70_Delta1.5_100averages_sum300.dat};
        \addlegendentry{$\lambda = 1$}
        \addplot[ smooth, thin, blue] table[ y=TestActivation, x=gamma]{Data/SyntheticL2lam=1e-2Gamma02-4_70_Delta1.5_100averages_sum300.dat};
        \addlegendentry{$\lambda = 10^{-2}$}
        \addplot[ smooth, thin, green] table[ y=TestActivation, x=gamma]{Data/SyntheticL2lam=1e-5Gamma02-4_70_Delta1.5_100averages_sum300.dat};
        \addlegendentry{$\lambda = 10^{-5}$}
        \addplot[ dashed, thin, red] table[ y=TestActivation, x=gamma]{Data/SyntheticL21Layerlam=1Gamma02-4_70_Delta1.5_100averages_sum300.dat};
        \addplot[ dashed, thin, blue] table[ y=TestActivation, x=gamma]{Data/SyntheticL21Layerlam=1e-2Gamma02-4_70_Delta1.5_100averages_sum300.dat};
        \addplot[ dashed, thin, green] table[ y=TestActivation, x=gamma]{Data/SyntheticL21Layerlam=1e-5Gamma02-4_70_Delta1.5_100averages_sum300.dat};
        \addplot[red, mark = triangle, mark size = 1pt, only marks] table[y=TestGaussian, x=gamma]{Data/SyntheticL2lam=1Gamma02-4_70_Delta1.5_100averages_sum300.dat};
       \addplot[blue, mark = triangle, mark size = 1pt, only marks] table[y=TestGaussian, x=gamma]
        {Data/SyntheticL2lam=1e-2Gamma02-4_70_Delta1.5_100averages_sum300.dat};
        \addplot[green, mark = triangle, mark size = 1pt, only marks] table[y=TestGaussian, x=gamma]
        {Data/SyntheticL2lam=1e-5Gamma02-4_70_Delta1.5_100averages_sum300.dat};
        \addplot[red, mark = square, mark size = 1pt, only marks] table[y=TestGaussian, x=gamma]{Data/SyntheticL21Layerlam=1Gamma02-4_70_Delta1.5_100averages_sum300.dat};
       \addplot[blue, mark = square, mark size = 1pt, only marks] table[y=TestGaussian, x=gamma]
        {Data/SyntheticL21Layerlam=1e-2Gamma02-4_70_Delta1.5_100averages_sum300.dat};
        \addplot[green, mark = square, mark size = 1pt, only marks] table[y=TestGaussian, x=gamma]
        {Data/SyntheticL21Layerlam=1e-5Gamma02-4_70_Delta1.5_100averages_sum300.dat};
        \end{axis}
    \end{tikzpicture}
        } 
    }
       \caption{Comparison of 1-Layer and 2-Layer RFs, with square loss function, $\ell_2^2$ regularization with regularization strength $\lambda$. Solid lines represent 2 layer and dashed lines 1-Layer. Triangles are the CGMT results for 2-layers and squares for 1-layer}
    \label{fig:L2reglarization}
\end{figure}

\begin{figure}[!hb]
    \centering
    
    \subfigure [Training Error] {
    \label{subfig:trainingErrorElasticSynthetic}
    \resizebox{0.45 \textwidth}{!}{%
        \begin{tikzpicture}
        \begin{axis}[
          xlabel={$\gamma = \frac{p}{n}$},
          ylabel=Training Error,
          grid = both,
            minor tick num = 1,
            major grid style = {lightgray},
            minor grid style = {lightgray!25},
            ymode=log
          ]
        \addplot[ smooth, thin, red] table[ y=TrainActivation, x=gamma]{Data/SyntheticElasticL21e-5_L1lam=1Gamma02-4_70_Delta1.5_100averages_sum300.dat};
        \addlegendentry{$\lambda = 1$}
        \addplot[ smooth, thin, blue] table[ y=TrainActivation, x=gamma]{Data/SyntheticElasticL21e-5_L1lam=1e-2Gamma02-4_70_Delta1.5_100averages_sum300.dat};
        \addlegendentry{$\lambda = 10^{-2}$}
        \addplot[ smooth, thin, green] table[ y=TrainActivation, x=gamma]{Data/SyntheticElasticL21e-5_L1lam=1e-5Gamma02-4_70_Delta1.5_100averages_sum300.dat};
        \addlegendentry{$\lambda = 10^{-5}$}
        \addplot[ dashed, thin, red] table[ y=TrainActivation, x=gamma]{Data/SyntheticElastic1LayerL21e-5_L1lam=1Gamma02-4_70_Delta1.5_100averages_sum300.dat};
        \addplot[ dashed, thin, blue] table[ y=TrainActivation, x=gamma]{Data/SyntheticElastic1LayerL21e-5_L1lam=1e-2Gamma02-4_70_Delta1.5_100averages_sum300.dat};
        \addplot[ dashed, thin, green] table[ y=TrainActivation, x=gamma]{Data/SyntheticElastic1LayerL21e-5_L1lam=1e-5Gamma02-4_70_Delta1.5_100averages_sum300.dat};
        \addplot[red, mark = triangle, mark size = 1pt, only marks] table[y=TrainGaussian, x=gamma]{Data/SyntheticElasticL21e-5_L1lam=1Gamma02-4_70_Delta1.5_100averages_sum300.dat};
       \addplot[blue, mark = triangle, mark size = 1pt, only marks] table[y=TrainGaussian, x=gamma]
        {Data/SyntheticElasticL21e-5_L1lam=1e-2Gamma02-4_70_Delta1.5_100averages_sum300.dat};
        \addplot[green, mark = triangle, mark size = 1pt, only marks] table[y=TrainGaussian, x=gamma]
        {Data/SyntheticElasticL21e-5_L1lam=1e-5Gamma02-4_70_Delta1.5_100averages_sum300.dat};
        \addplot[red, mark = square, mark size = 1pt, only marks] table[y=TrainGaussian, x=gamma]{Data/SyntheticElastic1LayerL21e-5_L1lam=1Gamma02-4_70_Delta1.5_100averages_sum300.dat};
       \addplot[blue, mark = square, mark size = 1pt, only marks] table[y=TrainGaussian, x=gamma]
        {Data/SyntheticElastic1LayerL21e-5_L1lam=1e-2Gamma02-4_70_Delta1.5_100averages_sum300.dat};
        \addplot[green, mark = square, mark size = 1pt, only marks] table[y=TrainGaussian, x=gamma]
        {Data/SyntheticElastic1LayerL21e-5_L1lam=1e-5Gamma02-4_70_Delta1.5_100averages_sum300.dat};
        \end{axis}
    \end{tikzpicture}
        } 
    }
 \subfigure [Generalization Error] {
    \label{subfig:GeneralizationErrorElasticSynthetic}
    \resizebox{0.45\textwidth}{!}{%
        \begin{tikzpicture}
        \begin{axis}[
          xlabel={$\gamma = \frac{p}{n}$},
          ylabel=Generalization Error,
          grid = both,
            minor tick num = 1,
            major grid style = {lightgray},
            minor grid style = {lightgray!25},
            ymode=log
          ]
        \addplot[ smooth, thin, red] table[ y=TestActivation, x=gamma]{Data/SyntheticElasticL21e-5_L1lam=1Gamma02-4_70_Delta1.5_100averages_sum300.dat};
        \addlegendentry{$\lambda = 1$}
        \addplot[ smooth, thin, blue] table[ y=TestActivation, x=gamma]{Data/SyntheticElasticL21e-5_L1lam=1e-2Gamma02-4_70_Delta1.5_100averages_sum300.dat};
        \addlegendentry{$\lambda = 10^{-2}$}
        \addplot[ smooth, thin, green] table[ y=TestActivation, x=gamma]{Data/SyntheticElasticL21e-5_L1lam=1e-5Gamma02-4_70_Delta1.5_100averages_sum300.dat};
        \addlegendentry{$\lambda = 10^{-5}$}
        \addplot[ dashed, thin, red] table[ y=TestActivation, x=gamma]{Data/SyntheticElastic1LayerL21e-5_L1lam=1Gamma02-4_70_Delta1.5_100averages_sum300.dat};
        \addplot[ dashed, thin, blue] table[ y=TestActivation, x=gamma]{Data/SyntheticElastic1LayerL21e-5_L1lam=1e-2Gamma02-4_70_Delta1.5_100averages_sum300.dat};
        \addplot[ dashed, thin, green] table[ y=TestActivation, x=gamma]{Data/SyntheticElastic1LayerL21e-5_L1lam=1e-5Gamma02-4_70_Delta1.5_100averages_sum300.dat};
        \addplot[red, mark = triangle, mark size = 1pt, only marks] table[y=TestGaussian, x=gamma]{Data/SyntheticElasticL21e-5_L1lam=1Gamma02-4_70_Delta1.5_100averages_sum300.dat};
       \addplot[blue, mark = triangle, mark size = 1pt, only marks] table[y=TestGaussian, x=gamma]
        {Data/SyntheticElasticL21e-5_L1lam=1e-2Gamma02-4_70_Delta1.5_100averages_sum300.dat};
        \addplot[green, mark = triangle, mark size = 1pt, only marks] table[y=TestGaussian, x=gamma]
        {Data/SyntheticElasticL21e-5_L1lam=1e-5Gamma02-4_70_Delta1.5_100averages_sum300.dat};
        \addplot[red, mark = square, mark size = 1pt, only marks] table[y=TestGaussian, x=gamma]{Data/SyntheticElastic1LayerL21e-5_L1lam=1Gamma02-4_70_Delta1.5_100averages_sum300.dat};
       \addplot[blue, mark = square, mark size = 1pt, only marks] table[y=TestGaussian, x=gamma]
        {Data/SyntheticElastic1LayerL21e-5_L1lam=1e-2Gamma02-4_70_Delta1.5_100averages_sum300.dat};
        \addplot[green, mark = square, mark size = 1pt, only marks] table[y=TestGaussian, x=gamma]
        {Data/SyntheticElastic1LayerL21e-5_L1lam=1e-5Gamma02-4_70_Delta1.5_100averages_sum300.dat};
        \end{axis}
    \end{tikzpicture}
        } 
    }
       \caption{Comparison of 1-Layer and 2-Layer RFs, with square loss function and $\ell_1 + \ell_2^2$ regularization with regularization strength $\lambda$ for the $\ell_1$ term and fixed $\ell_2$ regularization strength. Solid lines represent 2 layer and dashed lines are 1-Layer. Triangles are the CGMT result for 2-layers and squares for 1-layer}
    \label{fig:elacticNet}
\end{figure}
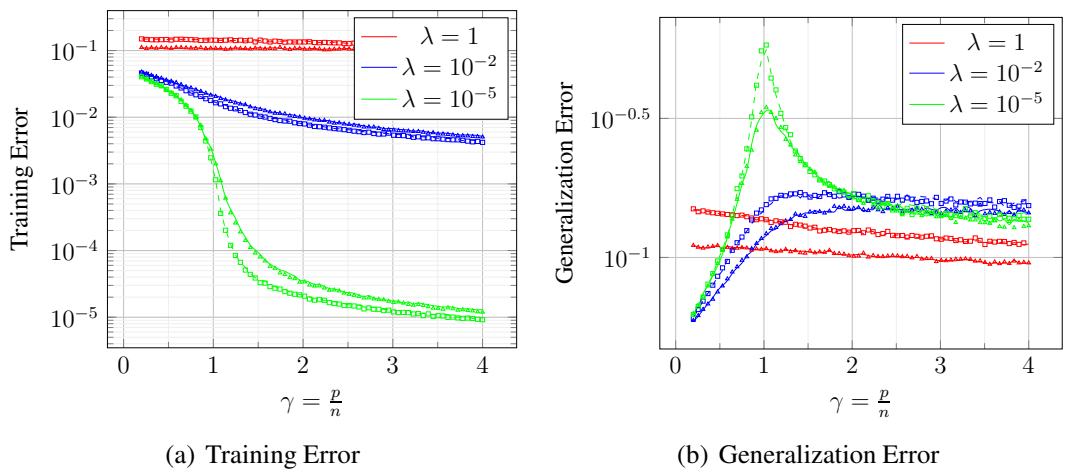

\subsection{Eigendistribution of the Covariance Matrix}
In the CGMT analysis performed above, where the input data is Gaussian, the Gaussian equivalent features $\gamma^{(L)}$ are distributed as $\calN(\bm{0}, \bR^{(L)})$ where $\bR^{(L)}$ is a covariance matrix defined recursively as
\begin{eqnarray}\label{eq:RlyapunovDefinition}
    \bR^{(0)} = \bI\qquad \bR^{(l)} = \rho_{1, l}^2\bW^{(l)}\bR^{(l-1)}\bW^{(l)T} + \rho_{2, l}^2\bI_{p_{l}},
\end{eqnarray}
where each $\bW^{(l)}$ has rows $\bw_j^{(l)}\sim\calN(\bm{0}, \frac{1}{p_{l-1}}\bI_{p_{l-1}})$. In the case of ridge regression of linear models, or any rotationally invariant setup, the optimal value is directly dependent upon the eigenvalues of the covariance matrix. As the covariance matrix is random we consider its eigendistribution, the marginal probability distribution over the eigenvalues.

We note that the type of recursion for $\bR^{(l)}$ is a form of a Lyapanov recursion, which has been studied in the literature \citep{vakili2011RandomMatrix, emery2007lectures}. 
We denote the eigendistribution of the matrix $\bR^{(l)}$ as $f_{\bR^{(l)}}(\lambda)$ for eigenvalues $\lambda$. In the case of $l = 1$, the matrix $\bR^{(1)}$ is a scaled Wishart matrix plus an identity, whose eigendistribution is given by a shifted version of the Marchenko–Pastur distribution. In figure~\ref{fig:HistEigDistribution} we consider the empirical eigendistribution of $\bR^{(2)}$, corresponding to the two layer case studied above. We choose $p_0 = 1000$ and $p_2= 1500$ fixing the input and output dimensions of the layers, and vary the size of the hidden layer $p_1$. We note as the size of the hidden layer grows the more concentrated the eigendistribution become around zero, while decreasing it results in a more flat structure. In the case of ridge regression, the decreased in the support of the eigenvalues could represent in an increase in model uncertainty at large sizes of the hidden layers.   

We also examine the eigendistribution analytically. We make use of the Stieltjes transform $S_{\bR^{(l)}}(z)$ of the distribution $f_{\bR^{(l)}}$. This transform and its inverse are give by 
\begin{eqnarray}
    S_{\bR^{(l)}}(z) = \int \frac{f_{\bR^{(l)}}(\lambda)}{\lambda - z}\mathrm{d}\lambda\qquad f_{\bR^{(l)}}(\lambda) = \frac{1}{\pi}\lim_{\omega \rightarrow 0^+}\mathrm{Im}[S(\lambda + i \omega)]
\end{eqnarray}
where $i$ is the imaginary unit, and $z$ is complex. We can demonstrate that the Stieltjes transform of the matrices $\bR^{(l)}$ follows the following recursion.
\begin{theorem}
    Let $\beta_l = \frac{p_l}{p_{l-1}}$, then the Stieltjes transform $S_l(z)$ of $\bR^{(l)}$ in \eqref{eq:RlyapunovDefinition} is given recursively by
    \begin{eqnarray}
        S_{l+1}(z) = \frac{1}{\rho_{1, l+1}^2}\Omega_{l}\left(\frac{z - \rho_{2, l+1}^2}{\rho_{1,l+1}^2} \right)\\
        \Omega_{l}(z) = \frac{1}{1 - \beta - \beta z\Omega_l(z)}S_l\left(\frac{z}{1 - \beta - \beta z\Omega_l(z)}\right)
    \end{eqnarray}
    Where $\Omega_0$ is the Stieltjes transform of a Wishart matrix, given by
    \begin{eqnarray}
        \Omega_0 = \frac{1-\beta_1 - z + \sqrt{z^2 - 2(\beta_1 +1)z + (\beta_1 -1)^2}}{2\beta_1z}
    \end{eqnarray}
\end{theorem}
\begin{proof}
The proof is given in appendix~\ref{app:sec:LyapanovProofs}.
\end{proof}
The recursive definitions given are difficult to compute empirically, as such we will leave visualizing these results to future work. However the recursive structure suggests that there exists a limiting distribution over the eigenvalues in the limit of infinite depth characterized by the different ratio in size between the various layers.

 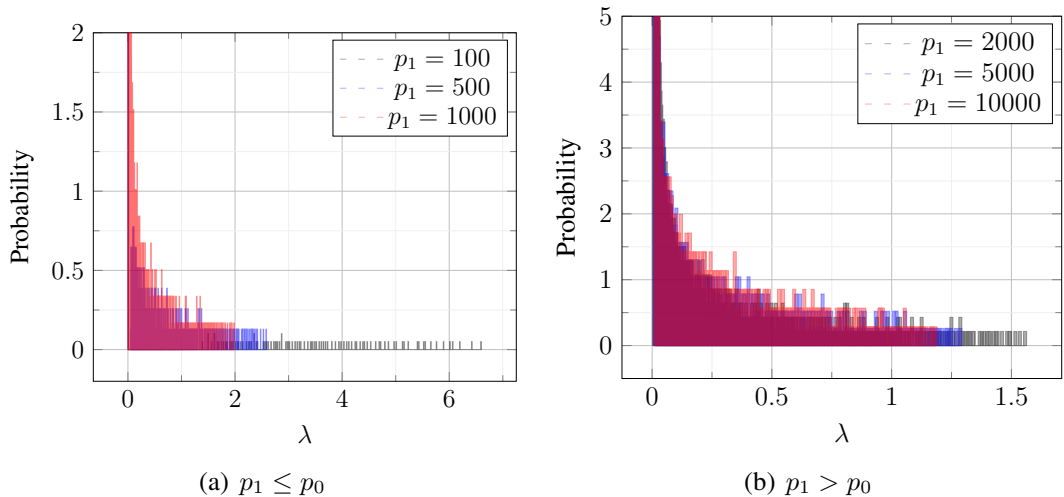
\begin{figure}[!hb]
    \centering
    \label{fig:HistEigDistribution}
    \subfigure [$p_1 \leq p_0$] {
    \label{subfig:StieltjesTransforms1}
    \resizebox{0.45 \textwidth}{!}{%
        \begin{tikzpicture}
        \begin{axis}[
          xlabel={$\lambda$},
          ylabel=Probability,
          grid = both,
            minor tick num = 1,
            major grid style = {lightgray},
            minor grid style = {lightgray!25},
            bar width = 0.01320055239333573,
            ymax =2
          ]
        \addplot[ybar,fill, black, opacity=0.33] table [x, y, col sep=space] {StieltjesData/HistVarmiddle=100.dat};
                \addlegendentry{$p_1= 100$}
        \addplot[ybar,fill ,blue, opacity=0.33] table [x, y, col sep=space] {StieltjesData/HistVarmiddle=500.dat};
                \addlegendentry{$p_1= 500$}
        \addplot[ybar,fill, red, opacity=0.33] table [x, y, col sep=space] {StieltjesData/HistVarmiddle=1000.dat};
                \addlegendentry{$p_1= 1000$}
        \end{axis}
        \end{tikzpicture}
         } 
        }
    \subfigure [$p_1 > p_0$] {
    \label{subfig:StieltjesTransforms2}
    \resizebox{0.45 \textwidth}{!}{%
        \begin{tikzpicture}
        \begin{axis}[
          xlabel={$\lambda$},
          ylabel=Probability,
          grid = both,
            minor tick num = 1,
            major grid style = {lightgray},
            minor grid style = {lightgray!25},
            bar width = 0.01320055239333573,
            ymax =5
          ]
        \addplot[ybar,fill, black, opacity=0.33] table [x, y, col sep=space] {StieltjesData/HistVarmiddle=2000.dat};
        \addlegendentry{$p_1= 2000$}
        \addplot[ybar,fill ,blue, opacity=0.33] table [x, y, col sep=space] {StieltjesData/HistVarmiddle=5000.dat};
                \addlegendentry{$p_1= 5000$}
        \addplot[ybar,fill, red, opacity=0.33] table [x, y, col sep=space] {StieltjesData/HistVarmiddle=10000.dat};
                \addlegendentry{$p_1= 10000$}
        \end{axis}
    \end{tikzpicture}
        } 
    }
    \caption{Empirical Eigendistribution of $\bR^{(l)}$ for various sizes $p_1$ of the 1st hidden layer}
\end{figure}

\section{Conclusion}\label{sec:Conclusion}
In this paper, we prove an asymptotic equivalence between deep random feature models and linear Gaussian models with respect to the training and generalization error. As a result of this universality, we can study a Gaussian equivalent model to the DRF model, in the asymptotic limit. We use this fact to provide an exact asymptotic analysis by means of the convex Gaussian min max theorem for an $L$-layer deep random feature model with Gaussian inputs. We further demonstrate that depth has an effect on training and generalization error both experimentally and by studying the eigendistribution of the Gaussian equivalent model's Covariance matrix.

\bibliographystyle{abbrv}
\bibliography{bib}

\begin{thebibliography}{10}

\bibitem{abbasi2019universality}
E.~Abbasi, F.~Salehi, and B.~Hassibi.
\newblock Universality in learning from linear measurements.
\newblock {\em Advances in Neural Information Processing Systems}, 32, 2019.

\bibitem{bach2017equivalence}
F.~Bach.
\newblock On the equivalence between kernel quadrature rules and random feature
  expansions.
\newblock {\em The Journal of Machine Learning Research}, 18(1):714--751, 2017.

\bibitem{baraniuk2008simple}
R.~Baraniuk, M.~Davenport, R.~DeVore, and M.~Wakin.
\newblock A simple proof of the restricted isometry property for random
  matrices.
\newblock {\em Constructive Approximation}, 28(3):253--263, 2008.

\bibitem{bartlettLongLugosiTsigler_2020}
P.~L. Bartlett, P.~M. Long, G.~Lugosi, and A.~Tsigler.
\newblock Benign overfitting in linear regression.
\newblock {\em arxiv:1906.11300}, 2020.

\bibitem{Bartlett30063}
P.~L. Bartlett, P.~M. Long, G.~Lugosi, and A.~Tsigler.
\newblock Benign overfitting in linear regression.
\newblock {\em Proceedings of the National Academy of Sciences},
  117(48):30063--30070, 2020.

\bibitem{belkin2019reconciling}
M.~Belkin, D.~Hsu, S.~Ma, and S.~Mandal.
\newblock Reconciling modern machine-learning practice and the classical
  bias--variance trade-off.
\newblock {\em Proceedings of the National Academy of Sciences},
  116(32):15849--15854, 2019.

\bibitem{belkin2020two}
M.~Belkin, D.~Hsu, and J.~Xu.
\newblock Two models of double descent for weak features.
\newblock {\em SIAM Journal on Mathematics of Data Science}, 2(4):1167--1180,
  2020.

\bibitem{bosch2021double}
D.~Bosch, A.~Panahi, and A.~{\"O}zcelikkale.
\newblock Double descent in feature selection: Revisiting lasso and basis
  pursuit.
\newblock In {\em International Conference on Machine Learning (ICML) 2021
  Workshop on Overparameterization: Pitfalls \& Opportunities}, 2021.

\bibitem{Bosch2022DoubleDescent}
D.~Bosch, A.~Panahi, A.~Özcelikkale, and D.~Dubhash.
\newblock Double descent in random feature models: Precise asymptotic analysis
  for general convex regularization, 2022.

\bibitem{boucheron2013concentration}
S.~Boucheron, G.~Lugosi, and P.~Massart.
\newblock {\em Concentration inequalities: A nonasymptotic theory of
  independence}.
\newblock Oxford university press, 2013.

\bibitem{chandrasekaran2012convex}
V.~Chandrasekaran, B.~Recht, P.~A. Parrilo, and A.~S. Willsky.
\newblock The convex geometry of linear inverse problems.
\newblock {\em Foundations of Computational mathematics}, 12:805--849, 2012.

\bibitem{chang2020provable}
X.~Chang, Y.~Li, S.~Oymak, and C.~Thrampoulidis.
\newblock Provable benefits of overparameterization in model compression: From
  double descent to pruning neural networks.
\newblock {\em arXiv preprint arXiv:2012.08749}, 2020.

\bibitem{daniely2017sgd}
A.~Daniely.
\newblock Sgd learns the conjugate kernel class of the network.
\newblock {\em Advances in Neural Information Processing Systems}, 30, 2017.

\bibitem{daniely2016toward}
A.~Daniely, R.~Frostig, and Y.~Singer.
\newblock Toward deeper understanding of neural networks: The power of
  initialization and a dual view on expressivity.
\newblock {\em Advances in neural information processing systems}, 29, 2016.

\bibitem{dhifallah2021inherent}
O.~Dhifallah and Y.~Lu.
\newblock On the inherent regularization effects of noise injection during
  training.
\newblock In {\em International Conference on Machine Learning}, pages
  2665--2675. PMLR, 2021.

\bibitem{dhifallah2020precise}
O.~Dhifallah and Y.~M. Lu.
\newblock A precise performance analysis of learning with random features.
\newblock {\em arXiv preprint arXiv:2008.11904}, 2020.

\bibitem{emery2007lectures}
M.~Emery, A.~Nemirovski, and D.~Voiculescu.
\newblock {\em Lectures on Probability Theory and Statistics: Ecole D'Ete de
  Probabilites de Saint-Flour XXVIII-1998}.
\newblock Springer, 2007.

\bibitem{gerace2022gaussian}
F.~Gerace, F.~Krzakala, B.~Loureiro, L.~Stephan, and L.~Zdeborov{\'a}.
\newblock Gaussian universality of linear classifiers with random labels in
  high-dimension.
\newblock {\em arXiv preprint arXiv:2205.13303}, 2022.

\bibitem{gerace2020generalisation}
F.~Gerace, B.~Loureiro, F.~Krzakala, M.~M{\'e}zard, and L.~Zdeborov{\'a}.
\newblock Generalisation error in learning with random features and the hidden
  manifold model.
\newblock In {\em International Conference on Machine Learning}, pages
  3452--3462. PMLR, 2020.

\bibitem{ghorbani2021linearized}
B.~Ghorbani, S.~Mei, T.~Misiakiewicz, and A.~Montanari.
\newblock {Linearized two-layers neural networks in high dimension}.
\newblock {\em The Annals of Statistics}, 49(2):1029 -- 1054, 2021.

\bibitem{goldt2020gaussian}
S.~Goldt, B.~Loureiro, G.~Reeves, F.~Krzakala, M.~M{\'e}zard, and
  L.~Zdeborov{\'a}.
\newblock The gaussian equivalence of generative models for learning with
  shallow neural networks.
\newblock In {\em Mathematical and Scientific Machine Learning}, pages
  426--471. PMLR, 2022.

\bibitem{goldt2019modeling}
S.~Goldt, M.~M{\'e}zard, F.~Krzakala, and L.~Zdeborov{\'a}.
\newblock Modeling the influence of data structure on learning in neural
  networks: The hidden manifold model.
\newblock {\em Physical Review X}, 10(4):041044, 2020.

\bibitem{gordon1985some}
Y.~Gordon.
\newblock Some inequalities for gaussian processes and applications.
\newblock {\em Israel Journal of Mathematics}, 50(4):265--289, 1985.

\bibitem{gordon1988milman}
Y.~Gordon.
\newblock On milman's inequality and random subspaces which escape through a
  mesh in r n.
\newblock In {\em Geometric aspects of functional analysis}, pages 84--106.
  Springer, 1988.

\bibitem{hastie2019surprises}
T.~Hastie, A.~Montanari, S.~Rosset, and R.~J. Tibshirani.
\newblock Surprises in high-dimensional ridgeless least squares interpolation.
\newblock {\em arXiv preprint arXiv:1903.08560}, 2019.

\bibitem{HuUniversalityLaws}
H.~Hu and Y.~M. Lu.
\newblock Universality laws for high-dimensional learning with random features.
\newblock {\em CoRR}, abs/2009.07669, 2020.

\bibitem{hu2022universality}
H.~Hu and Y.~M. Lu.
\newblock Universality laws for high-dimensional learning with random features.
\newblock {\em IEEE Transactions on Information Theory}, 2022.

\bibitem{jacot2018neural}
A.~Jacot, C.~Hongler, and F.~Gabriel.
\newblock Neural tangent kernel: Convergence and generalization in neural
  networks.
\newblock In {\em NeurIPS}, pages 8580--8589, 2018.

\bibitem{kakade2009duality}
S.~Kakade, S.~Shalev-Shwartz, A.~Tewari, et~al.
\newblock On the duality of strong convexity and strong smoothness: Learning
  applications and matrix regularization.
\newblock {\em Unpublished Manuscript, http://ttic. uchicago.
  edu/shai/papers/KakadeShalevTewari09. pdf}, 2(1):35, 2009.

\bibitem{korada2011applications}
S.~B. Korada and A.~Montanari.
\newblock Applications of the lindeberg principle in communications and
  statistical learning.
\newblock {\em IEEE transactions on information theory}, 57(4):2440--2450,
  2011.

\bibitem{lee2017deep}
J.~Lee, Y.~Bahri, R.~Novak, S.~S. Schoenholz, J.~Pennington, and
  J.~Sohl-Dickstein.
\newblock Deep neural networks as gaussian processes.
\newblock {\em arXiv preprint arXiv:1711.00165}, 2017.

\bibitem{lindeberg1922neue}
J.~W. Lindeberg.
\newblock Eine neue herleitung des exponentialgesetzes in der
  wahrscheinlichkeitsrechnung.
\newblock {\em Mathematische Zeitschrift}, 15(1):211--225, 1922.

\bibitem{liu2021random}
F.~Liu, X.~Huang, Y.~Chen, and J.~A. Suykens.
\newblock Random features for kernel approximation: A survey on algorithms,
  theory, and beyond.
\newblock {\em IEEE Transactions on Pattern Analysis and Machine Intelligence},
  44(10):7128--7148, 2021.

\bibitem{louart2018random}
C.~Louart, Z.~Liao, and R.~Couillet.
\newblock A random matrix approach to neural networks.
\newblock {\em The Annals of Applied Probability}, 28(2):1190--1248, 2018.

\bibitem{loureiro2021capturing}
B.~Loureiro, C.~Gerbelot, H.~Cui, S.~Goldt, F.~Krzakala, M.~M{\'e}zard, and
  L.~Zdeborov{\'a}.
\newblock Capturing the learning curves of generic features maps for realistic
  data sets with a teacher-student model.
\newblock {\em arXiv preprint arXiv:2102.08127}, 2021.

\bibitem{Loureiro2021Learning}
B.~Loureiro, C.~Gerbelot, H.~Cui, S.~Goldt, F.~Krzakala, M.~Mézard, and
  L.~Zdeborová.
\newblock Learning curves of generic features maps for realistic datasets with
  a teacher-student model, 2021.

\bibitem{mei2019generalization}
S.~Mei and A.~Montanari.
\newblock The generalization error of random features regression: Precise
  asymptotics and the double descent curve.
\newblock {\em Communications on Pure and Applied Mathematics}, 2019.

\bibitem{mezard1987spin}
M.~M{\'e}zard, G.~Parisi, and M.~A. Virasoro.
\newblock {\em Spin glass theory and beyond: An Introduction to the Replica
  Method and Its Applications}, volume~9.
\newblock World Scientific Publishing Company, 1987.

\bibitem{montanari2017universality}
A.~Montanari and P.-M. Nguyen.
\newblock Universality of the elastic net error.
\newblock In {\em 2017 IEEE International Symposium on Information Theory
  (ISIT)}, pages 2338--2342. IEEE, 2017.

\bibitem{montanari2019generalization}
A.~Montanari, F.~Ruan, Y.~Sohn, and J.~Yan.
\newblock The generalization error of max-margin linear classifiers:
  High-dimensional asymptotics in the overparametrized regime.
\newblock {\em arXiv preprint arXiv:1911.01544}, 2019.

\bibitem{montanari2022universality}
A.~Montanari and B.~N. Saeed.
\newblock Universality of empirical risk minimization.
\newblock In {\em Conference on Learning Theory}, pages 4310--4312. PMLR, 2022.

\bibitem{oymak2018universality}
S.~Oymak and J.~A. Tropp.
\newblock Universality laws for randomized dimension reduction, with
  applications.
\newblock {\em Information and Inference: A Journal of the IMA}, 7(3):337--446,
  2018.

\bibitem{panahi2017universal}
A.~Panahi and B.~Hassibi.
\newblock A universal analysis of large-scale regularized least squares
  solutions.
\newblock In {\em NIPS}, pages 3384--3393, 2017.

\bibitem{papaspiliopoulos2020high}
O.~Papaspiliopoulos.
\newblock High-dimensional probability: An introduction with applications in
  data science, 2020.

\bibitem{rahimi2007random}
A.~Rahimi and B.~Recht.
\newblock Random features for large-scale kernel machines.
\newblock {\em Advances in neural information processing systems}, 20, 2007.

\bibitem{seddik2020random}
M.~E.~A. Seddik, C.~Louart, M.~Tamaazousti, and R.~Couillet.
\newblock Random matrix theory proves that deep learning representations of
  gan-data behave as gaussian mixtures.
\newblock In {\em International Conference on Machine Learning}, pages
  8573--8582. PMLR, 2020.

\bibitem{thrampoulidis2015gaussian}
C.~{Thrampoulidis}, S.~{Oymak}, and B.~{Hassibi}.
\newblock {The Gaussian min-max theorem in the Presence of Convexity}.
\newblock {\em arXiv e-prints}, page arXiv:1408.4837, Aug. 2014.

\bibitem{thrampoulidis2015regularized}
C.~Thrampoulidis, S.~Oymak, and B.~Hassibi.
\newblock Regularized linear regression: A precise analysis of the estimation
  error.
\newblock In {\em Conference on Learning Theory}, pages 1683--1709. PMLR, 2015.

\bibitem{vakili2011RandomMatrix}
A.~Vakili.
\newblock {\em Random Matrix Recursions in Estimation, Control, and Adaptive
  Filtering}.
\newblock PhD thesis, California Institute of Technology, 2011.

\bibitem{voiculescu1991limit}
D.~Voiculescu.
\newblock Limit laws for random matrices and free products.
\newblock {\em Inventiones mathematicae}, 104(1):201--220, 1991.

\bibitem{zhang2021understanding}
C.~Zhang, S.~Bengio, M.~Hardt, B.~Recht, and O.~Vinyals.
\newblock Understanding deep learning (still) requires rethinking
  generalization.
\newblock {\em Communications of the ACM}, 64(3):107--115, 2021.

\end{thebibliography}

\appendix

\section{Technical Lemmas and Theorem}\label{app:sec:TechnicalLemmas}
In this section we give a number of Lemmas and Theorems that will be used in the proofs below.

In the following lemma we demonstrate that passing the input through an activation function with Gaussian weights result in a subgaussian random variable under mild assumptions.
\begin{lemma}\label{lem:SubgaussianityofRF}
Consider $\bX = [\bx_1\ \bx_2\ \cdots\ \bx_n]$, where $\bx_i\in\bbR^{d}$ and define $r = \norm{\bX}_{\op}$. Suppose that the derivative $\sigma'$ of the activation function $\sigma$ is bounded, i.e. $\norm{\sigma'}_{\infty} \leq \tau$. Let $\bw\sim\calN(\bm{0}, \frac{1}{d}\bI)$. Then, the random vector $(\sigma(\bx_k^T\bw))_k$ is $\frac{\tau r}{\sqrt{d}}-$sub-Gaussian
\end{lemma}
\begin{proof}
Take a unit vector $\ba = (a_j)\in\bbR^{n}$. We show that $\bA(\bw):=\sum_{k=1}^n\ba_k\sigma(\bw^T\bx_k)$ is sub-Gaussian with parameter $\tau r/\sqrt{d}$. For this, we show that the function $A(\bw)$ is $\tau r-$Lipschitz continuous, which implies the desired result (see \citep{boucheron2013concentration}). For this, observe that
\begin{eqnarray}
    \nabla A = \sum_{k=1}^n\bx_k\sigma'(\bw^T\bx_k)a_k = \bX\bsigma,
\end{eqnarray}
where $\bsigma = (\sigma'(\bw^T\bx_k)a_k)_k$ and hence by assumption $\norm{\bsigma} \leq \tau$. We conclude that 
\begin{eqnarray}
    \norm{\nabla A}\leq \norm{\bX}_{\op}\norm{\bsigma} \leq \tau r.
\end{eqnarray}
This concludes the proof.
\end{proof}
Here we give a lemma that gives a high probability bound on the norm of a random matrix.
\begin{lemma}\label{lem:NormOfRandomMatrix}
    Consider a $p\times n$ random matrix $\bS$ where each row is independent and $\tau$-sub-Gaussian. Moreover, the covariance of each row is bounded by $\tau$ in operator norm. Then, there exists constants $c_0, \kappa$ only depending on $\tau$ such that for any $c>c_0$ the following holds:
    \begin{eqnarray}
        \Pr\left[ \norm{\bS} > c(\sqrt{p} + \sqrt{n}) \right] \leq e^{-\kappa c n}.
    \end{eqnarray}
\end{lemma}
\begin{proof}
    The proof is based on the standard $\epsilon-$net argument. Hence we do not give it here. See, for example, \citep{baraniuk2008simple} for a similar proof.
\end{proof}

Next for completeness we state the Convex Gaussian Min Max Theorem \citep{gordon1985some, gordon1988milman, thrampoulidis2015gaussian}. We make heavy use of this theorem in the proof of theorem~\ref{thm:CGMTsquareloss}.

\begin{theorem}[Convex Gaussin Min Max Theorem (CGMT)]\label{app:thm:CGMT} Let $\bG \in \bbR^{n\times m}, \bg\in\bbR^{m},$ and $\bh\in\bbR^{n}$ be independent of each other and have entries distributed according to $\calN(0, 1)$. Let $\calS_1\subset\bbR^{n}$ and $\calS_2\subset\bbR^{m}$ be non empty compact sets. Let $f(\cdot, \cdot)$ be a continuous function on $\calS_1\times\calS_2$. We define the primary and alternative optimization problems as follows:
\begin{eqnarray}
    P(\bG):=& \min_{\bx\in\calS_1}\max_{\by\in\calS_2} \bx^T\bG\by + f(\bx, \by)\\
    A(\bg, \bh):=& \min_{\bx\in\calS_1}\max_{\by\in\calS_2}\norm{\bx}_2\bg^T\by + \norm{\by}_2\bh^T\bx + f(\bx, \by),
\end{eqnarray}
Then for any $c_1\in\bbR$ we have that
\begin{eqnarray}
    \Pr(P(\bG)<c_1) \leq 2\Pr(A(\bg, \bh)\leq c_1).
\end{eqnarray}
Under the further assumption that $\calS_1$ and $\calS_2$ are convex sets, and $f$ is concave-convex on $\calS_1\times \calS_2$ then for all $c_2\in\bbR$ we have that
\begin{eqnarray}
    \Pr(P(\bG)>c_2) \leq 2\Pr(A(\bg, \bh)\geq c_2).
\end{eqnarray}
\end{theorem}

We note this theorem demonstrates that if $A(\bg, \bh)$ concentrates on a particular value $c$, ie
\begin{eqnarray}
    \Pr(|A(\bg, \bh) - c| > \epsilon) \xrightarrow[n, m \rightarrow \infty]{P} 0, \qquad \forall \epsilon > 0
\end{eqnarray}
then $P(\bG) $ will concentrate on the same limit.

\section{Proof of Theorem \ref{thm:GenericUniversality}}\label{app:sec:UniversalityTheoremProof}
Our proof is based on an application of Lindeberg's argument to the sequence of features $\phi_j$ for $j = 1,\ldots, p$. We will adopt the following notation for this section. For a matrix $\bA$ we denote its $i$th row by means of superscript $\ba^i$ and its $j$th column by means of subscript $\ba_j$.

For simplicity, for any $m\times n$ matrix $\bZ$ with columns $(\bz_k)$ we define
\begin{eqnarray}
    L(\bz) = \min_{\btheta\in\bbR^{m}}\frac{1}{n}\sum_{k=1}^n \ell(\bz_k^T\btheta, y_k) + R(\btheta).
\end{eqnarray}
By means of a splitting technique, we may express this as
\begin{eqnarray}
    L(\bZ) = \min_{\btheta\in\bbR^{m}, \balpha\in\bbR^{n}}\max_{\bd\in\bbR^{n}} \frac{1}{n}\left( \sum_{k=1}^n \ell(\alpha_k, y_k) + d_k(\alpha_k - \bz_k^T\btheta) \right) + R(\btheta)\nwl 
    = - \min_{\bd\in\bbR^{n}}\underbrace{\frac{1}{n}\sum_{k=1}^n \ell^*(-d_k, y_k) + R^*(\frac{1}{n}\bZ\bd)}_{\Lambda(\bd, \bZ)}
\end{eqnarray}
where $\ell^*, R^*$ are the Legendre transforms of $\ell$ and $R$ respectively. We note that $L(\bF) = \calE_{train}(\bF)$ and $L(\bG) = \calE_{train}(\bG)$. Furthermore, we define $\bZ_r$ for $r = 0, 1, \ldots, m$ as
\begin{eqnarray}
    \bZ_r = \sum_{j=1}^r\bd_j\bgamma^j + \sum_{j=r+1}^p \bd_j\bphi^j,
\end{eqnarray}
where $\bgamma^j, \bphi^j$ are the $j$th row of $\bPhi$ and $\bGamma$ respectively. We note that, $\bd_j\bgamma^j$ and $\bd_j\bphi^j$ are outer (tensor) products, resulting in matrices. As a result $\bZ_0 = \bF$ and $\bZ_m = \bG$. We have that
\begin{eqnarray}\label{app:eq:lindebergstep}
    \left|\bbE\psi(L(\bF)) - \bbE\psi(L(\bG)) \right| \leq \sum_{r=1}^m\left| \bbE\psi(L(\bZ_r)) - \bbE\psi(L(\bZ_{r-1})) \right|.
\end{eqnarray}

Now , for $r =1, 2, \ldots, m$ for any vector $\bu \in\bbR^{n}$, define
\begin{eqnarray}
    \bZ_{-r}(\bu) = \sum_{j=1}^{r-1}\bd_j\bgamma^j + \bd_r\bu^T + \sum_{j=r+1}^p\bd_j\bphi^j.
\end{eqnarray}
We note that $\bZ_r = \bZ_{-r}(\bgamma^r)$ and that $\bZ_{r-1} = \bZ_{-r}(\bphi^r)$, as such
\begin{eqnarray}
    \bbE\psi(L(\bZ_r)) - \bbE\psi(L(\bZ_{r-1})) = \nwl
    \left[\bbE\psi(L(\bZ_{-r}(\bgamma^r))) - \bbE\psi(L(\bZ_{-r}(\bm{0}))) \right] - \left[\bbE\psi(L(\bZ_{-r}(\bphi^r))) - \bbE\psi(L(\bZ_{-r}(\bm{0}))) \right].
\end{eqnarray}

We now define $\hat{\bd}_r$ and $\hat{\bd}_{-r}$ as the minimal solutions of $\Lambda(\bd, \bZ_r)$ and $\Lambda(\bd, \bZ_{-r}(\bm{0}))$ respectively. We note that $\bgamma^r, \bphi^r$ are $\tau-$sub-Gaussian and independent of $\bZ_{-r}(\bm{0})$. Hence, we examine the following term:
\begin{eqnarray}
    \bbE\psi(L(\bZ_{-r}(\bu))) - \bbE\psi(L(\bZ_{-r}(\bm{0}))),
\end{eqnarray}
for a generic $\tau$-sub-Gaussian independent random vector $\bu$.

We recall that $\bR$ is $\mu$-strongly convex and $M$-smooth, we have that $R^*$ is $\frac{1}{M}-$strongly convex and $\frac{1}{\mu}$ smooth for any $\bZ$, \citep{kakade2009duality}[theorem 6]. The optimal solution $\hat{\bd}$ is therefore uniquely identified by the first order optimiality condition
\begin{eqnarray}
    \bh(\bd, \bZ):= \nabla_\bd\Lambda(\bd, \bZ) = \bzeta(\bd) + \frac{1}{n}\bZ^T\nabla R^*\left(\frac{1}{n}\bZ\bd\right) = \bm{0}
\end{eqnarray}
where $\bzeta(\bd)$ is the vector of values $(\ell(d_k, y_k))_k$ with $\ell'$ being the partial derivative of $\ell^*$ with respect to the first argument. In particular, $\bh(\hat{\bd}_{-r}, \bZ_{-r}(\bm{0})) = \bm{0}$. We can therefore conclude that for every $\bu = (u_k)_{k=1}^n$ that
\begin{eqnarray}\label{App:eq:firstorderexpansion}
    \bh(\hat{\bd}_{-r}, \bZ_{-r}(\bu)) = \frac{1}{n}\bu\bd_{r}^T\nabla R^*\left(\frac{1}{n}\bZ_{-r}\hat{\bd}_{-r} \right) + \nwl
    \frac{1}{n}\bZ^T_{-r}(\bu)\left(\nabla R^*\left(\frac{1}{n}\bZ_{-r}(\bu)\hat{\bd}_{-r}\right) - \nabla R^*\left(\frac{1}{n}\bZ_{-r}(\bm{0})\hat{\bd}_{-r} \right) \right).
\end{eqnarray}
Where we have used the fact that
\begin{eqnarray}
    \bZ_{-r}(\bu) = \bZ_{-r}(\bm{0}) + \bd_r\bu^T.
\end{eqnarray}

We can further conclude that
\begin{eqnarray}
    \frac{1}{n}\bZ_{-r}(\bu)\hat{\bd}_{-r} = \frac{1}{n}\bZ_{-r}(\bm{0})\hat{\bd}_{-r} + \bd_r\frac{\bu^T\hat{\bd}_{-r}}{n}.
\end{eqnarray}

\subsection{Bounding the terms in (49)}
Now, we introduce a series of bounds and approximations on the terms involved in \ref{App:eq:firstorderexpansion}. For ease of notation, we introduce the following:

\begin{definition}
We say than an expression including the parameter $c$ holds \textit{with high probability (w.h.p)} if there are constants $c_0, \kappa$ such that for any $c>c_0$, the expression holds with probability higher than $1-\kappa e^{-\kappa c n}$. We also denote $C:=\poly(c)$.
\end{definition}

Recall that we have assumed that $\bu$ is a $\tau$-sub-Gaussian vector. We now note that all the matrices $\bZ_r, \bZ_{-r}(\bm{0})$ and $\bZ_{-r}(\bu)$ can be expressed as $\bZ = \bD\bS$ where each row of $\bS$ is independent an associated with either a random feature, a replaced Gaussian feature, or $\bu$. Hence, for $p\sim n$, by assumption A1 and lemma \ref{lem:NormOfRandomMatrix}, we have that $\norm{\bS}_2 \leq C(\sqrt{p} + \sqrt{n}) \leq C\sqrt{n}$ holds with high probability, and by the conditions on $\bD$ assumed for the theorem the matrices $\bZ_r, \bZ_{-r}(\bm{0})$ and $\bZ_{-r}(\bu)$ are also bounded in operator norm by $C\sqrt{n}$ with high probability.

Next we  note by assumption A2 that $\norm{\bzeta(\bm{0})}\leq C\sqrt{n}$ and by assumption A1, that $\nabla R^*(\bm{0}) =\bm{0}$. Moreover, as $R^*$ is $\frac{1}{M}-$strongly convex, we obtain that
\begin{eqnarray}
    \norm{\hat{\bd}_{-r}} \leq M\norm{\bh(\bm{0}, \bZ_{-r}(\bm{0}))} \leq C\sqrt{n}.
\end{eqnarray}
By the $\frac{1}{\mu}-$ smoothness of $R^*$, we also obtain that:
\begin{eqnarray}
    \norm{\nabla R^*\left(\frac{1}{n}\bZ_{-r}(\bm{0})\hat{\bd}_{-r} \right)}\leq \frac{1}{\mu n}\norm{\bZ_{-r}(\bm{0})\hat{\bd}_{-r}} \leq C\whp
\end{eqnarray}

and 
\begin{eqnarray}
\norm{\nabla R^*\left(\frac{1}{n}\bZ_{-r}(\bu)\hat{\bd}_{-r}\right) - \nabla R^*\left(\frac{1}{n}\bZ_{-r}(\bm{0}\hat{\bd}_{-r}) \right)} \leq \frac{1}{\mu n}\norm{\bd_r}\frac{\left|\bu^T\hat{\bd}_{-r}\right|}{n} \leq \frac{C}{n^2}\left|\bu^T\hat{\bd}_{-r}\right|.   
\end{eqnarray}

Recalling that $\bu$ is $\tau$-sub-Gaussian, hence:
\begin{eqnarray}
    \Pr\left[\left|\bu^T\hat{\bd}_{-r} \right| > c\sqrt{n}\norm{\hat{\bd}_{-r}} \right]\leq e^{-kcn},
\end{eqnarray}
where $\kappa$ only depends on $\tau$. From this we conclude that,
\begin{eqnarray}
    \norm{\nabla R^*\left( \frac{1}{n}\bZ_{-r}(\bu)\hat{\bd}_{-r}\right) - \nabla R^*\left( \frac{1}{n}\bZ_{-r}(\bm{0})\hat{\bd}_{-r}\right)}\leq \frac{C}{n} \whp.
\end{eqnarray}

Finally, applying lemma \ref{lem:NormOfRandomMatrix} to $\bS = \bu^T$ (with $p=1$) shows that $\norm{\bu} \leq C\sqrt{n}$ with high probability. As such we can make the following conclusion about \eqref{App:eq:firstorderexpansion}:
\begin{eqnarray}
    \norm{\bh(\hat{\bd}_{-r}, \bZ_{-r}(\bu)) -\frac{\delta_r}{n}\bu} \leq \frac{C}{n^{3/2}} \whp
\end{eqnarray}
where $\delta_r = \bd_r^T\nabla R^*\left(\frac{1}{n}\bZ_{-r}(\bm{0})\hat{\bd}_{-r} \right)$. Hence, $|\delta_r|\leq C \whp$ and
\begin{eqnarray}
    \norm{\bh(\hat{\bd}_r, \bZ_{-r}(\bu))} \leq \frac{C}{\sqrt{n}} \whp.
\end{eqnarray}

\subsection{Approximating Loss Function Difference}
In this section we approximate the value of  $L(\bZ_{-r}(\bu))$.

We denote $\bJ_r = \frac{\partial \bh}{\partial \bd}(\hat{\bd}_{-r}, \bZ_{-r}(\bm{0}))$ and introduce the following point:
\begin{eqnarray}
    \hat{\bd}_{+r}(\bu) = \hat{\bd}_{-r} - \frac{\delta_r}{n}\bJ^{-1}_r\bu.
\end{eqnarray}
We note that $\frac{\delta_r}{n}\bu + \bJ_r(\hat{\bd}_{+r}(\bu) - \hat{\bd}_{-r}) =\bm{0}$ and by strong convexity that $\bJ_{r}  \succeq  \frac{1}{M}L$. Furthermore, by the assumption on the third derivatives,
\begin{eqnarray}
    \norm{\bh(\hat{\bd}_{+r}, \bZ_{-r}(\bu))} \nwl
    = \norm{\hat{\bd}_{+r}, \bZ_{-r}(\bu) - \frac{\delta_r}{n} -\bJ_r\left(\hat{\bd}_{+r}(\bu) - \hat{\bd}_{-r} \right)}\nwl
    \leq \norm{\bh(\hat{\bd}_{+r}, \bZ_{-r}(\bu)) - \bh(\hat{\bd}_{+r}, \bZ_{-r}(\bu)) - \bJ_r\left(\hat{\bd}_{+r}(\bu) - \hat{\bd}_{-r} \right)} + \frac{C}{n^{3/2}}\nwl 
    \leq C\norm{\hat{\bd}_{+r}(\bu) - \hat{\bd}_{-r}}^2 + \frac{C}{n^{3/2}} =\frac{C\delta-r^2}{n^2}\norm{\bJ_r^{-1}\bu}^2 + \frac{C}{n^{3/2}} \leq \frac{C}{n} \whp.
\end{eqnarray}
Finally, from strong convexity, we conclude that
\begin{eqnarray}
    0 \leq \Lambda(\hat{\bd}_{+r}(\bu), \bZ_{-r}(\bu)) + L(\bZ_{-r}(\bu)) \leq \frac{M}{2}\norm{\bh(\hat{\bd}_{+r}(\bu), \bZ_{-r}(\bu))}^2 \leq \frac{C}{n^2} \whp.
\end{eqnarray}
On the other hand, we note that
\begin{eqnarray}
    \frac{1}{n}\bZ_{-r}(\bu)\hat{\bd}_{+r} = \frac{1}{n}\bZ_{-r}(\bm{0})\bd_{-r} + \frac{1}{n}\bd_r\bu^T\bd_{-r} - \frac{\delta_r}{n^2}\bZ_{-r}(\bm{0})\bJ^{-1}_r\bu - \frac{\delta_r}{n^2}\bd_r\bu^T\bJ_r^{-1}\bu.
\end{eqnarray}

We now define
\begin{eqnarray}
    B_r(\bu):= \bmeta^T_r\left[\frac{1}{n}\bd_r\bu^T\bd_{-r} - \frac{\delta_r}{n^2}\bu^T\bJ^{-1}_r\bu \right] + \frac{1}{2n^2}\bd_r^T\bH_r\bd_r\left(\bu^T\bd_{-r}\right)^2 + \frac{\delta_r^2}{2n^3}\bu^T\bJ_r^{-1}\bLambda_r\bJ^{-1}_r\bu,
\end{eqnarray}
where $\bmeta_r$ and $\bH_r$ are the gradient and Hessian of $\bR^*$ respectively at $\frac{1}{n}\bZ_{-r}(\bm{0})\bd_{-r}$ and $\bLambda_r$ is the diagonal matrix of elements $(\ell''(d_k, y_k))$ where $\ell''$ is the second derivative of $\ell^*$ with respect to the first argument. From the previous bounds we conclude that
\begin{eqnarray}
    \left|\Lambda(\hat{\bd}_{+r}(\bu), \bZ_{-r}(\bu)) + L(\bZ_{-r}(\bm{0})) - B_r(\bu) \right|\leq \frac{C}{n^{3/2}} \whp
\end{eqnarray}
from which we find that
\begin{eqnarray}
    \left|L(\bZ_{-r}(\bu)) - L(\bZ_{-r}(\bm{0})) - B_r(\bu) \right| \leq \frac{C}{n^{3/2}}\whp.
\end{eqnarray}

Hence, by the bounded derivatives of $\psi$ we have:
\begin{eqnarray}
        \left|\psi(L(\bZ_{-r}(\bu))) - \psi(L(\bZ_{-r}(\bm{0}))) - B_r(\bu) \right| \leq \frac{C}{n^{3/2}}\whp.
\end{eqnarray}

From the mean value theorem, we have that
\begin{eqnarray}
    \left|\psi(L(\bZ_{-r}(\bm{0}))) + B_r(\bu) - \psi(L(\bZ_{-r}(\bm{0}))) \right.\nwl \left.- \psi'(L(\bZ_{-r}(\bm{0})))B_r(\bu) - \frac{1}{2}\psi''(L(\bZ_{-r}(\bm{0})))B_r^2(\bu) \right| \leq C\left|B_r(\bu)\right|^3
\end{eqnarray}

Again, making use of the previous bounds, under the product measure $\Omega$ we observe that

\begin{eqnarray}
    \left| |B_r(\bu)|^2 - \frac{(\bmeta^T\bd_r)^2(\bu^T\bd_{-r})^2}{n^2} \right| \leq \frac{C}{n^{3/2}}, \qquad |B_r(\bu)| \leq \frac{C}{\sqrt{n}} \whp
\end{eqnarray}
and hence
\begin{eqnarray}
        \left|\psi(L(\bZ_{-r}(\bm{0}))) + B_r(\bu) - \psi(L(\bZ_{-r}(\bm{0}))) \right.\nwl \left.- \psi'(L(\bZ_{-r}(\bm{0})))B_r(\bu) - \frac{1}{2}\psi''(L(\bZ_{-r}(\bm{0})))\frac{(\bmeta^T\bd_r)^2(\bu^T\bd_{-r})^2}{n^2} \right| \leq \frac{C}{n^{3/2}} \whp.
\end{eqnarray}
Combining all of the steps together, we obtain that
\begin{eqnarray}
\left|\psi(L(\bZ_{-r}(\bu))) - \psi(L(\bZ_{-r}(\bm{0}))) \right.\nwl \left.- \psi'(L(\bZ_{-r}(\bm{0})))B_r(\bu) - \frac{1}{2}\psi''(L(\bZ_{-r}(\bm{0})))\frac{(\bmeta^T\bd_r)^2(\bu^T\bd_{-r})^2}{n^2} \right| \leq \frac{C}{n^{3/2}} \whp
\end{eqnarray}

\subsection{Bounding the Increments of (44) and Final Steps}
We now employ the following observation:
\begin{lemma}\label{lem:BoundedExpectationValues}
    Suppose that $A$ is a non-negative random variable such that $A\leq C\whp$ with $C=\poly(c)$. There exists a universal constant $c_1$ such that $\bbE[A] \leq c_1$.
\end{lemma}
\begin{proof}
    Note that the assumptions imply that there exist universal constants $c_0, \kappa$ such that for $c>c_0$ 
    \begin{eqnarray}
        \Pr[A > C] \leq \kappa e^{-\kappa n c}
    \end{eqnarray}
    Note that $C=\poly(c)\leq (\alpha c)^{\beta}$ for some constants $\alpha, \beta>0$. Hence for $C>C_0:= (\alpha c_0)^\beta$, we have
    \begin{eqnarray}
        \Pr[A > C] \leq \kappa e^{-\frac{\kappa}{\alpha}nC^{\frac{1}{\beta}}}.
    \end{eqnarray}
    As such, by making use of Tonelli's theorem we have that
    \begin{eqnarray}
        \bbE[A] = \int_{0}^{\infty} \Pr[A > C]dC \leq C_0 + \kappa\int_{C_0}^\infty e^{-\frac{\kappa}{\alpha}nC^{\frac{1}{\beta}}}.
    \end{eqnarray}
    It is simple to check that the right hand side is bounded by a universal constant.
\end{proof}

According to lemma \ref{lem:BoundedExpectationValues} we have that
\begin{eqnarray}
\left|\bbE\psi(L(\bZ_{-r}(\bu))) - \bbE\psi(L(\bZ_{-r}(\bm{0}))) \right.\nwl \left.- \bbE\psi'(L(\bZ_{-r}(\bm{0})))B_r(\bu) - \frac{1}{2}\bbE\psi''(L(\bZ_{-r}(\bm{0})))\frac{(\bmeta^T\bd_r)^2(\bu^T\bd_{-r})^2}{n^2} \right| \leq \frac{c_1}{n^{3/2}},
\end{eqnarray}
for some universal constant $c_1$. Now we note that each expectation can be carried out y first conditioning on $\bZ_{-r}(\bm{0})$ and then taking the expectation with respect to it. Accordingly, we denote $\bbE_{\bu}:= \bbE[\cdot|\bZ_{-r}(\bm{0})]$ as this expectation is only over $\bu$, which is independent of $\bZ_{-r}(\bm{0})$. Furthermore, we repeat the above bound for $\bu = \bgamma^r$ and $\bu = \bphi^r$, from which we obtain:
\begin{eqnarray}
    \left|\bbE\psi(L(\bZ_r)) - \bbE\psi(L(\bZ_{r-1})) - \bbE\psi'(L(\bZ_{-r}(\bm{0})))[\bbE_\bu B_r(\bgamma^r) - \bbE_\bu B_r(\bphi^r)] \right.\nwl \left. 
    - \frac{1}{2}\bbE\psi''(L(\bZ_{-r}(\bm{0})))\frac{(\bmeta^T\bd_r)^2}{n^2}\left[\bbE_\bu (\bd^T_{-r}\bphi^r)^2 - \bbE_{\bu}(\bd^T_{-r}\bgamma^r)^2\right] \right| \leq \frac{2c_1}{n^{3/2}}.
\end{eqnarray}

Making use of the bounds on the derivatives of $\psi$, we obtain:
\begin{eqnarray}
    \left|\bbE\psi(L(\bZ_r)) - \bbE\psi(L(\bZ_{r-1})) \right| \leq c\bbE\left|\bbE_{\bu}B_r(\bgamma^r) - \bbE_\bu B_r(\bphi^r) \right| + \nwl
    \frac{c}{2}\bbE\left|\frac{(\bmeta^T\bd_r)^2}{n^2}[\bbE_\bu (\bd_{-r}^T\bphi^r)^2 - \bbE_{\bu}(\bd_{-r}^T\bgamma^r) ] \right| + \frac{2c_1}{n^{3/2}}.
\end{eqnarray}
By the previous bounds, it is straightforward to see that
\begin{eqnarray}
    \bbE\left|\bbE_{\bu}B_r(\bgamma^r) - \bbE_\bu B_r(\bphi^r) \right| \leq \frac{C}{n}\norm{\bK_r - \bK_r'}_{\op}\whp,
\end{eqnarray}
and
\begin{eqnarray}
    \frac{(\bmeta^T\bd_r)^2}{n^2}[\bbE_\bu (\bd_{-r}^T\bphi^r)^2 - \bbE_{\bu}(\bd_{-r}^T\bgamma^r) ] \leq \frac{C}{n}\norm{\bK_r - \bK_{r}'}_{\op} \whp.
\end{eqnarray}
Hence by lemma \ref{lem:BoundedExpectationValues} and \eqref{app:eq:lindebergstep} we conclude that there exists a universal constant $c_1$ such that
\begin{eqnarray}
    \left|\bbE\psi(L(\bF)) - \bbE\psi(L(\bG))\right| \leq \frac{c_1}{n}\sum_{r=1}^p\norm{\bK_r - \bK'_r}_{\op} + \frac{c_1}{\sqrt{n}}
\end{eqnarray}
This concludes the proof of part 1 of the theorem

\subsection{Proof of part 2}
For $\epsilon \in \bbR$, we define $R_{\epsilon}:= R + \epsilon h$. Define $L_\epsilon(F), L_\epsilon(G)$ as the optimal values with $R_\epsilon$ and note that for all $\epsilon$
\begin{eqnarray}
    \epsilon h(\hat{\btheta}_F) \geq L_\epsilon(F) - L(F),
\end{eqnarray}
and 
\begin{eqnarray}
    \epsilon h(\hat{\btheta}_G) \geq L_\epsilon(G) - L(G).
\end{eqnarray}
We then note that for sufficiently (but finitely) small $\epsilon$ the conditions of the theorem are satisfied, ie $R_\epsilon$ remains strongly convex and smooth. Choose $\epsilon >0$ such that both $\epsilon$ and $-\epsilon$ statisfy these conditions. Then we have
\begin{eqnarray}
    \frac{L_\epsilon(G) - L(G)}{\epsilon} + \frac{L_\epsilon(F) - L(F)}{\epsilon} \leq h(\hat{\btheta}_{F}) - h(\hat{\btheta}_G) \leq -\frac{L_{\epsilon}(G) - L(G)}{\epsilon} - \frac{L_\epsilon(F) - L(F)}{\epsilon}.
\end{eqnarray}
Taking the expectation, and making use of the results of part 1, with $\psi(x) = x$ we conclude that
\begin{eqnarray}
    \left|\bbE h(\hat{\btheta}_F) - \bbE h(\hat{\btheta}_G) \right| \leq \bbE\left[-\frac{L_{\epsilon}(G) - L(G)}{\epsilon} - \frac{L_\epsilon(F) - L(F)}{\epsilon} \right] + \frac{c}{n}\sum_{r=1}^p \norm{\bK_r - \bK_{r}'} + \frac{c}{\sqrt{n}}.
\end{eqnarray}
Noting that $\epsilon$ can be arbitraily small, and hence letting $\epsilon \rightarrow 0$ we observe that
\begin{eqnarray}
    -\frac{L_{\epsilon}(G) - L(G)}{\epsilon} - \frac{L_\epsilon(F) - L(F)}{\epsilon} \rightarrow 0,
\end{eqnarray}
Moreover, $-\frac{L_{\epsilon}(G) - L(G)}{\epsilon} - \frac{L_\epsilon(F) - L(F)}{\epsilon}$ is bounded by twice the bound $h$. Then, we may invoke the dominated convergence theorem and conclude that
\begin{eqnarray}
    \bbE\left[-\frac{L_{\epsilon}(G) - L(G)}{\epsilon} - \frac{L_\epsilon(F) - L(F)}{\epsilon} \right] \rightarrow 0
\end{eqnarray}
Which concludes the proof.

\section{Proof of Theorems \ref{thm:PreservationOfRegularity} and \ref{thm:DeepRFUniversality}}\label{app:sec:DeepRFUniversalityProof}
The proof of these theorem relies on two intermediate results, we shall prove both of these first. Firstly consider the following theorem:

\begin{theorem}\label{thm:ReplacemetnOfCovarianceMatrix}
    Assume that $\sigma$ is odd and that assumption A4 holds. Take
    \begin{eqnarray}
        \mu:= \sup_{i, j} \left|\frac{\bx_i^T\bx_j}{d} -\delta_{ij} \right|,
    \end{eqnarray}
    and $\bw\sim\calN(\bm{0}, \frac{1}{d}\bI)$. Consider the random vector $\bphi =(\sigma(\bx_k^T\bw))_k$ and denote its covariance matrix by $\bK$. Then,
    \begin{eqnarray}
        \norm{\bK - \left(\frac{\rho_1^2}{d}\bX^T\bX + \rho_2^2\bI \right)}_{\op} \leq c\left(\mu^3n + \mu + \mu\frac{\norm{\bX}_{\op}^2}{d} \right),
    \end{eqnarray}
    where $c$ is a universal constant.
\end{theorem}

\begin{proof}
    Note that $K_{ij} = \bbE_{\bw}[\sigma(\bw^T\bx_i)\sigma(\bw^T\bx_j)]$. For $i\neq j$, we have that $K_{ij} = \eta_1\left(\frac{\norm{\bx_i}^2}{d}, \frac{\norm{\bx_j}^2}{d}, \frac{\bx_i^T\bx_j}{d} \right)$ where $\eta_1$ and for $i= j$, we have that $K_{ii} = \eta_2\left(\frac{\norm{\bx_i}^2}{d}\right)$. Where $\eta_1$ and $\eta_2$ are defined in assumption A4. Note that by oddness of the activation function
    \begin{eqnarray}
        \eta_1(1,1,0) = 0, \qquad \eta_2(1) = \bbE\sigma^2(g)\\
        \nabla\eta_1(1,1,0) = (0, 0, \bbE[g\sigma(g)]^2),
    \end{eqnarray}
    where $g$ is a standard normal. We also note that the hessian of $\eta_1$
    \begin{eqnarray}
        H_{\eta_1}(1,1, 0) = \begin{bmatrix}
            0 & 0 & -\bbE[g\sigma(g)]^2\\
            0 & 0 & 0 \\
            -\bbE[g\sigma(g)]^2 & 0 & 0
        \end{bmatrix}.
    \end{eqnarray}
    Then, by the mean value theorem and assumption A4 we have that
    \begin{eqnarray}
        |K_{ij} - K_{ij}'| \leq \begin{cases}
        c\mu^3 & i\neq j\\
        c\mu & i=j
        \end{cases},
    \end{eqnarray}
    where 
    \begin{eqnarray}
        K'_{ij} = \begin{cases}
            \frac{\bx_i^T\bx_j}{d}[\bbE[g\sigma(g)]]^2\left(1 - \frac{\norm{\bx_i}^2}{d}  \right) & i\neq j\\
            \bbE[\sigma^2(g)] & i = j
        \end{cases}.
    \end{eqnarray}

    From this we conclude that
    \begin{eqnarray}
        \norm{\bK - \bK'}_{\op} \leq c(\mu^3n + \mu).
    \end{eqnarray}
    It can also straightforwardly be checked that $\norm{\bK' - \left(\frac{\rho_1^2}{d}\bX^T\bX + \rho_2^2\bI \right)}_{\op} \leq c\mu\frac{\norm{\bX}_{op}^2}{d}$, from which the desired result can be obtained.
\end{proof}

The second intermediate result is shown in the following theorem.

\begin{theorem}\label{thm:OrthogonalityPreservation}
    Suppose that $\sigma$ is odd with bounded derivatives and assumption A4 holds. Moreover, the set $\lbrace \bx_i \in\bbR^d \rbrace_{i=1}^n$ satisfies:
    \begin{eqnarray}
        \sup_{i, j}\left|\frac{\bx_i^T\bx_j}{d} -\delta_{ij}  \right| \leq \frac{\polylog n}{\sqrt{n}}.
    \end{eqnarray}
    Define $\bz_i = \sigma(\bW\bx_i)$ where $\bW\in\bbR^{p\times d}$ has independent row distributed by $\calN(0, \frac{1}{d}\bI)$. Then, with a probability higher than $1-n^{-10}$ it holds that\footnote{The exponent is arbitrary and can be replaced by any other number}:
    \begin{eqnarray}
        \sup_{i, j}\left|\frac{\bz_i^T\bz_j}{p}-\delta_{ij}\right|\leq \frac{\polylog n}{\sqrt{n}}.
    \end{eqnarray}
\end{theorem}
\begin{proof}
    We note that
    \begin{eqnarray}
        \frac{1}{p}\bz_i^T\bz_j = \frac{1}{p}\sum_{r}\sigma(\bw_r^T\bx_i)\sigma(\bw_r^T\bx_j),
    \end{eqnarray}
    and by the assumptions $\sigma(\bw_r^T\bx_i)\sigma(\bw_r^T\bx_j)$ are i.i.d and sub-exponential. Hence, there exists a constant $c$ such that for every $t = o(\sqrt{n})$:
    \begin{eqnarray}
        \Pr\left[\left|\frac{1}{p}\bz_i^T\bz_j - \bbE_{\bw}[\sigma(\bw^T\bx_i)\sigma(\bw^T\bx_j)] \right| > \frac{t}{\sqrt{n}} \right] \leq 2e^{-ct}.
    \end{eqnarray}
    In particular, we may take $t = c \log n$ for a sufficiently large c, which by the union bound leads to
    \begin{eqnarray}
        \sup_{i, j}\left|\frac{1}{p}\bz_i^T\bz_j - \bbE_{\bw}[\sigma(\bw^T\bx_i)\sigma(\bw^T\bx_j)]  \right| < \frac{c\log n}{\sqrt{n}}
    \end{eqnarray}
    with the desired probability. 
    On the other hand $\bbE_{\bw}[\sigma(\bw^T\bx_i)\sigma(\bw^T\bx_j)]$ equals either
    $\eta_1\left(\frac{\norm{\bx_i}^2}{d}, \frac{\norm{\bx_j}^2}{d}, \frac{\bx_i^T\bx_j}{d} \right)$ for $i\neq j$ or $\eta_2\left( \frac{\norm{\bx_i}^2}{d}\right)$ for $i = j$. Then by assumption A4 the result holds.
\end{proof}

\subsection{Proof of theorem \ref{thm:PreservationOfRegularity}}
Two prove the theorem we need to show two proprieties. Firstly,
\begin{eqnarray}
    \max_{i, j}\left|\frac{\bz_i^T\bz_j}{p} - \delta_{ij}\right| \leq \frac{\polylog n}{\sqrt{n}}.
\end{eqnarray}
This has been shown by theorem \ref{thm:OrthogonalityPreservation}. Next, we need to show that
\begin{eqnarray}
    \norm{\bZ} \leq c\sqrt{n}.
\end{eqnarray}
For this we note that the rows of $\bZ$ are independent. Moreover, by lemma \ref{lem:SubgaussianityofRF} and the assumptions, each row is $c-$sub-Gaussian for a constant $c$. Finally, by theorem \ref{thm:ReplacemetnOfCovarianceMatrix}, we have that
\begin{eqnarray}
    \norm{\bK}_{\op} \leq \frac{\polylog n}{\sqrt{n}} + \norm{\frac{\rho_1^2}{d}\bX\bX^T + \rho_2^2\bI}_{\op} \leq c.
\end{eqnarray}
Then by lemma \ref{lem:NormOfRandomMatrix} the result follows.

\subsection{Proof of theorem \ref{thm:DeepRFUniversality}}
The proof is by induction. For $l=0$, the claim is trivially holds. For a given $l$, note that $\bx_i^{(l)} = \phi(\bx_i^{(l-1)},\bW^{(l)})$. Furthermore, $\lbrace \bx_i^{(l-1)} \rbrace_{i=1}^n$ is regular with a probability higher than $1-n^{-10}$ and hence by lemma \ref{lem:SubgaussianityofRF}, each row of $\Tilde{\bX}^{(l)} = [\Tilde{\bx}_i\ \Tilde{\bx}_2\ \cdots \Tilde{\bx}_n]$ is $c-$sub-Gaussian. Moreover, by theorem \ref{thm:ReplacemetnOfCovarianceMatrix}, we have
\begin{eqnarray}
    \norm{\bK}_{\op} \leq \frac{\polylog n}{\sqrt{n}} + \norm{\frac{\rho_1^2}{d}\bX^T\bX + \rho_2^2\bI}_{\op} \leq c
\end{eqnarray}
and hence by assumption the first condition for Theorem \ref{thm:GenericUniversality} holds true with a probability higher than $1-n^{-10}$. As a result, defining
\begin{eqnarray}
    \Tilde{\bg}_i^{(l)'} = \begin{bmatrix}
        \rho_1\bW^{(l)}\bx_i^{(l-1)} + \rho_2\bh_i^{(l)}\\
        \bv_i^{(l)}
    \end{bmatrix},
\end{eqnarray}
then theorem \ref{thm:GenericUniversality} holds for $\lbrace \Tilde{\bx}_i^{(l)} \rbrace_{i=1}^n$ and $\lbrace \Tilde{\bg}_i^{(l)'} \rbrace_{i=1}^n$. Denoting the optimal value and the optimal point for the latter by $L', \hat{\btheta'}$, we note that
\begin{eqnarray}
    \left|\bbE_{\bW^{(l)}}[\psi(L(\bF))] - \bbE_{\bW^{(l)}}[\psi(L')] \right| \leq \frac{\polylog n}{\sqrt{n}}
\end{eqnarray}
with probability $1-n^{-10}$. Note that $L(\bF)$ and $L'$ are bounded, hence:
\begin{eqnarray}
    \left|\bbE[\psi(L(\bF))] - \bbE[\psi(L')] \right| \leq \frac{\polylog n}{\sqrt{n}},
\end{eqnarray}
Where $\bbE[\psi(L')] = \bbE[\bbE[\psi(L;)|\bW^{(l)}]]$. On the other hand,
\begin{eqnarray}
    \bD\Tilde{\bg}^{(l)'}_i = \underbrace{\bD\begin{bmatrix}
        \rho_1\bW^{(l)} & \rho_2\bI &\\
        0 & 0 & \bI
    \end{bmatrix}}_{\bD'}
    \begin{bmatrix}
        \bx_i^{(l-1)}\\
        \bh_i^{(l)}\\
        \bv_i^{(l)}
    \end{bmatrix}.
\end{eqnarray}

Now we observe that with probability higher than $1-e^{-cn}$ it holds that $\norm{\bD'}\leq C$ and hence we may invoke the induction hypothesis for layer $l-1$ with $\bD'$ and $\bv_i^{(l-1)} = [\bh_i^{(l)}\ \bv_i^{(l)}]$ to conclude that
\begin{eqnarray}
    \left| \bbE[\psi(L(\bG))|\bW^{(l)}] - \bbE[\psi(L')|\bW^{(l)}]  \right| \leq \frac{\polylog n}{\sqrt{n}},
\end{eqnarray}
with a probability higher than $1-e^{-cn}$. Again using the fact that the optimal value is bounded, we conclude that
\begin{eqnarray}
    \left| \bbE[\psi(L(\bG))] - \bbE[\psi(L')]  \right| \leq \frac{\polylog n}{\sqrt{n}}.
\end{eqnarray}
Which concludes the claim for part 1. Part 2 is proven exactly by the same argument.

\section{Proof of Theorem \ref{thm:CGMTsquareloss}}\label{app:sec:CGMTproof}
To prove Theorem \ref{thm:CGMTsquareloss}, our goal is to make use of the CGMT (theorem~\ref{app:thm:CGMT}) to obtain an alternative optimization problem to \eqref{eq:CGTMP2}. Upon simplification we note that this problem relies entirely upon $\bR^{(L)}$ and note that is can once again be expressed as another CGMT style optimization. Applying the CGMT again results in a problem dependent upon $\bR^{(L-1)}$. Repeating the processes iteratively eventually results in the alternative optimization problem given in \eqref{eq:CGTMP3}. We adopt the same process for a recursive CGMT solution as in \citep{Bosch2022DoubleDescent}, and follow the direction of their proof.

To begin this processes we first recall the definition of problem $P_2$ given in \eqref{eq:CGTMP2}. We fix $\bW^{(1)}, \ldots, \bW^{(L)}$ and make a change of variables. Recalling the definition of $\by$, given in \eqref{eq:ytildemodel}, we introduce the error vector $\be = \btheta - \btheta^*$:
\begin{eqnarray}
    P_2 = \min_{\be\in\bbR^{p_L}} \frac{1}{2n}\norm{\bnu - \frac{1}{\sqrt{p_L}}\Tilde{\bX}^{(L)}\be}_2^2 + R(\be + \btheta^*).
\end{eqnarray}
We now recall that the rows $\Tilde{\bx}_i^{(L)}$ of $\Tilde{\bX}^{(L)}$ are i.i.d normally distributed with covariance $\bR^{(L)}$. As such we can express $\Tilde{\bX}^{(L)} = \bU^{(L)}\left(\bR^{(L)}\right)^{1/2}$ where $\bU^{(L)}\in\bbR^{n\times p_L}$ and has i.i.d normal Guassian entries and $\bR^{(L)}$ is given by
\begin{eqnarray}\label{eq:LlayerRmatrixDefinition}
    \bR^{(0)} = \bI \qquad \bR^{(l)} = \rho_{1,l}^2\bW^{(l)}\bR^{(l-1)}\bW^{(l)T} + \rho_{2,l}^2\bI \quad 1\leq l\leq L.
\end{eqnarray}
For the sake of notational simplicity we will express $\left(\bR^{(L)}\right)^{1/2}$ as $\bR^{(L)/2}$ when there is no chance of confusion. 

Next we make use of the Legendre transform of the $2$-norm. We obtain
\begin{eqnarray}
    P_2 = \min_{\be\in\bbR^{{p_L}}}\max_{\blambda\in\bbR^{n}} \frac{1}{n}\blambda^T\bnu - \frac{1}{n\sqrt{p_L}}\blambda^T\bU^{(L)}\bR^{(L)/2}\be - \frac{1}{2n}\norm{\blambda}_2^2 + R(\be + \btheta^*).
\end{eqnarray}

We note that the problem is now in the correct form to apply the CGMT. However, the CGMT requires that the optimizations over $\be$ and $\blambda$ are over compact and convex sets. In the subsequent lemmas we show that we can restrict the problem to compact and convex subsets of $\bbR^{p_L}$ and $\bbR^{n}$.

Firstly, we show that $\bR^{(l)}$ for all $0\leq l\leq L$ can be bounded above by a constant in operator norm with high probability.

\begin{lemma}\label{lem:BoundingRL}
Let $\bR^{(l)}$ be defined as in \eqref{eq:LlayerRmatrixDefinition}, then for each $0\leq l \leq L$ there exists a constant $C_{\bR^{(l)}}$ such that
\begin{eqnarray}
    \Pr\left(\norm{\bR^{(l)}}_2 < C_{\bR^{(l)}}\right) \geq 1 - \sum_{j=1}^{l}2e^{-cp_l}.
\end{eqnarray}
For some universal constant $c > 0$. By $\norm{\cdot}_2$ we mean the spectral norm.
\end{lemma}
\begin{proof}
    The proof is by induction. For $\bR^{(0)} = \bI$ it is clear that $\norm{\bR^{(0)}}_2 = 1$. Now assume that the following event holds
    \begin{eqnarray}\label{app:eq:Rboundevent}
        \left\lbrace\norm{\bR^{(l-1)}}_2 \leq C_{\bR^{(l-1)}}\right\rbrace,
    \end{eqnarray}
    then by the definition of $\bR^{(l)}$ we have that
    \begin{eqnarray}
        \norm{\bR^{(l)}}_2 = \norm{\rho_{1,l}^2\bW^{(l)}\bR^{(l-1)}\bW^{(l)T} + \rho_{2,l}^2\bI} \leq \rho_{1,l}^2\norm{\bW^{(l)}}^2_2\norm{\bR^{(l-1)}}_2 + \rho_{2,l}^2 \nwl \leq \rho_{1,l}^2C_{\bR^{(l-1)}}\norm{\bW^{(l)}}^2_2 + \rho_{2,l}^2
    \end{eqnarray}
    Now we recall that the elements of $\bW^{(l)}$ are i.i.d normally distributed with variance $\frac{1}{p_{l-1}}$. Standard results from Random matrix theory (see for example \citep{papaspiliopoulos2020high}[corollary 7.3.3]) demonstrate that
    \begin{eqnarray}
        \Pr\left(\norm{\bW^{(l)}}_2 \geq 1 + \sqrt{p_{l}/p_{l-1}} + t \right) \leq 2e^{cp_{l-1}t^2}.
    \end{eqnarray}
    We choose $t = \sqrt{p_l/p_{l-1}}$ from which we obtain
    \begin{eqnarray}
        \Pr\left(\norm{\bW^{(l)}}_2 \geq 1 + 2\sqrt{p_{l}/p_{l-1}}\right) \leq 2e^{cp_l}.
    \end{eqnarray}
    As such we can choose
    \begin{eqnarray}
        \bC_{\bR^{(l)}} = \rho_{1,l}^2C_{\bR^{(l-1)}}(1+ 2\sqrt{p_l/p_{l-1}})^2 + \rho_{2,l}^2
    \end{eqnarray}
    Now we note that the probability of the event \eqref{app:eq:Rboundevent} hols true with probability
    \begin{eqnarray}
        \Pr\left(\norm{\bW^{(1)}}_2 < 1 + 2\sqrt{p_1/p_0}, \cdots, \norm{\bW^{(l-1)}}_2 < 1+2\sqrt{p_{l-1}/p_{l-2}}  \right) \geq 1 - \sum_{j}^{l-1}2e^{cp_{j}}
    \end{eqnarray}
    where we have made use of the union bound. As such we can say that with high probability $\norm{\bR^{(l)}}_2$ is bounded.
\end{proof}

Next, we show that the optimizations over $\be$ and $\blambda$ can be restricted to compact sets

\begin{lemma}\label{app:CGMT:compactsetlemma1}
    Consider the following two optimization problems, which correspond to the problem $P_2$ and the alternative problem after applying the CGMT:
    \begin{eqnarray}\label{app:CGMT:lemFirstBounds:eq:P21}
        P_{2,1} = \min_{\be\in\bbR^{{p_L}}}\max_{\blambda\in\bbR^{n}} \frac{1}{n}\blambda^T\bu - \frac{1}{n\sqrt{p_L}}\blambda^T\bU^{(L)}\bR^{(L)/2}\be - \frac{1}{2n}\norm{\blambda}_2^2 + R(\be + \btheta^*),\\
        \label{app:CGMT:lemFirstBounds:eq:P22}
        P_{2, 2} = \min_{\be\in\bbR^{{p_L}}}\max_{\blambda\in\bbR^{n}} \frac{1}{n}\blambda^T\bu - \frac{1}{n\sqrt{p_L}}\norm{\blambda}_2\bg^T\bR^{(L)/2}\be - \frac{1}{n\sqrt{p_L}}\norm{\bR^{(L)/2}\be}_2\bh^T\blambda \nwl - \frac{1}{2n}\norm{\blambda}_2^2 + R(\be + \btheta^*).
    \end{eqnarray}
    where $\bg\in\bbR^{p_L},\bh\in\bbR^n$ are standard normal vectors. We define $\hat{\be}_{1}$ and $\hat{\be}_2$ to be the optimal solutions of $P_{2,1}$ and $P_{2,2}$ respectively. Furthermore, let $\hat{\blambda}_1(\be), \hat{\blambda}_2(\be)$ be the optimal solutions of the inner optimization of $P_{2,1}$ and $P_{2,2}$ respectively as functions of $\be$. Let $R$ be $\mu$-strongly convex and let $\norm{\nabla R(\btheta^*)} = \mathcal{O}(\sqrt{p_L})$. Then there exist positive constants $C_{\be}$ and $C_{\blambda}$ that depend only on $\mu$ such that
    \begin{itemize}
        \item The solutions $\hat{\be}_1, \hat{\be}_2$ are
        \begin{eqnarray}
            \lim_{p_{L}\rightarrow \infty}\Pr\left(\max\lbrace \norm{\hat{\be}_1}, \norm{\hat{\be}}_2 \rbrace \leq C_{\be}\sqrt{p_L}\right) = 1
        \end{eqnarray}
        \item and
        \begin{eqnarray}
            \lim_{n\rightarrow \infty}\Pr\left(\sup_{\be: \norm{\be} \leq C_{\be}\sqrt{m}}\max\lbrace \norm{\hat{\blambda}_1}, \norm{\hat{\blambda}_2} \rbrace \leq C_{\blambda}\sqrt{n}\right) = 1
        \end{eqnarray}
    \end{itemize}
\end{lemma}
\begin{proof}
    We recall that $R$ is $\mu$ strongly convex, and we let the function $B(\be) = R(\be + \btheta^*)$. Solving for $\blambda$ in both optimizations, we may expressed the resultant optimization over $\be$ as 
    \begin{eqnarray}
        \min_{\be}F_i(\be)\qquad i =1, 2.
    \end{eqnarray}
    Such that $F_i(\be)$ is the optimal value over the parameter $\blambda$. Next, we note if we set $\blambda = \bm{0}$, both optimizations yield $F_i(\be) \geq R(\be)$. Then we note that
    \begin{eqnarray}
        B(\be)  \geq B(\bm{0}) + \bd^T\be + \mu\norm{\be}_2^2,
    \end{eqnarray}
    from the strong convexity of $R$, where $\bd = \nabla B(\bm{0}) = \nabla R(\btheta^*)$. We note that by assumption $\norm{\bd} = \mathcal{O}(\sqrt{p_L})$. 

    For the first optimization $P_1$, we note that
    \begin{eqnarray}
        F(\bm{0}) = B(\bm{0}) + \frac{1}{2n}\norm{\bnu}_2^2.
    \end{eqnarray}
    From this we note that for the optimal solution $\hat{\be}$ we have
    \begin{eqnarray}\label{app:CGMT:lemFirstBounds:eqKeyIneq}
        B(\bm{0}) + \frac{1}{2n}\norm{\bnu}_2^2 = F(\bm{0}) \geq F(\hat{\be}_1) \geq R(\bm{0}) + \bd^T\hat{\be}_1 + \mu\norm{\hat{\be}}_2^2,
    \end{eqnarray}
    from which we obtain
    \begin{eqnarray}
        \mu\norm{\hat{\be}_1 + \frac{1}{\mu}\bd} \leq \frac{1}{2n}\norm{\bnu}_2^2 + \frac{1}{4\mu}\norm{\bd}_2^2.
    \end{eqnarray}
    As such
    \begin{eqnarray}
        \norm{\hat{\be}_1}_2 \leq \norm{\frac{1}{\mu}\bd}_2 + \sqrt{\frac{1}{2n\mu}\norm{\bnu}_2^2 + \frac{1}{\mu^2}\norm{\bd}_2^2}
    \end{eqnarray}
    
    We recall that from standard random matrix theory \citep{papaspiliopoulos2020high}[Theorem 2.8.1] we know that $\norm{\bnu}_2^2 \leq cn$ for some $n$ with high probability. We may therefore observe that exists a constant $C_{\be_1}$ such that
    \begin{eqnarray}
        \lim_{p_{L}\rightarrow \infty}\Pr(\norm{\hat{\be}_1}_2\geq C_{\be_1}\sqrt{p_{L}}) = 0.
    \end{eqnarray}

    We can now consider problem \eqref{app:CGMT:lemFirstBounds:eq:P21}. We make use of the same strategy in this case. We note that, when we let $\beta = \norm{\blambda}$, the optimization over $\blambda$ with fixed norm can be solved to obtain:
    \begin{eqnarray}
        F(\be) = \max_{\beta \geq 0} \frac{\beta}{n}\norm{\bnu - \frac{1}{\sqrt{p_L}}\norm{\bR^{(L)/2}\be}_2\bg} - \frac{\beta}{n\sqrt{m}}\bh^T\bR^{(L)/2}\be - \frac{\beta^2}{2nm} + B(\be).
    \end{eqnarray}

    We note that this optimization is constrained to the set $\beta \geq 0$, as such dropping the constraint can only increase the optimal value. Dropping the constrains results in a quadratic optimizations which may be solved. We obtain the following inequality
    \begin{eqnarray}
        F(\be) \leq B(\be) + \frac{1}{2n}\left(\norm{\bnu - \frac{1}{\sqrt{p_L}}\norm{\bR^{(L)/2}\be}_2\bg}_2 - \frac{\beta}{\sqrt{m}}\bh^T\bR^{(L)/2}\be \right)^2,
    \end{eqnarray}
    and in particular
    \begin{eqnarray}
    F(\bm{0}) \leq B(\bm{0}) + \frac{1}{2n}\norm{\be}_2^2.
    \end{eqnarray}

    Now making use of the same inequality as in equation~\eqref{app:CGMT:lemFirstBounds:eqKeyIneq} from which we may find that
    \begin{eqnarray}
        \norm{\hat{\be}_2}_2 \leq \norm{\frac{1}{\mu}\bd}_2 + \sqrt{\frac{1}{2n\mu}\norm{\bnu}_2^2 + \frac{1}{\mu^2}\norm{\bd}_2^2}.
    \end{eqnarray}
    As such we can demonstrate that
    \begin{eqnarray}
        \lim_{p_{L}\rightarrow \infty}\Pr(\norm{\hat{\be}_2}_2\geq C_{\be_2}\sqrt{p_L}) = 0.
    \end{eqnarray}
    We let $C_{\be} = \max(C_{\be_1}, C_{\be_2})$, and we make use of this constant to define $A_{\be} = \lbrace \be \in \bbR^{p_L}|\ \norm{\be}_2\leq C_{\be}\sqrt{m} \rbrace$

    Making use of the optimiality condition of the inner optimization in equation~\eqref{app:CGMT:lemFirstBounds:eq:P21}, we see that
    \begin{eqnarray}
        \hat{\blambda}_1(\be) = \bnu - \frac{1}{\sqrt{m}}\bU\bR^{(L)/2}\be.
    \end{eqnarray}
    As such, for all $\be\in A_{\be}$
    \begin{eqnarray}
        \norm{\hat{\blambda}_1(\be)}_2 \leq \norm{\bnu}_2 + \norm{\frac{1}{\sqrt{m}}\bU\bR^{(L)/2}}_2\norm{\be}_2 \leq \norm{\bnu}_2 + \norm{\frac{1}{\sqrt{m }}\bU}_2\norm{\bR^{(L)/2}}_2\norm{\be}_2.
    \end{eqnarray}

    We can then note by lemma \ref{lem:BoundingRL} that $\norm{\bR^{(L)/2}}_2$ is bounded. Furthermore, by standard random matrix theory results we can conclude that $\norm{\frac{1}{\sqrt{m}}\bU}_2 < C$ for some constant $C$ with high probability. Then, using the same arguments as above, we can conclude that t here must exist a constant $C_{\blambda_1}$ such that for all $\be \in A_{\be}$:
    \begin{eqnarray}
        \lim_{n\rightarrow\infty}\Pr\left( \sup_{\be\in A_{\be}}\norm{\hat{\blambda}_1(\be)}_2 \geq C_{\blambda_1}\sqrt{n} \right) = 0
    \end{eqnarray}

    Finally, consider the optimality condition over $\beta$ of problem \ref{app:CGMT:lemFirstBounds:eq:P22} we see that for all $\be \in A_{\be}$ that
    \begin{eqnarray}
        \hat{\beta} = \norm{\hat{\blambda}_2(\be)}_2 = \norm{\bnu - \frac{1}{\sqrt{m}}\norm{\bR^{(L)/2}\be}_2\bg}_2 - \frac{1}{\sqrt{m}}\bR^{(L)/2}\bh\nwl
        \leq \norm{\bnu}_2 + \frac{1}{\sqrt{m}}\norm{\bg}_2\norm{\bR^(L)/2}_2\norm{\be}_2 + \frac{1}{\sqrt{m}}\norm{\bR^{(L)/2}}_2\norm{\bh}_2
    \end{eqnarray}
    With high probability we note that $\norm{\bnu}_2 < C\sqrt{n}, \norm{\bg}_2 < C\sqrt{n}$ and $\norm{\bh} < C\sqrt{p_L}$. As such we can find a constant $C_{\blambda_2}$ with
    \begin{eqnarray}
        \lim_{n\rightarrow \infty} \Pr\left( \sup_{\be \in A_{\be}}\norm{\hat{\blambda}_{2}(\be)}_2 \geq C_{\blambda_2}\sqrt{n}\right) = 0.
    \end{eqnarray}

    Choosing $C_{\blambda} = \max(C_{\blambda_1}, C_{\blambda_2})$, the proof is complete.
 \end{proof}

Making use of this lemma we can define the sets $S_1 = \lbrace\be|\ \norm{\be} \leq C_{\be}\sqrt{m} \rbrace$ and $S_2 = \lbrace \blambda|\ \norm{\blambda}\leq C_{\blambda}\sqrt{n} \rbrace$ and note that these sets are compact and convex. We can with high probability restrict ourselves to the problem 
\begin{eqnarray}
    P_{2}' = \min_{\be\in S_1}\max_{\blambda\in S_2} \frac{1}{n}\blambda^T\bnu - \frac{1}{n\sqrt{p_L}}\blambda^T\bU^{(L)}\bR^{(L)/2}\be - \frac{1}{2n}\norm{\blambda}_2^2 + R(\be + \btheta^*)
\end{eqnarray}
and note that the optimal value of $P_2'$ will be close that of $P_2$. We now statify the conditions for applying the CMGT. Applying it we obtain the following problem:
\begin{eqnarray}
    A_2 = \min_{\be\in S_1}\max_{\blambda\in S_2} \frac{1}{n}\blambda^T\bnu - \frac{1}{n\sqrt{p_L}}\norm{\blambda}_2\bg^T\bR^{(L)/2}\be - \frac{1}{n\sqrt{p_L}}\norm{\bR^{(L)/2}\be}_2\bh^T\blambda \nwl - \frac{1}{2n}\norm{\blambda}_2^2 + R(\be + \btheta^*).
\end{eqnarray}
Where $\bg\in\bbR^{p_L}, \bh\in\bbR^{n}$ have elements that are i.i.d standard normals. By theorem~\ref{app:thm:CGMT} we know that the optimal values of $A_2$ and $P_{2}'$ will be asymptotically equal if $A_2$ converges to a finite value. Next we let $\beta = \frac{1}{\sqrt{n}}\norm{\blambda}$. We note that $0\leq\beta\leq \beta_{max}$, where $\beta_{max} \in \bbR$ is some constant, whose value can be chosen arbitrarily larger than $C_{\blambda}$. We can now solve the optimization over the vector $\blambda$ fixing its length to $\beta$. We obtain
\begin{eqnarray}
    A_2 = \min_{\be \in S_1}\max_{0\leq \beta \leq \beta_{max}} \beta\norm{\frac{1}{\sqrt{n}}\bnu - \frac{1}{\sqrt{np_L}}\norm{\bR^{(L)/2}\be}\bh}_2 \nwl - \frac{\beta}{\sqrt{np_L}}\bg^T\bR^{(L)/2}\be - \frac{\beta^2}{2} + R(\be + \btheta^*).
\end{eqnarray}

Now we note that the first term in the $2-$norm concentrates as $n$ grows large. We prove this in the following lemma

\begin{lemma}\label{app:CGMT:lem:concentration1}
Let $A$ be given by
\begin{eqnarray}
    A(\be, \beta) = \beta\norm{\frac{1}{\sqrt{n}}\bnu - \frac{1}{\sqrt{np_L}}\norm{\bR^{(L)/2}\be}\bh} - \frac{\beta}{\sqrt{np_L}}\bg^T\bR^{(L)/2}\be - \frac{\beta^2}{2} + R(\be + \btheta^*).
\end{eqnarray}
Let $\Tilde{A}(\be, \beta)$ be given by
\begin{eqnarray}
\tilde{A}(\be, \beta) = \beta\sqrt{\sigma_{\bnu}^2 + \frac{1}{p_L}\be^T\bR^{(L)}\be} - \frac{\beta}{\sqrt{np_L}}\bg^T\bR^{(L)/2}\be - \frac{\beta^2}{2} + R(\be + \btheta^*).
\end{eqnarray}
Then, there exists positive constants $C, c$ such that for any $\epsilon > 0$:
\begin{eqnarray}
    \Pr\left(\sup_{\be\in S_1, 0\leq \beta\leq \beta_{max}} |A(\be, \beta) - \Tilde{A}(\be, \beta)| \geq \epsilon \right) \leq Ce^{-cn\epsilon}
\end{eqnarray}
\end{lemma}
\begin{proof}
We note that $A(\be, \beta)$ can be expressed as
\begin{eqnarray}
    A = \beta\sqrt{\frac{1}{n}\norm{\bnu}^2_2 + \frac{1}{np_L}\norm{\bR^{(L)/2}\be}_2^2\norm{\bh}_2^2 - \frac{2}{n\sqrt{p_L}}\norm{\bR^{(L)/2}\be}_2\bnu^T\bh} \nwl 
    - \frac{\beta}{\sqrt{np_L}}\bg^T\bR^{(L)/2}\be - \frac{\beta^2}{2} + R(\be + \btheta^*)
\end{eqnarray}
Or equivalently 
\begin{eqnarray}
    A = \beta\left[\left(\frac{1}{n}\norm{\bnu}^2_2 - \sigma_{\bnu}^2 \right) + \sigma_{\bnu}^2 + \frac{1}{p_L}\norm{\bR^{(L)/2}\be}_2^2\left( \frac{1}{n}\norm{\bh}_2^2 - \right) + \frac{1}{p_L}\norm{\bR^{(L)/2}\be}_2^2 \right. \nwl \left.  - \frac{2}{\sqrt{p_L}}\norm{\bR^{(L)/2}\be}_2\frac{\bnu^T\bh}{n} \right]^{1/2} - \frac{\beta}{\sqrt{np_L}}\bg^T\bR^{(L)/2}\be - \frac{\beta^2}{2} + R(\be + \btheta^*)\nwl
    \leq \bar{A} + \beta\sqrt{\delta} \leq \bar{A} + \beta_{max}\sqrt{\bar{\delta}}
\end{eqnarray}
where
\begin{eqnarray}
    \delta = \left(\frac{1}{n}\norm{\bnu}^2_2 - \sigma_{\bnu}^2 \right) + \frac{1}{p_L}\norm{\bR^{(L)/2}\be}_2^2\left( \frac{1}{n}\norm{\bh}_2^2 - \right) - \frac{2}{\sqrt{p_L}}\norm{\bR^{(L)/2}\be}_2\frac{\bnu^T\bh}{n} \nwl
    \leq  \left(\frac{1}{n}\norm{\bnu}^2_2 - \sigma_{\bnu}^2 \right) + C_{\be}^2C_{\bR^{(L)}}\left( \frac{1}{n}\norm{\bh}_2^2 - \right) - 2\sqrt{C_{\bR^{(L)}}}C_{\be}\left|\frac{\bnu^T\bh}{n}\right| \overset{def}{=} \bar{\delta}.
\end{eqnarray}
From the lemmas above we note that $C_{\bR^{(L)}}$ and $C_{\be}$ are universal constants. Furthermore, it can be readily observed that $\Pr(|\bar{\delta}| \geq \epsilon) \leq Ce^{-cn\epsilon}$ for some constants $C, c > 0$. As such, we see that 
\begin{eqnarray}
    \Pr\left(\sup_{\be \in S_1, 0\leq \beta \leq \beta_{max}} |A(\be, \beta) - \bar{A}(\be, \beta)| \geq \epsilon \right) \leq \nwl
    \Pr\left( \sup_{\be\in S_1, 0\leq \beta \leq \beta_{max}} |\delta\beta| \geq \epsilon\right) \leq \Pr\left(|\beta_{max}\bar{\delta}| \geq \epsilon \right) \leq Ce^{-cn\epsilon}
\end{eqnarray}
For some constants $C, c > 0$.
\end{proof}

By means of this lemma we can, with high probability, consider the following problem
\begin{eqnarray}
    \bar{A}_2 = \min_{\be \in S_1}\max_{0\leq \beta \leq \beta_{max}} \beta\sqrt{\sigma_{\bnu}^2 + \frac{1}{p_L}\be^T\bR^{(L)}\be} - \frac{\beta}{\sqrt{np_L}}\bg^T\bR^{(L)/2}\be - \frac{\beta^2}{2} + R(\be + \btheta^*).
\end{eqnarray}
We now note that this optimization problem is convex in $\be$ and  concave in $\beta$. Furthermore, both optimizations are over convex sets. As such we can interchange the order of min and max
\begin{eqnarray}
    \bar{A}_2 = \max_{0\leq \beta \leq \beta_{max}}\min_{\be \in S_1} \beta\sqrt{\sigma_{\bnu}^2 + \frac{1}{p_L}\be^T\bR^{(L)}\be} - \frac{\beta}{\sqrt{np_L}}\bg^T\bR^{(L)/2}\be - \frac{\beta^2}{2} + R(\be + \btheta^*).
\end{eqnarray}
Now we make use of the "square root trick", which notes that for any scalar $c > 0$ we can express $\sqrt{c} = \min_{q>0}\frac{q}{2} + \frac{c}{2q}$. Using this technique we obtain:
\begin{eqnarray}
    \bar{A}_2 = \max_{0\leq \beta \leq \beta_{max}}\min_{q_{min} < q \leq q_{\max}}\frac{\beta\sigma_{\bnu}^2}{2q} + \frac{\beta q}{2} - \frac{\beta^2}{2} \nwl +
    \min_{\be \in S_1} \frac{\beta}{2qp_L}\be^T\bR^{(L)}\be - \frac{\beta}{\sqrt{np_L}}\bg^T\bR^{(L)/2}\be + R(\be + \btheta^*).
\end{eqnarray}
Where we have interchanged the order of the two minimizations, and have noted that $q$ can be both upper bounded and lower bounded, by $q_{min} = \sigma_{\bnu}$, achieved when $\be = 0$ and $q_{max} > \sqrt{\sigma_{\bnu}^2 + C_{\be}^2C_{\bR^{(L)}}}$. 

We now fix the values of $\beta$ and $q$ and focus only on the inner optimization over $\be$. We shall discuss the outer optimizations below. We define

\begin{eqnarray}
    D^{(L)} = D_2^{(L)}(\beta, q) = \min_{\be\in S_1} \frac{c_L}{2p_L}\be^T\bR^{(L)}\be - \frac{d_L}{p_L}\bg^T\bR^{(L)/2}\be + R(\be + \btheta^*).\\
    c_L = \frac{\beta}{q}\qquad d_L = \beta\sqrt{\frac{m}{n}} \qquad T_L(\beta, q) = \frac{\beta\sigma_{\bnu}^2}{2q} + \frac{\beta q}{2} - \frac{\beta^2}{2}
\end{eqnarray}
such that 
\begin{eqnarray}
A_2 = \max_{\beta}\min_{q} T_L(\beta, q) + D_2^{(L)}(\beta, q).
\end{eqnarray}

We shall focus on $D^{(L)}$ for fixed $\beta, q$. We shall now demonstrate that studying $D^{(L)}$ it maybe expressed as another min max problem. Applying the CGMT recursively to the inner problem and simplifying results in a new problem. 

First we recall the definition of $\bR^{(L)}$ and further note that for a Gaussian $\bg$ that
\begin{eqnarray}
    \bR^{(l)/2}\bg = \tilde{\bg} \sim \calN(\bm{0}, \bR^{(l)}) = \calN(\bm{0}, \frac{\rho_{1, l}^2}{p_{l-1}}\bW^{(l)}\bR^{(l-1)/2}\bW^{(l)T} + \rho_{2, l}^2\bI_{p_l})\nwl
     = \rho_{1, l}\bW^{(l)}\bR^{(l-1)/2}\bg_1 + \rho_{2,l}\bg_2\qquad \bg_1\sim\calN(\bm{0}, \bI_{p_{l-1}}), \bg_2\sim \calN(\bm{0}, \bI_{p_l}).
\end{eqnarray}
We can now substitute in this definition. We obtain:

\begin{eqnarray}
    \min_{\be\in S_1^{(l)}} \frac{c_L\rho_{1,L}^2}{2p_Lp_{L-1}}\be^T\bW^{(L)}\bR^{(L-1)}\bW^{(L)T}\be + \frac{d_L\rho_{1,L}}{p_L\sqrt{p_{L-1}}}\bg_1^T\bR^{(L-1)/2}\bW^{(L)T}\be + \frac{c_L\rho_{2, L}^2}{2p_L}\norm{\be}^2 \nwl + \frac{d_L\rho_{2, L}}{p_L}\bg_2^T\be + R(\be + \btheta^*),
\end{eqnarray}
Where $\bg_1\in\bbR^{p_{L-1}}, \bg_2\in\bbR^{p_L}$ are standard normal vectors. We then complete the square over the vector $\bR^{(L-1)/2}\bW^{(L)T}\be$, we obtain
\begin{eqnarray}
        \min_{\be\in S_1^{(l)}} \frac{c_L\rho_{1,L}^2}{2p_Lp_{L-1}}\norm{\bR^{(L-1)/2}\bW^{(L)T}\be + \frac{d_L\sqrt{p_{L}}}{c_L\rho_{1, L}}\bg_1^T}^2 - \frac{d_L^2}{2c_Lp_{L}}\norm{\bg_1}^2 \nwl + \frac{c_L\rho_{2, L}^2}{2p_L}\norm{\be}^2 + \frac{d_L\rho_{2, L}}{p_L}\bg_2^T\be + R(\be + \btheta^*).
\end{eqnarray}
We can then introduce a new variable $\bs \in \bbR^{p_{L-1}}$ and take the Legendre transform of the 2-norm to create a min-max problem

\begin{eqnarray}\label{app:eq:compactsetproblem2P1}
 \min_{\be\in S_1^{(l)}} \max_{\bs} \frac{c_L\rho_{1,L}^2}{p_Lp_{L-1}}\bs^T\bR^{(L-1)/2}\bW^{(L)T}\be + \frac{d_L\rho_{1,L}}{p_L\sqrt{p_{L-1}}}\bs^T\bg_1 - \frac{c_L\rho_{1,L}^2}{2p_Lp_{L-1}}\norm{\bs}^2- \frac{d_L^2}{2c_Lp_{L}}\norm{\bg_1}^2 \nwl + \frac{c_L\rho_{2, L}^2}{2p_L}\norm{\be}^2 + \frac{d_L\rho_{2, L}}{p_L}\bg_2^T\be + R(\be + \btheta^*). 
\end{eqnarray}

We note that $\bW^{(L)}$ is a Random Matrix with i.i.d standard normal entries, as such we if we can restrict the problem over $\bs$ to a compact and convex set we may make use of the CGMT theorem. We show that we make this restriction in Lemma \ref{app:CGMT:lem:compactSetlemma2}. As such we can consider the following problem:

\begin{eqnarray}
 \min_{\be\in S_1^{(l)}} \max_{\bs\in S_2^{(l)}} \frac{c_L\rho_{1,L}^2}{p_Lp_{L-1}}\bs^T\bR^{(L-1)/2}\bW^{(L)T}\be + \frac{d_L\rho_{1,L}}{p_L\sqrt{p_{L-1}}}\bs^T\bg_1 - \frac{c_L\rho_{1,L}^2}{2p_Lp_{L-1}}\norm{\bs}^2 \nwl - \frac{d_L^2}{2c_Lp_{L}}\norm{\bg_1}^2  + \frac{c_L\rho_{2, L}^2}{2p_L}\norm{\be}^2 + \frac{d_L\rho_{2, L}}{p_L}\bg_2^T\be + R(\be + \btheta^*),  
\end{eqnarray}
where the set $S_2^{(l)} = \lbrace \bs\in \bbR^{p_{L-1}}|\ \norm{\bs} \leq C_{\bs}\sqrt{p_Lp_{L-1}} \rbrace$ where $C_\bs$ is a postive constant. We can then apply the CGMT to obtain the following problem

\begin{eqnarray}\label{app:eq:compactsetproblem2P2}
 \min_{\be\in S_1^{(l)}} \max_{\bs\in S_2^{(l)}} \frac{c_L\rho_{1,L}^2}{p_Lp_{L-1}}\norm{\bR^{(L-1)}\bs}\be^T\bg_3 + \frac{c_L\rho_{1,L}^2}{p_Lp_{L-1}}\norm{\be}\bg_4^T\bR^{(L-1)/2}\bs + \frac{d_L\rho_{1,L}}{p_L\sqrt{p_{L-1}}}\bs^T\bg_1 \nwl - \frac{c_L\rho_{1,L}^2}{2p_Lp_{L-1}}\norm{\bs}^2- \frac{d_L^2}{2c_Lp_{L}}\norm{\bg_1}^2  + \frac{c_L\rho_{2, L}^2}{2p_L}\norm{\be}^2 + \frac{d_L\rho_{2, L}}{p_L}\bg_2^T\be + R(\be + \btheta^*)  
\end{eqnarray}
 where $\bg_3 \in \bbR^{p_{L}}$ and $\bg_4 \in \bbR^{p_{L-1}}$ are standard normal vectors. We introduce a new variable $\bv = \bR^{(L-1)/2}\bs$ and note that $\bv$ can be restricted to a compact set, due to the bounds on $\bR^{(L-1)}$ and $\bs$. We can denote this set $S_3^{(l)} =  \lbrace \bv\in \bbR^{p_{L-1}}|\ \norm{\bv} \leq C_{\bv}\sqrt{p_Lp_{L-1}} \rbrace$ where $C_\bv$ is a positive constant. We then reintroduce this constrain with a Lagrange multiplier $\rho_{1, L}^2\bmu/p_{L-1}\sqrt{p_L}\in\bbR^{L-1}$. We obtain
 \begin{eqnarray}
 \min_{\be\in S_1^{(l)}, \bmu} \max_{\bs\in S_2^{(l)}, \bv\in S_3^{(l)}} \frac{c_L\rho_{1,L}^2}{p_Lp_{L-1}}\norm{\bv}\be^T\bg_3 + \frac{c_L\rho_{1,L}^2}{p_Lp_{L-1}}\norm{\be}\bg_4^T\bv + \frac{d_L\rho_{1,L}}{p_L\sqrt{p_{L-1}}}\bs^T\bg_1 \nwl - \frac{c_L\rho_{1,L}^2}{2p_Lp_{L-1}}\norm{\bs}^2- \frac{d_L^2}{2c_Lp_{L}}\norm{\bg_1}^2  + \frac{c_L\rho_{2, L}^2}{2p_L}\norm{\be}^2 + \frac{d_L\rho_{2, L}}{p_L}\bg_2^T\be + R(\be + \btheta^*) \nwl 
 + \frac{\rho_{1, L}^2}{p_{L-1}\sqrt{p_L}}\bmu^T\bv - \frac{\rho_{1, L}^2}{p_{L-1}\sqrt{p_L}}\bmu^T\bR^{(L-1)/2}\bs
\end{eqnarray}
We then let $\xi_L = \frac{\rho_{1,L}}{\sqrt{p_{L}p_{L-1}}}\norm{\bs}$ and $\chi_L = \frac{\rho_{1,L}}{\sqrt{p_{L}p_{L-1}}}\norm{\bv}$ and solve the optimizations over $\bs$ and $\bv$. We obtain the following problem:

\begin{eqnarray}
     \min_{\be\in S_1^{(l)}, \bmu} \max_{0\leq \xi_L \leq \xi_{L,max}, 0\leq \chi_{L}\leq \chi_{L, max}} \nwl
     \frac{c_L\rho_{1,L}\chi}{\sqrt{p_Lp_{L-1}}}\be^T\bg_3 + \chi\norm{\frac{c_L\rho_{1,L}}{\sqrt{p_Lp_{L-1}}}\norm{\be}\bg_4 + \frac{\rho_{1, l}}{\sqrt{p_{L-1}}}\bmu} \nwl
     - \frac{c_L\xi^2}{2} + \xi\norm{\frac{d_L\rho_{1,L}}{\sqrt{p_L}}\bg_1 - \frac{\rho_{1,L}}{\sqrt{p_{L-1}}}\bR^{(L-1)/2}\bmu} \nwl
     - \frac{d_L^2}{2c_Lp_{L}}\norm{\bg_1}^2  + \frac{c_L\rho_{2, L}^2}{2p_L}\norm{\be}^2 + \frac{d_L\rho_{2, L}}{p_L}\bg_2^T\be + R(\be + \btheta^*) 
\end{eqnarray}

We interchange the order of the min and max terms and then make use of the square root trick to get rid of the two norms. We introduce two new variables $t_L$ and $k_L$:

\begin{eqnarray}
     \max_{0\leq \xi_L \leq \xi_{L,max}, 0\leq \chi_{L}\leq \chi_{L, max}} \min_{0\leq t_l \le t_{L, min}, 0\leq k_{L}\leq k_{L, max}}\min_{\be\in S_1^{(l)}, \bmu}  \nwl
     \frac{c_L\rho_{1,L}\chi}{\sqrt{p_Lp_{L-1}}}\be^T\bg_3 + \frac{\chi_L k_L}{2} + \frac{\chi_Lc_L^2\rho_{1,L}^2}{2k_Lp_Lp_{L-1}}\norm{\be}^2\norm{\bg_4}^2 + \frac{\chi_Lc_L\rho_{1, L}^2}{2k_Lp_{L-1}\sqrt{p_L}}\norm{\be}\bg_4^T\bmu + \frac{\chi_L\rho_{1, l}^2}{2k_Lp_{L-1}}\norm{\bmu}^2 \nwl
     - \frac{c_L\xi^2}{2} + \frac{\xi_L t_L}{2} + \frac{\xi_L d_L^2\rho_{1,L}^2}{2t_Lp_L}\norm{\bg_1} - \frac{\xi_Ld_L\rho_{1,L}^2}{\sqrt{2t_Lp_Lp_{L-1}}}\bg_1\bR^{(l-1)/2}\bmu - \frac{\xi_L\rho_{1,L}^2}{2t_Lp_{L-1}}\bmu^T\bR^{(L-1)}\bmu \nwl
     - \frac{d_L^2}{2c_Lp_{L}}\norm{\bg_1}^2  + \frac{c_L\rho_{2, L}^2}{2p_L}\norm{\be}^2 + \frac{d_L\rho_{2, L}}{p_L}\bg_2^T\be + R(\be + \btheta^*) 
\end{eqnarray}
Using the same arguments as in lemma \ref{app:CGMT:lem:concentration1} it can be seen that the problem concentrates on:

\begin{eqnarray}
    \max_{0\leq \xi_L \leq \xi_{L,max}, 0\leq \chi_{L}\leq \chi_{L, max}} \min_{0\leq t_l \le t_{L, min}, 0\leq k_{L}\leq k_{L, max}}\min_{\be\in S_1^{(l)}, \bmu}  \nwl
     \frac{c_L\rho_{1,L}\chi}{\sqrt{p_Lp_{L-1}}}\be^T\bg_3 + \frac{\chi_L k_L}{2} + \frac{\chi_Lc_L^2\rho_{1,L}^2}{2k_Lp_L}\norm{\be}^2 + \frac{\chi_Lc_L\rho_{1, L}^2}{2k_Lp_{L-1}\sqrt{p_L}}\norm{\be}\bg_4^T\bmu + \frac{\chi_L\rho_{1, l}^2}{2k_Lp_{L-1}}\norm{\bmu}^2 \nwl
     - \frac{c_L\xi^2}{2} + \frac{\xi_L t_L}{2} + \frac{\xi_L d_L^2\rho_{1,L}^2p_{L-1}}{2t_Lp_L} - \frac{\xi_Ld_L\rho_{1,L}^2}{2t_L\sqrt{p_Lp_{L-1}}}\bg_1\bR^{(l-1)/2}\bmu + \frac{\xi_L\rho_{1,L}^2}{2t_Lp_{L-1}}\bmu^T\bR^{(L-1)}\bmu \nwl
     - \frac{d_L^2p_{L-1}}{2c_Lp_{L}}  + \frac{c_L\rho_{2, L}^2}{2p_L}\norm{\be}^2 + \frac{d_L\rho_{2, L}}{p_L}\bg_2^T\be + R(\be + \btheta^*)
\end{eqnarray}

We now let 
\begin{eqnarray}
    T_{L-1} = - \frac{d_L^2p_{L-1}}{2c_Lp_{L}} - \frac{c_L\xi^2}{2} + \frac{\xi_L t_L}{2} + \frac{\xi_L d_L^2\rho_{1,L}^2p_{L-1}}{2t_Lp_L} + \frac{\chi_Lk_L}{2}\\
    a = \frac{\chi_Lc_L^2\rho_{1,L}^2}{k_L} + c_L\rho_{2, L}^2 \qquad b = \sqrt{\frac{c_L^2\rho_{1,L}^2\chi^2 p_L}{p_{L-1}} + d_L^2\rho_{2,L}^2}\\
    c_{L-1}= \frac{\chi_L\rho_{1,l}^2}{k_l} \qquad d_{L-1} = \frac{\chi_L c_L\rho_{1,L}^2}{k_L}
    \qquad \bar{c} =  \frac{\xi_L\rho_{1, L}^2}{t_L} \qquad \bar{d} = \frac{\xi_l\rho_{1,L}^2}{t_L}
    \end{eqnarray}

as such we can obtain:
\begin{eqnarray}\label{app:CGMT:eq:intermediateResult}
        \max_{0\leq \xi_L \leq \xi_{L,max}, 0\leq \chi_{L}\leq \chi_{L, max}} \min_{0\leq t_l \le t_{L, min}, 0\leq k_{L}\leq k_{L, max}} T_{L-1} + \min_{\be\in S_1^{(l)}, \bmu}  \nwl \frac{a}{2p_L}\norm{\be}^2 + \frac{b}{p_L}\be^T\bg_1 + \frac{c_{L-1}}{2p_{L-1}}\norm{\bmu}^2 + \frac{d_{L-1}}{p_{L-1}}\frac{\norm{\be}}{\sqrt{p_{L-1}}}\bg_2^T\bmu + \frac{\bar{c}}{2p_{L-1}}\bmu^T\bR^{(L-1)}\bmu \nwl + \frac{\bar{d}}{p_L}\bg_3^T\bR^{(L-1)/2}\bmu + R(\be + \btheta^*)
\end{eqnarray}
Where $\bg_1\in\bbR^{p_L}, \bg_2, \bg_3\in \bbR^{p_{L-1}}$ are standard normal vectors. We now fix all parameters of the optimization except for $\bmu$ and focus specifically on the last four terms terms. We shall note that this can once again be expressed as a min-max optimization amenable to the CGMT. However at this point we enter a recursive structure. We demonstrate in Lemma \ref{app:CGMT:recursiveStepProof} that a problem of the form

\begin{eqnarray}\label{app:eq:genericRecursiveProblem}
\max_{\bmu} \frac{\gamma_1}{2p_{l-1}}\norm{\bmu}^2 + \frac{\gamma_2}{p_{l-1}}\bg_2^T\bmu + \frac{\gamma_3}{2p_{l-1}}\bmu^T\bR^{(l-1)}\bmu + \frac{\gamma_4}{p_l}\bg_3^T\bR^{(l-1)/2}\bmu
\end{eqnarray}
With generic constants $\bgamma_i$ ($i = 1,\ldots 4$) can be expressed by means of the CGMT as:

\begin{eqnarray}\label{app:eq:genericRecursiveProblem2}
   \max_{0\leq \xi_l\leq \xi_{l, max}, 0\leq \chi_l\leq \chi_{l, max}} \min_{0\leq t_l \leq t_{l, max}, 0\leq k_l\leq k_{l, max}} T_l \nwl +  \min_{\bmu} \frac{\bar{\gamma}_1}{2p_{l}}\norm{\bmu}^2 + \frac{\bar{\gamma}_2}{p_{l}}\bg_2^T\bmu + \frac{\bar{\gamma}_3}{2p_{l}}\bmu^T\bR^{(l-1)}\bmu + \frac{\bar{\gamma}_4}{p_l}\bg_3^T\bR^{(l)/2}\bmu
\end{eqnarray}
Where
\begin{eqnarray}
T_l =  \frac{\chi_l k_l}{2} - \frac{\gamma_3\xi^2_l}{2} + \frac{\xi_lt_l}{2} + \frac{\xi_l\gamma_4^2p_{l-1}}{2t_lp_l} - \frac{\gamma_4^2p_{l-1}}{2\gamma_3p_l} \nwl  - \left(\gamma_1 + \frac{\gamma_3^2\rho_{1,l}^2\chi_l}{k_l} +  \gamma_3\rho_{2, l}^2\right)^{-1}\left(\gamma_4^2\rho_{2, l}^2 + \frac{\gamma_3^2\rho_{1,l}^2\chi_l^2p_l}{p_{l-1}} + \gamma_2 \right)\\
\bar{\gamma_1} = \frac{\xi_l\rho_{1,l}^2}{k_l} -\left(\gamma_1 + \frac{\gamma_3^2\rho_{1,l}^2\chi_l}{k_l} +  \gamma_3\rho_{2, l}^2\right)^{-1}\frac{\gamma_3^2\rho_{1,l}^4\xi_l^2}{2k_l^2}\\
\bar{\gamma_2} = - \left(\gamma_1 + \frac{\gamma_3^2\rho_{1,l}^2\chi_l}{k_l} +  \gamma_3\rho_{2, l}^2\right)^{-1}\left(\gamma_4^2\rho_{2, l}^2 + \frac{\gamma_3^2\rho_{1,l}^2\chi_l^2p_l}{p_{l-1}} + \gamma_2\right)^{1/2}\frac{\gamma_3\rho_{1,l}^2\chi_l}{2k_l}\\
\bar{\gamma_3} = \frac{\xi_l\rho_{1,l}^2}{t_l} \qquad \bar{\gamma}_4 = \frac{\xi_l\gamma_4\rho_{1,l}\sqrt{p_l}}{2t_l\sqrt{p_l-1}}
\end{eqnarray}

We can also note that the termination of the recursion is given by the optimization problem where $\bR^{(0)} = \bI$, in this case

\begin{eqnarray}
\min_{\bmu} \frac{\bar{\gamma}_1}{2p_{l}}\norm{\bmu}^2 + \frac{\bar{\gamma}_2}{p_{l}}\bg_2^T\bmu + \frac{\bar{\gamma}_3}{2p_{l}}\bmu^T\bmu + \frac{\bar{\gamma}_4}{p_l}\bg_3^T\bmu \nwl
 = -\frac{\bar{\gamma}_2^2 + \bar{\gamma}_4^2}{\bar{\gamma_1} + \bar{\gamma}_3} \overset{def}{=}F_0
\end{eqnarray}

As such we can express the final result for the $L-$layer deep RF model as being given by

\begin{eqnarray}
    \max_{\beta >0 }\min_{q} T_L + \max_{\xi_L > 0, \chi_L > 0} \min_{t_L > 0, k_L > 0} T_{L-1} + \min_{\be} \frac{a}{2pl}\norm{\be} + \frac{b}{p_L}\be^T\bg_1 + + R(\btheta + \btheta^*) + \nwl   \max_{\xi_{L-1} > 0, \chi_{L-1} > 0} \min_{t_{L-1} > 0, k_{L-1} > 0} \cdots \max_{\xi_0 \geq 0, \chi_0 \geq 0} \min_{t_0 > 0, k_0 > 0} \sum_{i=1}^{L-2} T_l(\be) 
\end{eqnarray}
Where
\begin{eqnarray}\label{CGMTFullEquations}
    T_L = \frac{\beta q}{2} + \frac{\beta \sigma_{\bnu^2}}{2q} - \frac{\beta^2}{2}\\
    T_{L-1} = - \frac{d_L^2p_{L-1}}{2c_Lp_{L}} - \frac{c_L\xi^2}{2} + \frac{\xi_L t_L}{2} + \frac{\xi_L d_L^2\rho_{1,L}^2p_{L-1}}{2t_Lp_L} + \frac{\chi_Lk_L}{2}\\
    T_l =  \frac{\chi_l k_l}{2} - \frac{\gamma_3\xi^2_l}{2} + \frac{\xi_lt_l}{2} + \frac{\xi_l\gamma_4^2p_{l-1}}{2t_lp_l} - \frac{\gamma_4^2p_{l-1}}{2\gamma_3p_l} \nwl  - \left(\gamma_1 + \frac{c_l^2\rho_{1,l}^2\chi_l}{k_l} +  \gamma_3\rho_{2, l}^2\right)^{-1}\left(d_l^2\rho_{2, l}^2 + \frac{c_l^2\rho_{1,l}^2\chi_l^2p_l}{p_{l-1}} + \bar{d} \right) \qquad 1\leq l \leq L-2\\ 
    T_0 = \frac{d_0^2 + \bar{d}_0^2}{c_0 + \bar{c}_0}\\
    a = \frac{\chi_Lc_L^2\rho_{1,L}^2}{k_L} + c_L\rho_{2, L}^2 \qquad b = \sqrt{\frac{c_L^2\rho_{1,L}^2\chi^2 p_L}{p_{L-1}} + d_L^2\rho_{2,L}^2}
\end{eqnarray}
and the constants $c_i, d_i, \bar{c}_i, \bar{d}_i$ are given by
\begin{eqnarray}
c_L = \frac{\beta}{q} \qquad d_{\beta}\sqrt{\frac{p_L}{n}}
c_{L-1}= \frac{\chi_L\rho_{1,l}^2}{k_l} \qquad d_{L-1} = \frac{\chi_L c_L\rho_{1,L}^2}{k_L}\frac{\norm{\be}}{\sqrt{p_{l-1}}}
\\ \bar{c}_L =  \frac{\xi_L\rho_{1, L}^2}{t_L} \qquad \bar{d}_L = \frac{\xi_l\rho_{1,L}^2}{t_L}\\
c_l = \frac{\xi_l\rho_{1,l}^2}{k_l} -\left(c_{l+1} + \frac{\bar{c}_{l+1}^2\rho_{1,l}^2\chi_l}{k_l} +  \bar{c}_{l+1}\rho_{2, l}^2\right)^{-1}\frac{\bar{c}_l^2\rho_{1,l}^4\xi_l^2}{2k_l^2}\\
d_l = - \left(c_{l+1} + \frac{\bar{c}_{l+1}^2\rho_{1,l}^2\chi_l}{k_l} +  \bar{c}_{l+1}\rho_{2, l}^2\right)^{-1}\left(\bar{d}_{l+1}^2\rho_{2, l}^2 + \frac{\bar{c}_{l+1}^2\rho_{1,l}^2\chi_l^2p_l}{p_{l-1}} + d_{l+1}\right)^{1/2}\frac{\bar{c}_{l+1}\rho_{1,l}^2\chi_l}{2k_l}\\
\bar{c}_{l} = \frac{\xi_l\rho_{1,l}^2}{t_l} \qquad \bar{d}_l = \frac{\xi_l\bar{d}_{l+1}\rho_{1,l}\sqrt{p_l}}{2t_l\sqrt{p_{l-1}}}
\end{eqnarray}

For the final step of the proof we note that for each successive application of the CGMT we froze all previous values of $\beta, q$ as well as $\xi_l, \chi_l, t_l, k_l$ for $l \leq L$. By the properties of the CGMT we know that for these fixed values we have pointwise convergence. However, we wish to demonstrate uniform convergence for the properties that we are interested in. This however this is simple to see in this case.

There are two problems we need to consider. We need to show that  Eq \eqref{app:CGMT:eq:intermediateResult} converges uniformly to \eqref{app:eq:compactsetproblem2P1} For each value of $\beta, q$ and that for each problem \eqref{app:eq:genericRecursiveProblem2} converges uniformly to \eqref{app:eq:genericRecursiveProblem}.  We can see that all optimization variablse $\beta, q, \xi_l, \chi_l, t_l, k_l$ exist in bounded regions. For example $\beta \in [0, \beta_{max}]$. Our goal is to show that each problem is Lipschitz continuous on these regions with some Lipschitz constant K. As each problem is strongly convex it has a unique solution, and all are continuously differentiable on the existing region. As such to show Lipschitz continuity one has to show that each of the partial derivatives is bounded, calculation is tedious but can be completed readily. By bounding the derivatives we can show that all problems are Lipschitz. Uniform convergence can then be demonstrated by means of a simple $\epsilon$-net argument. For an application of this to a recursive CGMT problem, see \cite{Bosch2022DoubleDescent}[Appendix B]. This completes the proof of part 1 of the theorem.

\subsection{Proof of Part 2 of the Theorem}
The proof of part 2 is the same as the proof of part 2 of theorem \ref{thm:GenericUniversality} given in Appendix \ref{app:sec:UniversalityTheoremProof}. A Regularization function $R_{\epsilon}(\be + \btheta^*) = R(\be + \btheta^*) \pm \epsilon h(\be + \btheta^*)$ with $\epsilon$ chosen sufficiently small for $R_{\epsilon}$ to remain strongly convex. As such the first part of the theorem holds. Then by bounding the difference and making use of the bounds on $h(\be)$ the proof can be obtained.

\subsection{Auxiliary Lemmas}
\begin{lemma}\label{app:CGMT:lem:compactSetlemma2}
    Consider the following two problems given in equations \eqref{app:eq:compactsetproblem2P1}, \eqref{app:eq:compactsetproblem2P2} that correspond to a problem and the alternative problem given by the CGMT 
\begin{eqnarray}
 P_1 = \min_{\be\in S_1^{(l)}} \max_{\bs\in S_2^{(l)}} \frac{c_L\rho_{1,L}^2}{p_Lp_{L-1}}\bs^T\bR^{(L-1)/2}\bW^{(L)T}\be + \frac{d_L\rho_{1,L}}{p_L\sqrt{p_{L-1}}}\bs^T\bg_1 - \frac{c_L\rho_{1,L}^2}{2p_Lp_{L-1}}\norm{\bs}^2 \nwl- \frac{d_L^2}{2c_Lp_{L}}\norm{\bg_1}^2  + \frac{c_L\rho_{2, L}^2}{2p_L}\norm{\be}^2 + \frac{d_L\rho_{2, L}}{p_L}\bg_2^T\be + R(\be + \btheta^*)\\
 P_2 =  \min_{\be\in S_1^{(l)}} \max_{\bs\in S_2^{(l)}} \frac{c_L\rho_{1,L}^2}{p_Lp_{L-1}}\norm{\bR^{(L-1)}\bs}\be^T\bg_3 + \frac{c_L\rho_{1,L}^2}{p_Lp_{L-1}}\norm{\be}\bg_4^T\bR^{(L-1)/2}\bs + \frac{d_L\rho_{1,L}}{p_L\sqrt{p_{L-1}}}\bs^T\bg_1 \nwl - \frac{c_L\rho_{1,L}^2}{2p_Lp_{L-1}}\norm{\bs}^2- \frac{d_L^2}{2c_Lp_{L}}\norm{\bg_1}^2  + \frac{c_L\rho_{2, L}^2}{2p_L}\norm{\be}^2 + \frac{d_L\rho_{2, L}}{p_L}\bg_2^T\be + R(\be + \btheta^*) 
\end{eqnarray}
Where $\bg_i$ are standard normal vectors. Denote $\hat{\be}_1$ and $\hat{\be}_2$ as the optimal points of the two problems and let $\hat{s}_1(\be)$ and $\hat{s}_2(\be)$ be the optimal points of the inner optimizations as functions of a fixed $\be$. Recall that $R$ is $\mu-$strongly convex and that $\norm{\nabla R(\bm{0})}= \mathcal{O}(\sqrt{p_L})$. Then there exists positive constants $C_{\be}$ and $C_{\bs}$ depending only on $\mu$ such that
\begin{eqnarray}
    \lim_{p_L \rightarrow \infty }\Pr\left(\norm{\hat{\be}_i}_2 \leq C_{\be}\sqrt{m} \right) = 1\qquad i =1, 2
\end{eqnarray}
and 
\begin{eqnarray}
    \lim_{p_L\rightarrow\infty}\Pr\left(\sup_{\be| \norm{\be} \leq C_\be\sqrt{m}} \norm{\hat{\bs}_i(\be)} \leq C_{\bs}\sqrt{p_Lp_{L-1}}\right) = 1 \qquad i = 1, 2 
\end{eqnarray}
\end{lemma}

\begin{proof}
    We note that $\be$ in problem $P_1$ is already bounded to a compact set. For both optimizations, we solve the inner optimization over $\bs$ and denote this solution as
    \begin{eqnarray}
        \min_{\be} F_i(\be) \qquad i =1,2
    \end{eqnarray}
    Such that $F_i$ is the optimal value over $\bs$. When we set $\bs = 0$ in both optimizations we see that
    \begin{eqnarray}
        F(\be) \geq T(\be) \overset{def}{=}  - \frac{d_L^2}{2c_Lp_{L-1}}\norm{\bg_1}_2^2   + \frac{c_L\rho_{2,L}^2}{2p_L}\norm{\be}_2^2 + \frac{d_L\rho_{2,L}}{p_L}\bg_2^T\be + R(\be + \btheta^*)
    \end{eqnarray}
    We can note readily that $T(\be)$ is $\nu$-strongly convex, for some constant $\nu$ with respect to $\be$. We see that 
    \begin{eqnarray}
        T(\be) \geq T(\bm{0}) + \bd^T\be + \frac{\nu}{2}\norm{\be}_2^2
    \end{eqnarray}
    where $\bd = \nabla T(\bm{0})$. By assumption we note that $\bd = \mathcal{O}(\sqrt{p_L})$. For problem $p_2$ we now note the following:
    \begin{eqnarray}
    F_2(\be) = \max_{\bs} \frac{c_L\rho_{1,L}^2}{p_lp_{L-1}}\norm{\bR^{(L-1)/2}\bs}\bg_3^T\be + \frac{c_L\rho_{1,L}^2}{p_Lp_{L-1}}\norm{\be}_2\bg_4^T\bR^{(L-1)/2}\bs + \frac{d_L\rho_{1,L}}{p_L\sqrt{p_{L-1}}}\bs^T\bg_1 \nwl - \frac{c_L\rho_{1,L^2}}{2p_Lp_{L-1}}\norm{\bs}_2^2 + T(\be)\nwl
    \leq \max_{\bs} \frac{c_L\rho_{1,L}^2}{p_lp_{L-1}}\norm{\bR^{(L-1)/2}}\norm{\bs}\bg_3^T\be + \frac{c_L\rho_{1,L}^2}{p_Lp_{L-1}}\norm{\be}_2\bg_4^T\bR^{(L-1)/2}\bs + \frac{d_L\rho_{1,L}}{p_L\sqrt{p_{L-1}}}\bs^T\bg_1 \nwl - \frac{c_L\rho_{1,L^2}}{2p_Lp_{L-1}}\norm{\bs}_2^2 + T(\be)
    \end{eqnarray}
    Then letting $\xi = \norm{\bs}$ the optimization over $\bs$ may be solved to find that
    \begin{eqnarray}
        F_2(\be) \leq \max_{\xi > 0} \frac{c_L\rho_{1,L}^2\xi}{p_lp_{L-1}}\norm{\bR^{(L-1)/2}}\bg_3^T\be + \xi\norm{\frac{c_L\rho_{1,L}^2}{p_Lp_{L-1}}\norm{\be}_2\bR^{(L-1)/2}\bg_4 + \frac{d_L\rho_{1,L}}{p_L\sqrt{p_{L-1}}}\bg_1} \nwl - \frac{c_L\rho_{1,L^2}\xi^2}{2p_Lp_{L-1}} + T(\be)
    \end{eqnarray}
    We now note that this value will only be increased if the constraint over $\xi$ is dropped, as such
    \begin{eqnarray}
        F_2(\be) \leq \max_{\xi} \frac{c_L\rho_{1,L}^2\xi}{p_lp_{L-1}}\norm{\bR^{(L-1)/2}}\bg_3^T\be + \xi\norm{\frac{c_L\rho_{1,L}^2}{p_Lp_{L-1}}\norm{\be}_2\bR^{(L-1)/2}\bg_4 + \frac{d_L\rho_{1,L}}{p_L\sqrt{p_{L-1}}}\bg_1} \nwl - \frac{c_L\rho_{1,L^2}\xi^2}{2p_Lp_{L-1}} + T(\be)
    \end{eqnarray}
    solving this optimization we see that 
    \begin{eqnarray}
        F_2(\bm{0}) \leq \frac{d_L^2}{c_Lp_{L}}\norm{\bg_1^2}^2 + T(\bm{0})
    \end{eqnarray}
    As such we can see that

    \begin{eqnarray}
      \frac{d_L^2}{c_Lp_{L}}\norm{\bg_1^2}^2 + T(\bm{0}) \geq F(\bm{0}) \geq F(\hat{\be}) \geq T(\bm{0}) + \bd^T\be + \frac{\nu}{2}\norm{\be}  
    \end{eqnarray}
    Hence,
    \begin{eqnarray}
        \frac{\nu}{2}\norm{\be  + \frac{1}{\nu}\bd}^2 \leq \frac{1}{\nu}\norm{\bd}_2^2 + \frac{d_L^2}{c_Lp_{L}}\norm{\bg_1^2}^2
    \end{eqnarray}
    and as such
    \begin{eqnarray}
        \norm{\be}_2 \leq \frac{1}{\nu}\norm{\bd}_2 + \sqrt{\frac{2}{\nu}\norm{\bd}_2^2 + \frac{d_L^2}{c_Lp_{L}}\norm{\bg_1^2}^2}
    \end{eqnarray}
    We recall that with high probability $\norm{\bg_1^2} < C\sqrt{p_{L-1}}$. Recalling the assumptions on $\bd$ and that all contants $p_0\cdots, p_L$ grow at constant ratios we see that there must exist a constant $C_{\be}$ such that
    \begin{eqnarray}
        \Pr(\norm{\hat{\be}} > C_{\be}\sqrt{m}) \rightarrow 0
    \end{eqnarray}

    We now consider the bounds on $\bs$. For problem $P_1$ we can note from the optimality condition over $\bs$ that
    \begin{eqnarray}
        \hat{\bs}_1(\be) = \bR^{(L-1)/2}\bW^{(l)T}\be + \frac{d_l\sqrt{p_{L-1}}}{c_1\rho_{1,l }}\bg
    \end{eqnarray}
    As such for all $\be\in S_1^{(l)}$ we can see that
    \begin{eqnarray}
        \norm{\hat{\bs}(\be)}_2 \leq \norm{\bR^{(L-1)/2}}_2\norm{\bW}\norm{\be}_2 + \frac{d_l\sqrt{p_{l-1}}}{c_1\rho_{1, L}}\norm{\bg}
    \end{eqnarray}
    From Standard results we know that $\norm{\bW^{(l)}}_2 < C\sqrt{p_{L-1}}$ and that $\norm{\bg}_2 \leq C\sqrt{p_{L-1}}$. Using the bounds on $\be$ and $\bR^{(l)}$ we and recalling that $p_L\sim p_{L-1}$ we note that there exists a constant $C_{\bs_1}$ exists. 

    Now noting that $\hat{\xi}$ is an upper bound for $\norm{\hat{\bs}_2}$ in problem for problem $P_2$ we can note from its optimality condition that
    \begin{eqnarray}
        \norm{\hat{\bs}_2(\be)} \leq \hat{\xi} = \norm{\bR^{(L-1)/2}}\bg_3^T\be+ \norm{\norm{\be}\bR^{(L-1)/2}\bg_4 + \frac{d_L\sqrt{p_{L-1}}}{\rho_{1,l}c_L}\bg_1}_2\nwl 
        \leq \norm{\bR^{(L-1)/2}}\norm{\bg_3}\norm{\be} + \norm{\be}\norm{\bR^{(L-1)/2}}\norm{\bg_4} + \frac{d_L\sqrt{p_{L-1}}}{\rho_{1,l}c_L}\norm{\bg_1}
    \end{eqnarray}
    Which making use of the bounds used above we can once again determine that there exists a constant $C_{\bs_2}$. Choosing $C_{\bs}$ to be the maximum of $\bC_{\bs_1}, C_{\bs_2}$ we can then construct the set $S_2^{(l)} = \lbrace \bs\in \bbR^{p_{L-1}}|\ \norm{\bs}\leq C_{\bs}\sqrt{p_{L-1}p_{L}} \rbrace$
\end{proof}

\begin{lemma}\label{app:CGMT:recursiveStepProof}
Consider the following optimization problem given in \eqref{app:eq:genericRecursiveProblem}
\begin{eqnarray}
    \min_{\bmu} \frac{\gamma_1}{2p_{l}}\norm{\bmu}^2 + \frac{\gamma_2}{p_{l}}\bg_2^T\bmu + \frac{\gamma_3}{2p_{l}}\bmu^T\bR^{(l-1)}\bmu + \frac{\gamma_4}{p_l}\bg_3^T\bR^{(l)/2}\bmu
\end{eqnarray}
This problem is asymptotically equivalent to the following problem
\begin{eqnarray}
   \max_{0\leq \xi_l\leq \xi_{l, max}, 0\leq \chi_l\leq \chi_{l, max}} \min_{0\leq t_l \leq t_{l, max}, 0\leq k_l\leq k_{l, max}} T_l \nwl +  \min_{\bmeta} \frac{\bar{\gamma}_1}{2p_{l}}\norm{\bmu}^2 + \frac{\bar{\gamma}_2}{p_{l}}\bg_2^T\bmu + \frac{\bar{\gamma}_3}{2p_{l}}\bmu^T\bR^{(l-1)}\bmu + \frac{\bar{\gamma}_4}{p_l}\bg_3^T\bR^{(l)/2}\bmu
\end{eqnarray}
Where
\begin{eqnarray}
T_l =  \frac{\chi_l k_l}{2} - \frac{\gamma_3\xi^2_l}{2} + \frac{\xi_lt_l}{2} + \frac{\xi_l\gamma_4^2p_{l-1}}{2t_lp_l} - \frac{\gamma_4^2p_{l-1}}{2\gamma_3p_l} \nwl  - \left(\gamma_1 + \frac{\gamma_3^2\rho_{1,l}^2\chi_l}{k_l} +  \gamma_3\rho_{2, l}^2\right)^{-1}\left(\gamma_4^2\rho_{2, l}^2 + \frac{\gamma_3^2\rho_{1,l}^2\chi_l^2p_l}{p_{l-1}} + \gamma_2 \right)\\
\bar{\gamma_1} = \frac{\xi_l\rho_{1,l}^2}{k_l} -\left(\gamma_1 + \frac{\gamma_3^2\rho_{1,l}^2\chi_l}{k_l} +  \gamma_3\rho_{2, l}^2\right)^{-1}\frac{\gamma_3^2\rho_{1,l}^4\xi_l^2}{2k_l^2}\\
\bar{\gamma_2} = - \left(\gamma_1 + \frac{\gamma_3^2\rho_{1,l}^2\chi_l}{k_l} +  \gamma_3\rho_{2, l}^2\right)^{-1}\left(\gamma_4^2\rho_{2, l}^2 + \frac{\gamma_3^2\rho_{1,l}^2\chi_l^2p_l}{p_{l-1}} + \gamma_2\right)^{1/2}\frac{\gamma_3\rho_{1,l}^2\chi_l}{2k_l}\\
\bar{\gamma_3} = \frac{\xi_l\rho_{1,l}^2}{t_l} \qquad \bar{\gamma}_4 = \frac{\xi_l\gamma_4\rho_{1,l}\sqrt{p_l}}{2t_l\sqrt{p_{l-1}}}
\end{eqnarray}

\end{lemma}
\begin{proof}
    We first substitute in the value of $\bR^{(l)}$. From this we obtain
    \begin{eqnarray}
    \min_{\bmu} \frac{\gamma_1}{2p_{l}}\norm{\bmu}^2 + \frac{\gamma_2}{p_{l}}\bg_2^T\bmu + \frac{\gamma_3\rho_{1,l}^2}{2p_{l}p_{l-1}}\bmu^T\bW^{(l)}\bR^{(l-1)}\bW^{(l)}\bmu +  \frac{\gamma_3\rho_{2, l}^2}{2p_{l}}\norm{\bmu}^2 \nwl + \frac{\gamma_4\rho_{1,l}}{p_l\sqrt{p_l}}\bg_2^T\bR^{(l)/2}\bmu + \frac{\gamma_4\rho{2, l}}{p_l}\bg_3^T\bmu.
    \end{eqnarray}
We then complete the square over the vector $\bR^{(l-1)/2}\bW^{(l)T}\bmu$ from which we obtain
\begin{eqnarray}
    \min_{\bmu}\frac{\gamma_3\rho_{1,l}^2}{2p_lp_{l-1}}\norm{\bR^{(l-1)/2}\bW^{(l)T}\bmu + \frac{\gamma_4\sqrt{p_{l-1}}}{\gamma_3\rho_{1,l}}\bg_2}^2 - \frac{\gamma_4^2}{2\gamma_3p_{l}}\norm{\bg_2}\nwl
    \frac{\gamma_1}{2p_{l}}\norm{\bmu}^2 + \frac{\gamma_2}{p_{l}}\bg_2^T\bmu +  \frac{\gamma_3\rho_{2, l}^2}{2p_{l}}\norm{\bmu}^2 + \frac{\gamma_4\rho_{2, l}}{p_l}\bg_3^T\bmu.
\end{eqnarray}
We then take the Legendre transform of the $2$-norm and introduce a new variable $\bs$
\begin{eqnarray}
    \min_{\bmu}\max_{\bs}\frac{\gamma_3\rho_{1,l}^2}{p_lp_{l-1}}\bs^T\bR^{(l-1)/2}\bW^{(l)T}\bmu + \frac{\gamma_4\rho_{1, l}}{p_l\sqrt{p_{l-1}}}\bs^T\bg_2 - \frac{\gamma_3\rho_{1,l}^2}{2p_lp_{l-1}}\norm{\bs}^2 - \frac{\gamma_4^2}{2\gamma_3p_{l}}\norm{\bg_2}\nwl
    \frac{\gamma_1}{2p_{l}}\norm{\bmu}^2 + \frac{\gamma_2}{p_{l}}\bg_2^T\bmu +  \frac{\gamma_3\rho_{2, l}^2}{2p_{l}}\norm{\bmu}^2 + \frac{\gamma_4\rho_{2, l}}{p_l}\bg_3^T\bmu  
\end{eqnarray}
Using the same argument as lemmas \ref{app:CGMT:compactsetlemma1} and \ref{app:CGMT:lem:compactSetlemma2} we can show that these problems can be bounded to compact sets $\bS_1^{(l)}, \bS_2^{(l)}$. As such we can consider the problem
\begin{eqnarray}
    \min_{\bmu\in S_1^{(l)}}\max_{\bs \in S_{2}^{(l)}}\frac{\gamma_3\rho_{1,l}^2}{p_lp_{l-1}}\bs^T\bR^{(l-1)/2}\bW^{(l)T}\bmu + \frac{\gamma_4\rho_{1, l}}{p_l\sqrt{p_{l-1}}}\bs^T\bg_2 - \frac{\gamma_3\rho_{1,l}^2}{2p_lp_{l-1}}\norm{\bs}^2 - \frac{\gamma_4^2}{2\gamma_3p_{l}}\norm{\bg_2}\nwl
    \frac{\gamma_1}{2p_{l}}\norm{\bmu}^2 + \frac{\gamma_2}{p_{l}}\bg_2^T\bmu +  \frac{\gamma_3\rho_{2, l}^2}{2p_{l}}\norm{\bmu}^2 + \frac{\gamma_4\rho_{2, l}}{p_l}\bg_3^T\bmu.
\end{eqnarray}

We can now apply the CGMT obtaining:
\begin{eqnarray}
    \min_{\bmu\in S_1^{(l)}}\max_{\bs \in S_{2}^{(l)}}\frac{\gamma_3\rho_{1,l}^2}{p_lp_{l-1}}\norm{\bmu}\bs^T\bR^{(l-1)/2}\bg_4 + \frac{\gamma_3\rho_{1,l}^2}{p_lp_{l-1}}\norm{\bR^{(l-1)}\bs}\bmu^T\bg_5 + \frac{\gamma_4\rho_{1, l}}{p_l\sqrt{p_{l-1}}}\bs^T\bg_2 \nwl- \frac{\gamma_3\rho_{1,l}^2}{2p_lp_{l-1}}\norm{\bs}^2 - \frac{\gamma_4^2}{2\gamma_3p_{l}}\norm{\bg_2}
    \frac{\gamma_1}{2p_{l}}\norm{\bmu}^2 + \frac{\gamma_2}{p_{l}}\bg_2^T\bmu +  \frac{\gamma_3\rho_{2, l}^2}{2p_{l}}\norm{\bmu}^2 + \frac{\gamma_4\rho_{2, l}}{p_l}\bg_3^T\bmu. 
\end{eqnarray}
We introduce a new variable $\bv = \bR^{(l-1)/2}\bs$ and note that it can be restricted to compact and convex set by means of the bounds on $\bs$ and $\bR^{(l-1)}$. We reintroduce the constraint with a Lagrange multiplier $\frac{\rho_{1,l}}{\sqrt{p_l}p_{l-1}}\bmeta$
\begin{eqnarray}
\min_{\bmu\in S_1^{(l)}}\max_{\bs \in S_{2}^{(l)}, \bv\in S_3^{(l)}}\frac{\gamma_3\rho_{1,l}^2}{p_lp_{l-1}}\norm{\bmu}\bv^T\bg_4 + \frac{\gamma_3\rho_{1,l}^2}{p_lp_{l-1}}\norm{\bv}\bmu^T\bg_5 + \frac{\gamma_4\rho_{1, l}}{p_l\sqrt{p_{l-1}}}\bs^T\bg_2 \nwl- \frac{\gamma_3\rho_{1,l}^2}{2p_lp_{l-1}}\norm{\bs}^2 - \frac{\gamma_4^2}{2\gamma_3p_{l}}\norm{\bg_2}
+ \frac{\gamma_1}{2p_{l}}\norm{\bmu}^2 + \frac{\gamma_2}{p_{l}}\bg_2^T\bmu +  \frac{\gamma_3\rho_{2, l}^2}{2p_{l}}\norm{\bmu}^2 + \frac{\gamma_4\rho_{2, l}}{p_l}\bg_3^T\bmu  \nwl
+ \frac{\rho_{1,l}}{\sqrt{p_l}p_{l-1}}\bmeta^T\bv - \frac{\rho_{1,l}}{\sqrt{p_l}p_{l-1}}\bmeta^T\bR^{(l-1)/2}\bs.
\end{eqnarray}
We then let $\xi_l = \rho_{1, l}\norm{\bs}/\sqrt{p_lp_{l-1}}$ and let $\chi_l = \rho_{1, l}\norm{\bv}/\sqrt{p_lp_{l-1}}$ and solve the optimizations over $\bs$ and $\bv$, from which we obtain:

\begin{eqnarray}
 \min_{\bmu\in S_1^{(l)}}\max_{0\leq \xi_l\leq \xi_{l, max}, 0\leq \chi_l\leq \chi_{l, max}}   \frac{\gamma_3\rho_{1,l}\chi_l}{\sqrt{p_lp_{l-1}}}\bg_5^T\bmu + \chi_l \norm{\frac{\gamma_3\rho_{1,l}}{\sqrt{p_lp_{l-1}}}\norm{\bmu}\bg_4 + \frac{\rho_{1,l}}{\sqrt{p_{l-1}}}\bmeta} \nwl
    - \frac{\gamma_3\xi^2}{2} + \xi_l\norm{\frac{\gamma_4}{\sqrt{p_l}}\bg_2 - \frac{\rho_{1,l}}{\sqrt{p_{l-1}}}\bR^{(l-1)/2}\bmeta}\nwl
    - \frac{\gamma_4^2}{2\gamma_3p_{l}}\norm{\bg_2}
+ \frac{\gamma_1}{2p_{l}}\norm{\bmu}^2 + \frac{\gamma_2}{p_{l}}\bg_2^T\bmu +  \frac{\gamma_3\rho_{2, l}^2}{2p_{l}}\norm{\bmu}^2 + \frac{\gamma_4\rho_{2, l}}{p_l}\bg_3^T\bmu.
\end{eqnarray}
We interchange the order of the min and max and then make use of the square root trick twice introducing new variables $t_l, k_l$. We obtain
\begin{eqnarray}
     \max_{0\leq \xi_l\leq \xi_{l, max}, 0\leq \chi_l\leq \chi_{l, max}} \min_{0\leq t_l \leq t_{l, max}, 0\leq k_l\leq k_{l, max}}\min_{\bmu\in S_1^{(l)}, \bmeta}\nwl
     \frac{\gamma_3\rho_{1,l}\chi_l}{\sqrt{p_lp_{l-1}}}\bg_5^T\bmu + \frac{\chi_l k_l}{2} + \frac{\gamma_3^2\rho_{1,l}^2\chi_l}{2k_l p_lp_{l-1}}\norm{\bmu}^2\norm{\bg_4}^2 + \frac{\gamma_3\rho_{1,l}^2\chi_l}{2k_lp_{l-1}\sqrt{p_l}}\norm{\bmu}\bg_4^T\bmeta + \frac{\chi_l\rho_{1,l}^2}{2k_lp_{l-1}}\norm{\bmeta}^2 \nwl
    - \frac{\gamma_3\xi^2}{2} + \frac{\xi_lt_l}{2} + \frac{\xi_l \gamma_4^2}{2t_lp_l}\norm{\bg_2}^2 -\frac{\xi_l\gamma_4\rho_{1,l}}{2t_l\sqrt{p_lp_{l-1}}}\bg_2^T\bR^{(l-1)/2}\bmeta + \frac{\xi_l\rho_{1,l}^2}{2t_lp_{l-1}}\bmeta^T\bR^{(l-1)}\bmeta\nwl
    - \frac{\gamma_4^2}{2\gamma_3p_{l}}\norm{\bg_2}
+ \frac{\gamma_1}{2p_{l}}\norm{\bmu}^2 + \frac{\gamma_2}{p_{l}}\bg_2^T\bmu +  \frac{\gamma_3\rho_{2, l}^2}{2p_{l}}\norm{\bmu}^2 + \frac{\gamma_4\rho_{2, l}}{p_l}\bg_3^T\bmu.
\end{eqnarray}
Now let $\balpha = \norm{\bmu}/\sqrt{p_l}$ and solve over $\bmu$, from this we obtain
\begin{eqnarray}
      \max_{0\leq \xi_l\leq \xi_{l, max}, 0\leq \chi_l\leq \chi_{l, max}} \min_{0\leq t_l \leq t_{l, max}, 0\leq k_l\leq k_{l, max}}\min_{\alpha\leq \alpha_{max}, \bmeta}\nwl
      \frac{\chi_l k_l}{2} + \frac{\gamma_3^2\rho_{1,l}^2\chi_l\alpha^2}{2k_lp_{l-1}}\norm{\bg_4}^2 + \frac{\gamma_3\rho_{1,l}^2\chi_l\alpha}{2k_lp_{l-1}}\bg_4^T\bmeta + \frac{\chi_l\rho_{1,l}^2}{2k_lp_{l-1}}\norm{\bmeta}^2 \nwl
    - \frac{\gamma_3\xi^2}{2} + \frac{\xi_lt_l}{2} + \frac{\xi_l \gamma_4^2}{2t_lp_l}\norm{\bg_2}^2 -\frac{\xi_l\gamma_4\rho_{1,l}}{2t_l\sqrt{p_lp_{l-1}}}\bg_2^T\bR^{(l-1)/2}\bmeta + \frac{\xi_l\rho_{1,l}^2}{2t_lp_{l-1}}\bmeta^T\bR^{(l-1)}\bmeta\nwl
    - \frac{\gamma_4^2}{2\gamma_3p_{l}}\norm{\bg_2}
+ \frac{\gamma_1\alpha^2}{2} +  \frac{\gamma_3\rho_{2, l}^2\alpha^2}{2} + \alpha\norm{\frac{\gamma_4\rho_{2, l}}{\sqrt{p_l}}\bg_3 + \frac{\gamma_3\rho_{1,l}\chi_l}{\sqrt{p_{l-1}}}\bg_5 +  \frac{\gamma_2}{\sqrt{p_{l}}}\bg_2}.
\end{eqnarray}

This now using the same arguments as lemma \ref{app:CGMT:lem:concentration1} this problem concentrates on:

\begin{eqnarray}
  \max_{0\leq \xi_l\leq \xi_{l, max}, 0\leq \chi_l\leq \chi_{l, max}} \min_{0\leq t_l \leq t_{l, max}, 0\leq k_l\leq k_{l, max}}\min_{\alpha\leq \alpha_{max}, \bmeta}\nwl
      \frac{\chi_l k_l}{2} + \frac{\gamma_3^2\rho_{1,l}^2\chi_l\alpha^2}{2k_l} + \frac{\gamma_3\rho_{1,l}^2\chi_l\alpha}{2k_lp_{l-1}}\bg_4^T\bmeta + \frac{\chi_l\rho_{1,l}^2}{2k_lp_{l-1}}\norm{\bmeta}^2 \nwl
    - \frac{\gamma_3\xi^2}{2} + \frac{\xi_lt_l}{2} + \frac{\xi_l \gamma_4^2p_{l-1}}{2t_lp_l} -\frac{\xi_l\gamma_4\rho_{1,l}}{2t_l\sqrt{p_lp_{l-1}}}\bg_2^T\bR^{(l-1)/2}\bmeta + \frac{\xi_l\rho_{1,l}^2}{2t_lp_{l-1}}\bmeta^T\bR^{(l-1)}\bmeta\nwl
    - \frac{\gamma_4^2p_{l-1}}{2\gamma_3p_{l}}
+ \frac{\gamma_1\alpha^2}{2} +  \frac{\gamma_3\rho_{2, l}^2\alpha^2}{2} + \alpha\left(\gamma_4^2\rho_{2, l}^2 + \frac{\gamma_3^2\rho_{1,l}^2\chi_l^2p_l}{p_{l-1}} + \gamma_2\right)^{1/2}.      
\end{eqnarray}
Examining just the optimization over $\alpha$ we see that this may be solved explicitly:

\begin{eqnarray}
    \min_{\alpha} \left(\frac{\gamma_1}{2} + \frac{\gamma_3^2\rho_{1,l}^2\chi_l}{2k_l} +  \frac{\gamma_3\rho_{2, l}^2}{2}\right)\alpha^2 + \left(\left(\gamma_4^2\rho_{2, l}^2 + \frac{\gamma_3^2\rho_{1,l}^2\chi_l^2p_l}{p_{l-1}} + \gamma_2\right)^{1/2} +  \frac{\gamma_3\rho_{1,l}^2\chi_l}{2k_lp_{l-1}}\bg_4^T\bmeta\right) \alpha
\end{eqnarray}
Which has optimal value
\begin{eqnarray}
 - \left(\gamma_1 + \frac{\gamma_3^2\rho_{1,l}^2\chi_l}{k_l} +  \gamma_3\rho_{2, l}^2\right)^{-1}\left(\gamma_4^2\rho_{2, l}^2 + \frac{\gamma_3^2\rho_{1,l}^2\chi_l^2p_l}{p_{l-1}} + \gamma_2 \right. \nwl \left.+  \left(\gamma_4^2\rho_{2, l}^2 + \frac{\gamma_3^2\rho_{1,l}^2\chi_l^2p_l}{p_{l-1}} + \gamma_2\right)^{1/2}\frac{\gamma_3\rho_{1,l}^2\chi_l}{2k_lp_{l-1}}\bg_4^T\bmeta + \frac{\gamma_3^2\rho_{1,l}^4\xi_l^2}{4k_l^22p_{l-1}}\norm{\bmeta}^2\right)
\end{eqnarray}

As such we can collect all of the terms together. Making the following definitions:

\begin{eqnarray}
T_l =  \frac{\chi_l k_l}{2} - \frac{\gamma_3\xi^2_l}{2} + \frac{\xi_lt_l}{2} + \frac{\xi_l\gamma_4^2p_{l-1}}{2t_lp_l} - \frac{\gamma_4^2p_{l-1}}{2\gamma_3p_l} \nwl  - \left(\gamma_1 + \frac{\gamma_3^2\rho_{1,l}^2\chi_l}{k_l} +  \gamma_3\rho_{2, l}^2\right)^{-1}\left(\gamma_4^2\rho_{2, l}^2 + \frac{\gamma_3^2\rho_{1,l}^2\chi_l^2p_l}{p_{l-1}} + \gamma_2 \right)\\
\bar{\gamma_1} = \frac{\xi_l\rho_{1,l}^2}{k_l} -\left(\gamma_1 + \frac{\gamma_3^2\rho_{1,l}^2\chi_l}{k_l} +  \gamma_3\rho_{2, l}^2\right)^{-1}\frac{\gamma_3^2\rho_{1,l}^4\xi_l^2}{2k_l^2}\\
\bar{\gamma_2} = - \left(\gamma_1 + \frac{\gamma_3^2\rho_{1,l}^2\chi_l}{k_l} +  \gamma_3\rho_{2, l}^2\right)^{-1}\left(\gamma_4^2\rho_{2, l}^2 + \frac{\gamma_3^2\rho_{1,l}^2\chi_l^2p_l}{p_{l-1}} + \gamma_2\right)^{1/2}\frac{\gamma_3\rho_{1,l}^2\chi_l}{2k_l}\\
\bar{\gamma_3} = \frac{\xi_l\rho_{1,l}^2}{t_l} \qquad \bar{\gamma}_4 = \frac{\xi_l\gamma_4\rho_{1,l}\sqrt{p_l}}{2t_l\sqrt{p_{l-1}}}.
\end{eqnarray}
As such we find that the optimization is equal to

\begin{eqnarray}
   \max_{0\leq \xi_l\leq \xi_{l, max}, 0\leq \chi_l\leq \chi_{l, max}} \min_{0\leq t_l \leq t_{l, max}, 0\leq k_l\leq k_{l, max}} T_l \nwl +  \min_{\bmeta} \frac{\bar{\gamma}_1}{2p_{l}}\norm{\bmu}^2 + \frac{\bar{\gamma}_2}{p_{l}}\bg_2^T\bmu + \frac{\bar{\gamma}_3}{2p_{l}}\bmu^T\bR^{(l-1)}\bmu + \frac{\bar{\gamma}_4}{p_l}\bg_3^T\bR^{(l)/2}\bmu.
\end{eqnarray}
\end{proof}

\subsection{All Layers of Same Size}

Consider the case that the input dimension is $d$ and all subsequent hidden layers are of dimension $p$. In this case we note that $\bR^{(l)}\in\bbR^{p\times p}$ for all $l > 1$. In this case the recursive application of the CGMT analysis simplifies considerably. The recursion is given in the following lemma.

\begin{theorem} \label{app:CGMT:allLayersSameSize}
Consider the problem $P_2$ given in \eqref{eq:CGTMP2} and assume that the layers $p_1 = p_2 = \cdots p_L = p$, ie all layers are of the same size. Let the input dimension be of size $p_0$ which is not necessarily the same as $p$. In this case the alternative optimization problem may be given by:
\begin{eqnarray}
    \max_{\beta > 0}\min_{q>0}\max_{\xi_1 > 0} \min_{t_1 > 0}\cdots\max_{\xi_L > 0} \min_{t_L > 0} \calM_{\frac{p}{C}\ R(\cdot + \btheta^*)}\left(-\frac{D}{C}\bg\right) + T_L
\end{eqnarray}
Where
\begin{eqnarray}
    c_0 = \frac{\beta}{q} \qquad d_0 = \beta\sqrt{\frac{n}{p}}\qquad T_0 = \frac{\beta q}{2} + \frac{\beta\sigma_{\bnu}^2}{2q} - \frac{\beta^2}{q}\\
    c_{l+1} = \frac{\xi_lc_{l}^2\rho_{1,l}^2}{t_l} \qquad d_{l+1} = c_{l}^2\xi_l^2\rho_{1,l}^2\frac{p_{L-l-1}}{p_{L-l}}\\
    C = c_{L} + \sum_{l=0^{L-1}} \rho_{2,L-l}^2c_{L} \qquad D = \sqrt{d_{L}^2 + \sum_{l = 0}^{L-1}\rho_{2, L-l}^2d_l}\nwl 
    T_{l+1} = T_{l} + \frac{d_l^2\rho_{1,l}^2\xi_l}{2t_l}\frac{p_{L-l-1}}{p_{L-l}} - \frac{c_l\xi_l^2}{2} + \frac{\xi_lt_l}{2} - \frac{d_l^2}{2c_l}\frac{p_{L-l-1}}{p_{L-l}}
\end{eqnarray}
Note that as $p_1 = p_2 = \cdots p_L = p$ the value of $\frac{p_{L-l-1}}{p_{L-l}} = 1$ except in the case of $p_0 = d$.
\end{theorem}
\begin{proof}
The proof is the same as the one given for the CGMT analysis for layers of different sizes. We therefore do not give it here in full.
\end{proof}

\section{Lyapunov Recursions} \label{app:sec:LyapanovProofs}
Let $\bA$ be a $n\times m$ matrix with random entries. Consider the function $f_A(\lambda)$ with gives the probability distribution, or \textit{eigendistribution}, of the eigenvalues of the matrix $A$, defined to be
\begin{eqnarray}
    f_\bA(\lambda) = \frac{1}{n}\sum_{i=1}^n \delta_{\lambda_i}
\end{eqnarray}
where $\lambda_i$ is the $i$th eigenvalue of $A$ and $\delta$ is the dirac measure .

To analyze this distribution, we may instead consider the Stieltjes transform of the distribution $f_{A}$, this transform is defined by 
\begin{eqnarray}\label{app:eq:FormalStieltjesTransform}
    S_{\bA}(z) = \bbE\left[\frac{1}{\lambda - z} \right] = \int \frac{f_\bA(\lambda)}{\lambda - z}d\lambda.
\end{eqnarray}
Here $z$ is a complex number. The original distribution may be recovered by means of the inverse transform

\begin{eqnarray}
    f_{\bA}(\lambda) = \lim_{\omega \rightarrow 0^+} \frac{1}{\pi}\mathrm{Im}\left[S_{\bA}(\lambda + i\omega) \right]
\end{eqnarray}
where $i$ is the imaginary unit. The Stieltjes transform can also be compute directly from the random matrix $\bA$ instead of using equation \eqref{app:eq:FormalStieltjesTransform}. We give the following lemma

\begin{lemma}
    The Stieltjes transform of the expected eigendistribution of a Hermitian random $n\times n$ matrix $\bA$ may be expressed as
    \begin{eqnarray}
        S_{\bA}(z) = \frac{1}{n}\bbE\ \tr(\bA - z\bI)^{-1}
    \end{eqnarray}
\end{lemma}
\begin{proof}
For a proof see \cite{vakili2011RandomMatrix} [lemma 2.3.1]
\end{proof}

Another transform that we will make use of in our analysis of the recursively defined matrix $\bR$ is the $S$-transform may be expressed in terms of the Stieltjes transform by means of 
\begin{eqnarray}
    \Sigma_{\bA}(z) = \frac{z+1}{z}\left(-\frac{1}{z}S_{\bA}\left(\frac{1}{z}\right) - 1 \right)^{\lbrace-1\rbrace} = \frac{z+1}{z}\left(\sum_{i=1}^{\infty} m_iz^i \right)^{\lbrace -1\rbrace}
\end{eqnarray}
Here $\lbrace -1 \rbrace$ denotes the functional inverse, and $m_i$ is the $i$th moment of the distribution $f_{\bA}$. The S-transform has two properties that are instrumental for our analysis. Firstly, the S-transform and the Stieltjes transform satisfy the following relation:
\begin{eqnarray}\label{app:eq:S-StieltjesRelation}
    \Sigma_{\bA} = -\frac{1}{z}S_{\bA}\left(\frac{1+z}{z\Sigma_{\bA}(z)} \right)
\end{eqnarray}

The second key property of the S-transform relates it how it behaves with respect to matrix product. For this we introduce the following lemma

\begin{lemma}
Let $\bA, \bB$ be two non negative unitarily invariant matrices, and let $\bC = \bA\bB$, then the S transform of the eigendistribution of $\bC$ satisfies 
\begin{eqnarray}
    \Sigma_{\bC}(z) = \Sigma_{\bA}(z)\Sigma_{\bB}(z) 
\end{eqnarray}
\end{lemma}
\begin{proof}
The S-transform is multiplicative for matrix product that are asymptotically free \cite{emery2007lectures}. To see that unitarily invariant matrices are free see \cite{voiculescu1991limit}.
\end{proof}

Finally, we note that if $\bH$ is a $m\times n$ matrix with element distributed as $\calN(0, 1)$, then the matrix $\bA = \frac{1}{n}\bH\bH^T$ is a Wishart matrix. We note that the Stieltjes transform of a Wishart matrix is given by the Marcenko-Pastur Law
\begin{eqnarray}
    S_{\bA}(z) = \frac{1 - \frac{m}{n} - z + \sqrt{z^2 - 2\left(\frac{m}{n} + 1\right)z + \left(\frac{m }{n} - 1\right)^2}}{\frac{2m}{n}z}
\end{eqnarray}

and the $S-$transform of a Wishart is given by
\begin{eqnarray}\label{app:eq:StransformWishart}
    \Sigma_{\bA}(z) = \frac{1}{1+\frac{m}{n}z}
\end{eqnarray}

\subsection{Analysis of the Covariance Matrix \textbf{R}}
In this section we adopt an approach for studying Stieltjes transforms of Lyapanov Recursions of Random matrices discussed by \cite{vakili2011RandomMatrix}[Section 3].

We recall that $\bR^{(l)}$ is given by

\begin{eqnarray}
    \bR^{(l)} = \frac{\rho_{1,l}^2}{p_{l-1}}\bW^{(l)}\bR^{(l-1)}\bW^{(l)T} + \left(\rho_{2,l} \right)^2\bI.
\end{eqnarray}
where we recall that $\bR^{(0)} = \bI$ and that the rows of $\bW^{(l)}$, $\bw_i^{(l)}$ are distributed as $\calN(0, \bI)$. We can note that $\bW^{(l)T}\bW^{(l)}/p_{l-1}$ is a Wishart matrix.
We now wish to compute the Stieltjes transform of $\bR^{(l)}$. The Stieltjes transform is given by
\begin{eqnarray}
    S_{\bR^{(l)}}(z) = \frac{1}{p_l}\bbE\tr \left(\frac{\rho_{1,l}^2}{p_{l-1}}\bW^{(l)}\bR^{(l-1)}\bW^{(l)T} + (\rho_{2,l}^2-z)\bI \right)^{-1}
\end{eqnarray}

We now let the matrix $\bA^{(l)} = \bW^{(l)T}\bR^{(l-1)}\bW^{(l)T}/p_{l-1}$. We can then note that
\begin{eqnarray}
    S_{\bR^{(l)}}(z) = \frac{1}{p_{l}}\bbE\tr\left( \rho_{1, l}^2\bA^{(l)} + (\rho_{2,l}^2 -z)\bI\right)^{-1} = \frac{1}{\rho_{1,l}^2}\frac{1}{p_{l}}\bbE\tr\left(\bA^{(l)} + \frac{(\rho_{2,l}^2 -z)}{\rho_{1,l}^2}\bI\right)^{-1}\nwl
    =  \frac{1}{\rho_{1,l}^2}S_{\bA^{(l)}}\left(\frac{z-\rho_{2,l}^2}{\rho_{1,l}^2}\right)
\end{eqnarray}
Our goal is to now find an expression for the Stieltjes transform of $\bA^{(l)}$. We note that $\bW^{(l)T}\bR^{(l-1)}\bW^{(l)T}/p_{l-1}$ has the same eigenvalues as $\bW^{(l)T}\bW^{(l)T}/p_{l-1}\bR^{(l-1)}$. we recall that  $\bW^{(l)T}\bW^{(l)T}/p_{l-1}$ is Wishart and unitarily Invariant, and similarly is $\bR^{(l-1)}$. As such we can make use of the properties of $S-$transforms to note that:
\begin{eqnarray}
    \Sigma_{\bA^{(l)}}(z) = \Sigma_{\bW^{(l)T}\bW^{(l)T}/p_{l}}(z)\Sigma_{\bR^{(l-1)}}(z).
\end{eqnarray}
We can then make use of equation \eqref{app:eq:S-StieltjesRelation} to obtain
\begin{eqnarray}
    S_{\bA^{(l)}}\left(\frac{1+z}{z\Sigma_{\bA^{(l)}}}(z)\right) = \Sigma_{\bW^{(l)T}\bW^{(l)T}/p_{l-}}(z)S_{\bR^{(l-1)}}\left(\frac{1+z}{z\Sigma_{\bR^{(l-1)}}(z)}\right) \nwl = \Sigma_{\bW^{(l)T}\bW^{(l)T}/p_{l-1}}(z)S_{\bR^{(l-1)}}\left(\frac{1+z}{z\Sigma_{\bA^{(l)}}(z)}\Sigma_{\bW^{(l)T}\bW^{(l)T}/p_{l-1}}(z)\right)
\end{eqnarray}

We now let 
\begin{eqnarray}
    x = \frac{1+z}{z\Sigma_{\bA^{(l)}}(z)},
\end{eqnarray}
and then note that
\begin{eqnarray}
    x\Sigma_{\bA^{(l)}}(z) = \frac{1+z}{z} \Rightarrow x\left(\frac{-1}{z}\right)S_{\bA^{(l)}}(x) = \frac{1+z}{z}\nwl
    \Rightarrow z = -1-xS_{\bA^{(l)}}(x)
\end{eqnarray}

By substituting in this expression we obtain
\begin{eqnarray}
    S_{\bA^{(l)}}(x) =  \Sigma_{\bW^{(l)T}\bW^{(l)T}/p_{l-1}}(-1 - xS_{\bA^{(l)}}(x)) S_{\bR^{(l-1)}}\left(x \Sigma_{\bW^{(l)T}\bW^{(l)T}/p_{l-1}}(-1-xS_{\bA^{(l)}}(x)) \right)
\end{eqnarray}
Finally, we recall equation \eqref{app:eq:StransformWishart}. Letting $\beta_l = \frac{p_l}{p_{l-1}}$  we use this property to simplify the relation to:

\begin{eqnarray}
    S_{\bA^{(l)}}(x) = \frac{1}{1 - \beta_l - \beta_l xS_{\bA^{(l)}}(x)}S_{\bR^{(l-1)}}\left(\frac{x}{1 - \beta_l - \beta_lxS_{\bA^{(l)}}(x)} \right)
\end{eqnarray}
Finally, letting $\Omega_{l-1}(\cdot) = S_{\bA^{(l)}}(\cdot)$. We can conclude that

\begin{eqnarray}
  S_{\bR^{(l+1)}}(z) = \frac{1}{\rho_{1,l}^2} \Omega_{l}\left( \frac{z- \rho_{2,l}^2}{\rho_{1,l}^2} \right)\\
  \Omega_l(z) = \frac{1}{1 - \beta_l - z\beta_l\Omega_{l}(z)}S_{\bR^{(l)}}\left(\frac{z}{1 - \beta_l - \beta_lz\Omega_l(z)} \right)
\end{eqnarray}
Which concludes the proof.

\end{document}